\documentclass[11pt]{article}

\usepackage{fullpage}
\usepackage{libertine}

\usepackage[ruled]{algorithm2e} 

\usepackage{mathtools}
\usepackage{amssymb,amsmath,amsfonts,amstext,amsthm}

\usepackage{url}

\newtheorem{theorem}{Theorem}
\newtheorem{lemma}{Lemma}

\newtheorem{corollary}{Corollary}

\theoremstyle{definition}

\allowdisplaybreaks

\usepackage{graphicx}
\usepackage{subcaption}
\usepackage{multirow}
\usepackage{ctable}

\newcommand{\remove}[1]{ }

\newcommand{\EEE}{\mathcal{E}}

\DeclareMathOperator{\sgn}{{\tt sgn}}
\newcommand{\ud}{\,\mathrm{d}}
\newcommand{\bigO}{\mathcal{O}}

\newcommand{\cc}{\mathbf{c}}
\newcommand{\dd}{\mathbf{d}}

\newcommand{\calT}{\mathcal{T}}

\usepackage{authblk}
\usepackage{natbib}

\begin{document}

\title{\bf How effective can simple ordinal peer grading be?\thanks{A preliminary version of this paper appeared in {\em Proceedings of the 17th ACM Conference on Economics and Computation (EC)}, pages 323--340, 2016. This work has been partially supported by a PhD scholarship from the Onassis Foundation, and by the European Research Council (ERC) under grant number 639945 (ACCORD).}
} 

\author{Ioannis Caragiannis \quad George A. Krimpas}
\affil{\em Department of Computer Engineering and Informatics, University of Patras, Greece}

\author{Alexandros A. Voudouris}
\affil{\em School of Computer Science and Electronic Engineering, University of Essex, UK}

\date{}

\maketitle   

\begin{abstract}
Ordinal peer grading has been proposed as a simple and scalable solution for computing reliable information about student performance in massive open online courses. The idea is to outsource the grading task to the students themselves as follows. After the end of an exam, each student is asked to rank --- in terms of quality --- a bundle of exam papers by fellow students. An aggregation rule then combines the individual rankings into a global one that contains all students. We define a broad class of simple aggregation rules, which we call type-ordering aggregation rules, and present a theoretical framework for assessing their effectiveness. 
When statistical information about the grading behaviour of students is available 
(in terms of a noise matrix that characterizes the grading behaviour of the average student from a student population), the framework can be used to compute the optimal rule from this class with respect to a series of performance objectives that compare the ranking returned by the aggregation rule to the underlying ground truth ranking.
For example, a natural rule known as Borda is proved to be optimal when students grade correctly. In addition, we present extensive simulations that validate our theory and prove it to be extremely accurate in predicting the performance of aggregation rules even when only rough information about grading behaviour (i.e., an approximation of the noise matrix) is available. 
Both in the application of our theoretical framework and in our simulations, we exploit data about grading behaviour of students that have been extracted from two field experiments in the University of Patras.
\end{abstract}

\section{Introduction}\label{sec:intro}
Educational platforms such as Coursera and EdX provide easy access to high level education to everyone who has a decent Internet access. At the end of 2018, these platforms had more than 101 million users --- essentially, students attending the offered courses --- and this number is expected to further increase in the near future. The term ``massive open online course'', or simply MOOC, is very descriptive of the service these platforms offer. A MOOC is the result of their partnership with a faculty member in a top university, whose role is to design the course and organize the course material so that it takes advantage of the most popular Internet apps that the platform utilizes. Courses offered include literally everything.

Even though the service provided is certainly useful, the viability of MOOCs will strongly depend on their revenue sources. Currently, investments from VCs have secured their survival for a short term, but their long term success requires a more stable business model. A feature that could be the main source of revenue for MOOCs is the so-called {\em verified certificate} which the students can get at a reasonable cost. The verified certificate keeps information about the performance of a student in a course (or in a chain of courses) and can be used to justify a student's quality to potential employers. So, the verified certificate should have {\em reliable information} about the student performance in the courses she has participated in. Even though the means to guarantee this in the traditional University system is well-established, achieving this in a MOOC is a challenge.

The big issue is in the massive student participation. Of course, the Internet provides tools so that organizing exams with huge numbers of students is logistically feasible. But what about assessment and grading? As the most popular courses attract $50\,000$ students or more and the vision of MOOCs enthusiasts is for millions of students per course, is grading of assignments or exams possible? Undoubtedly, professional graders would be extremely costly. Organizing the material using multiple-choice questions and answers that could be graded automatically cannot be an option when the students are asked to prepare an essay or a formal mathematical proof or express their critical thinking over some issue. Grading is a typical example of a {\em human computation}~\citep{LvA11} task in these cases.

The only solution that seems consistent to the MOOCs vision is known as {\em peer grading} \citep{KWL+13,PHC+13,W14}, according to which the grading task is outsourced to the students that participated to the exam themselves. This approach has been already implemented in some MOOCs, and standalone experimental tools such as \texttt{crowdgrader.org}~\citep{dAS14}, \texttt{peergrading.org}~\citep{RJ14}, and our own \texttt{co-rank}\footnote{Available at \texttt{co-rank.ceid.upatras.gr}.}~\citep{CKPV16} are already available. Even though the approach seems straightforward, there are subtle implementation issues. For example, allowing the students to use cardinal scores is problematic, since they participate both in the exam and in grading and they may have incentives to assign low grades in order to improve their personal relative performance. Even if we assume that they grade honestly, their experience in doing so is very limited and the result will most probably be unreliable.

In this paper, we focus our attention on {\em ordinal peer grading}, which has recently received attention in the AI and machine learning community~\citep{CKV15,RJ14,SBP+13}. Following the setting that we considered in our previous work~\citep{CKV15}, each student gets a bundle of a small number (our favourite number that we have extensively used recently is $6$) of exam papers so that each exam paper is given to the same number of students. Each student has to {\em rank} the exam papers in her bundle (in terms of quality) and an {\em aggregation rule} will then combine the (partial) rankings submitted by the students and come up with a final ranking of all exam papers; this will be the grading outcome.\footnote{We remark that ordinal peer grading has also been used ---in a smaller scale--- in the evaluation of proposals for research funding, e.g., by the Sensors and Sensing Systems (SSS) program of NSF in 2013 \citep{H13}, using a Borda-like method proposed earlier by \citet{MS09}; see also~\citep{KLMP15}.} Information about the position of a student in the final ranking (e.g., top $10\%$ out of $33\,000$ students) can be included in her verified certificate.

In our previous related work~\citep{CKV15}, we formally proved that a simple aggregation rule, inspired from Borda's rule from social choice theory~\citep{BCELP16}, recovers correctly an expected fraction of $1-\bigO(1/k)$ of the pairwise relations in the underlying {\em ground truth} ranking, when bundles of size $k$ are used and students make no mistakes when grading. The assumption for a ground truth and the comparison of the grading outcome to it is similar in spirit to recent approaches that combine voting and learning~\citep{SPX14,BM08,CPS16,CPS14,CK12,CS05,LB11b,MPC13,P13,X14,XC11,Y88}. The new aspect in \citep{CKV15}, as well as in the current paper, the recent paper of our group~\citep{CCKV17}, and the papers by \citet{SW17}, \citet{WGS16}, and \citet{WJJ13}, is the relaxed requirement of recovering the ground truth only {\em approximately}.
Simulation results in \citep{CKV15} show that Borda has very good performance in an imperfect grading scenario inspired by a noisy model of generating random rankings that has been proposed by \citet{Mallows}. Note that, unlike other studies \citep{GWL16,RJ14,SBP+13}, we investigate the potential of applying ordinal peer grading exclusively, without involving any professionals in grading. 

We remark that theoretical analysis in \citep{CKV15} requires to handle with extra care dependencies between several random variables that appear due to the distribution of exam papers to bundles. The analysis of Borda was possible only due to its particular definition; until the current paper, we had not managed to extend the analysis to any other aggregation rule. Also, the $\bigO$ notation in the theoretical guarantee for Borda above hides large constant terms that constitute the bound of theoretical interest only. We follow a different approach here. We would like to develop a ``theory'' for determining the performance of Borda with the highest possible accuracy and, more importantly, extend our study to more aggregation rules.

We define and study a large class of {\em simple} aggregation rules, which we call {\em type-ordering aggregation rules}. A type-ordering aggregation rule determines the position each exam paper has in the final ranking, based only on the ranks each paper has in the bundles that contain it. This class includes Borda. We present a theoretical framework for assessing the performance of each member of this class with respect to a series of performance objectives. A crucial step in our study is that we have completely neglected the dependencies between the random variables that make the rigorous analysis difficult. This sacrifice of mathematical rigor is formally incorrect but makes sense (a rigorous proof is given in appendix) when the number of students tends to infinity; this can be justified by the massive participation in MOOCs. But the best justification of our approach is that the theoretical predictions of performance are experimentally shown --- through extensive simulations --- to be {\em exact}. This apparently means that the dependencies between random variables have no positive or negative impact on performance. Furthermore, once (statistical) information about the grading behaviour of students and the desired performance objectives are known (both in specific formats, which are introduced later in Section~\ref{sec:rules}), our framework can serve as an {\em optimization toolkit} for selecting the optimal type-ordering aggregation rule. This requires an exact solution to an instance of the {\em feedback arc set} problem which, albeit NP-hard in general~\citep{A06}, can be solved exactly for the instances that do arise.

Our theoretical framework allows us to obtain a series of results. For example, we establish that Borda is the optimal type-ordering aggregation rule when students act as perfect graders. This is rather surprising, since Borda is among the simplest aggregation rules in the class we consider. Even though it was not observed to be optimal in any other scenario we considered, its performance is always  close to optimality. Furthermore, as mentioned above, the optimization task of deciding the optimal aggregation rule strongly depends on the information about grading behaviour. We study how inaccuracies of this information affect the choice of the optimal aggregation rule and its performance for the Mallows model as well as for a simple {\em random utility} model~\citep{APX12}. The results suggest a very minor impact and, essentially, a tiny sample of a student population is enough for building a fairly accurate model of grading behaviour.

Overall, our approach combines theory, simulations, and experimentation and is presented graphically in Figure \ref{fig:approach}. The lower chain of the figure describes what one would expect from a simulated exam. There is a student population and some of them participate in an exam. The preparation level of the students that determines their performance in the exam is a random variable following a uniform probability distribution. After the exam, each student acts as the grader of a small number of exam papers submitted by other students. The grading performance typically depends on the preparation level as well. The grades are combined using the aggregation rule and the final ranking is compared to the ground truth to come up with the observed performance. 

\begin{figure*}[ht]
\centering
\includegraphics[trim=0 130 0 196, clip=true, scale=0.5]{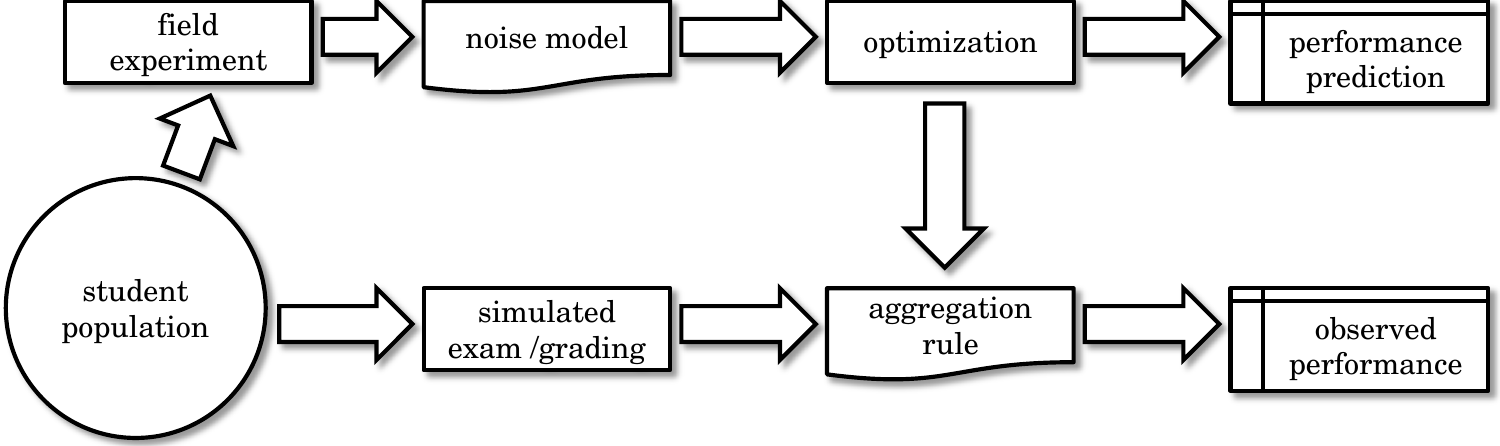}
\caption{A graphical overview of our approach.}
\label{fig:approach}
\end{figure*} 

The most interesting part of Figure \ref{fig:approach} is the upper chain. First, a {\em field experiment} can be used to extract information about the student population. We have performed such field experiments with students in our home institution; we describe them in detail and present the collected data later in the paper. These data are used to build {\em noise models} which, together with the desired performance objective, are given as input to the optimization engine. The optimal aggregation rule for the particular scenario is then constructed, and a theoretical prediction about the performance the rule is expected to have is reported. The optimal aggregation rule can also be applied to the grades from our simulated exams (hence, the downward arrow in Figure \ref{fig:approach}) and a comparison of the theoretically predicted performance with the observed performance of the simulated exam can validate our theory.

The rest of the paper is structured as follows. We begin with a description of the basic tasks that support ordinal peer grading and related preliminary definitions and notation in Section~\ref{sec:prelim}. The type-ordering aggregation rules and our theoretical framework are presented in Section~\ref{sec:rules}. The field experiments and the validation of our framework are then presented in Section~\ref{sec:exp}. We conclude in Section~\ref{sec:open} with a short discussion on future research directions.


\section{Preliminaries} \label{sec:prelim}
We assume that $n$ students have participated in an exam and have submitted their exam papers. Our approach to ordinal peer grading has three distinct tasks: the distribution of papers to students, the grading task by each student, and the aggregation of the grades into a final result. We describe these tasks in detail here and give definitions that will be useful later.

\subsection{Distributing the exam papers}\label{subsec:bundles}
All students that participated in the exam will have to participate in grading as well. The goal of the first task is to balance their grading load. This is done by distributing (copies of) each paper to the students so that each exam paper is given to exactly $k$ students and each student receives exactly $k$ (distinct) exam papers. The $k$ papers that a student receives form her {\em bundle}. These are the exam papers which the student has to grade. Crucially, the bundle of a student should not contain her own exam paper.

A $k$-regular bipartite graph $G=(U,V,E)$ with $n$ nodes on each side of the bipartition (called bundle graph) can be used to represent the distribution of exam papers to students. Each node of sets $U$ and $V$ represents a student. An edge of the graph $G$ between a node $u\in U$ and a node $v\in V$ indicates that the exam paper of the student corresponding to node $u$ is in the bundle of the student corresponding to node $v$. The restriction on the degree of the nodes of set $U$ means that each exam paper is given to exactly $k$ students and the restriction on the degree of the nodes of $V$ means that all bundles have size $k$.

In \citep{CKV15}, we considered bundle graphs that satisfy a particular structural property, namely they contain no cycle of length $4$. This was a technical constraint, required only in theoretical analysis. Simulation results in that paper indicate that uniformly random $k$-regular bipartite graphs are almost as good as bundle graphs. These are the bundle graphs we consider in the current work. A random $k$-regular graph can be built as follows. Starting from the complete bipartite graph $K_{n,n}$ with node sets $U$ and $V$, first remove the edges between nodes corresponding to the same student in $U$ and $V$. Then, draw a perfect matching uniformly at random among all perfect matchings of $K_{n,n}$ that do not include previously removed edges. The edges in the $k$ perfect matchings obtained by repeating the above step $k$ times form the bundle graph.\footnote{Equivalently, we can start from $K_{n,n}$, obtain $k$ perfect matchings (by selecting each of them uniformly at random among all perfect matching of $K_{n,n}$ that do not include edges that have been included in previous matchings), and then rename the nodes in one side of the bipartition so that no student is assigned a bundle that contains her exam paper. This alternative process is used in our formal analysis in Appendix~\ref{app:sec:formal}.}

\subsection{Modelling the grading task}\label{subsec:grading}
Throughout the paper, we assume that there is an underlying strict ranking of the exam papers, the {\em ground truth}, which we aim to recover. As it will shortly become apparent, the setting we consider is so restrictive that we should not expect to recover the ground truth exactly. Instead, we aim to recover the ground truth {\em approximately}. 

A restriction of our setting is that each student is given only $k$ exam papers to grade. Another restriction is that the grading task for each student is simply to rank the exam papers in her bundle, in decreasing order of quality. We consider different scenarios for the grading behaviour of the students. In a first scenario, we assume that, after the end of the exam, the instructor announces indicative solutions and gives detailed instructions that the students can use during grading. Here, we assume that students will act as {\em perfect graders}. Admittedly, this is an unrealistic assumption but we include it as an extreme case in our study together with many others.

In a second scenario, we assume that the students receive no solutions or grading guidelines by the instructor after the end of the exam. In this case, students will inevitably make mistakes when grading and it is reasonable to assume that the performance of a student in grading is strongly correlated to her preparation level and her performance in the exam. We will use the term {\em imperfect grading} to refer to this scenario.

In the study of imperfect grading scenarios, we will consider student populations with different characteristics. In the first such population, each student has a quality drawn uniformly at random from the interval $[1/2,1]$, which affects her position in the ground truth and her ability to grade as well. The ground truth is the ranking of the students in decreasing order of quality. A student $b$ of quality $q$ performs the grading task as follows: she considers every pair of exam papers $x$ and $y$ in her bundle, such that $x$ appears ahead of $y$ in the ground truth, and temporarily determines $x \succ_b y$ with probability $q$ and $y \succ_b x$ with probability $1-q$; the pairwise relation $\succ_b$ will evolve into her ranking of the exam papers in her bundle. If, after considering all pairs of exam papers in the bundle, the pairwise relation $\succ_b$ is cyclic, the whole process is repeated from scratch. Otherwise, the ranking of the exam papers in the bundle induced by $\succ_b$ is the grading outcome of student $b$. Due to its similarities with the well-known Mallows model~\citep{Mallows} for generating random rankings, we refer to this grading behaviour as {\em Mallows grading}.

In another interesting type of student population, grading behaviour follows the general structure of random utility models in the literature; e.g., see \citet{APX12}. Each student has a quality drawn uniformly at random from the interval $[0,1]$. The ground truth is again defined as the ranking of the students in decreasing order of quality. A student $b$ of quality $q$ performs the grading task by assigning a score to every exam paper $x$ of quality $q_x$ in her bundle as follows: with probability $q$ she sets the score of $x$ equal to $q_x$ and with probability $1-q$ the score is drawn uniformly at random from the interval $[0,1]$. Then, the ranking of the exam papers is computed by sorting them in non-increasing order of these scores. We use the term {\em RUM grading} to refer to this grading behaviour.

The two paragraphs above describe how the behaviour of populations of Mallows and RUM graders is simulated in the experiments that we discuss in Sections~\ref{subsec:optimal-exp} and~\ref{subsec:apx}. Admittedly, these two populations are very stylized. We will introduce two more in Section~\ref{subsec:real-exp}, which are closer to the grading behaviour of real students.

\subsection{Aggregation rules}
The third important task is to aggregate the partial rankings provided by the graders into a final output ranking. This is done using an {\em aggregation rule}. A simple but very compelling aggregation rule is inspired by the Borda count voting rule. In our context, Borda computes a score for each exam paper by examining the positions it has in the rankings of the graders that have this exam paper in their bundles. A first position by an exam paper contributes $k$ points to its score, a second position contributes $k-1$ points, and so on. The outcome of Borda is a ranking of the exam papers in non-increasing order in terms of their Borda scores. When we use Borda, we assume that ties are broken uniformly at random but other tie-breaking schemes could be considered as well. 

In our previous work \citep{CKV15}, we also considered several other aggregation rules such as a rule that we call Random Serial Dictatorship (RSD) as well as rules that are based on appropriately defined Markov chains, motivated by early work on rank aggregation on the web \citep{DKN+01,PBM+99}. RSD is very slow in the computation of the final outcome and, even though it performs remarkably well with perfect graders, it has a poor performance in simulated exams with Mallows graders. We will not consider it in the current paper; actually, applying it with input from $10\,000$ graders, which is the typical scenario we consider in this paper, is a computational challenge. The aggregation rules that are based on Markov chains were defined in an unsuccessful attempt to distinguish between high and low quality graders and put more weight on the partial rankings of the former. These ideas are not considered in this work either. 

\citet{RJ14} use optimization (stochastic gradient descent) methods that yield aggregation rules which are maximum likelihood estimators with respect to the cardinal scores of exam papers that are supposed to be part of the ground truth. Since we assume that the ground truth is just a ranking of all exam papers, such methods are not applicable in our case.
Instead, we focus on much simpler aggregation rules.


\section{Type-ordering aggregation rules and their theoretical analysis}\label{sec:rules}
We will use the term {\em type} to refer to the grading result for an exam paper. Its type consists of the ranks the exam paper gets from the $k$ graders that have it in their bundles. So, the type is a vector of $k$ integers from $[k]=\{1, 2, ..., k\}$. We follow the convention that the $k$ entries in types appear in monotone non-decreasing order. We use 
$$\calT_k=\{\sigma=(\sigma_1, \sigma_2, ..., \sigma_k)| 1\leq \sigma_1 \leq \sigma_2 \leq ... \leq \sigma_k\leq k\}$$
to denote the set of all types for bundle size $k$. It is not hard to see that $\calT_k$ contains ${2k-1\choose k}$ different types.

As an example with $k=6$, an exam paper of type $(1,2,2,2,2,5)$ is ranked first by one of its graders, second by four graders, and fifth by one grader. Now, consider another exam paper of type $(2,2,2,2,3,3)$ and observe that Borda would give the same Borda score of $28$ to both exam papers. Is there some  particular reason for which these two exam papers should be very close in the final ranking? Now, consider the two types $(1,1,1,2,5,6)$ and $(2,2,2,3,3,3)$ of Borda scores $26$ and $27$, respectively. Borda indicates that an exam paper with the second type is better. But looking carefully at the ranks, we could come up with the following interpretation. The first exam paper is very good (and most probably in one of the two top positions in any bundle) and the two low ranks are due to poor judgement by the graders. In contrast, the second exam paper is just above average and this is reflected in all grades. Of course, such interpretations are valid only when they can be supported by information about the graders (e.g., about the frequency with which they make mistakes). But, certainly, there are cases where such interpretations are indeed valid.

So, it seems that Borda is restrictive; then, one would think that this is due to the particular scores that Borda uses. We will not investigate whether different scores could yield better results. This, in a slightly different context, is the subject of another recent paper of our group \citep{CCKV17}. Instead, we will define a much broader class of aggregation rules. A {\em type-ordering aggregation rule} uses a strict ordering $\succ$ of the types in $\calT_k$. Then, the final ranking of the exam papers follows the ordering $\succ$ of their types, breaking ties uniformly at random. In general, rules of this class seem to be very powerful. Compared to Borda which partitions the set of exam papers into only $k^2-k+1$ different scores, a type-ordering aggregation rule can distinguish between exponentially many (in terms of $k$) different types. In the following, we use the term {\em Borda ordering} to refer to any ordering of the types in non-increasing order of Borda score. We also use $B(\sigma)$  to denote the Borda score of an exam paper with type $\sigma=(\sigma_1, ..., \sigma_k)$. Clearly, $B(\sigma)=\sum_{i=1}^k{(k+1-\sigma_i}) = k^2+k-\sum_{i=1}^k{\sigma_i}$.

We remark that the use of types in the definition of a broad class of aggregation rules has been possible due to the regularity that we imposed on the bundles and the distribution of exam papers to them. Of course, this creates issues related to the theoretical analysis of these rules (such as dependencies between the random variables involved in the distribution to bundles and in grading). In the next section, we discuss how to overcome such issues by making several simplifying assumptions. A (much more involved) rigorous analysis that justifies these assumptions is presented in Appendix~\ref{app:sec:formal}.

\subsection{A framework for theoretical analysis}\label{subsec:framework}
For the analysis of type-ordering aggregation rules, we will assume an {\em infinite} number of students. This is close to the vision of MOOCs with huge numbers of enrolled students and is the important assumption that constitutes the theoretical analysis possible. So, the positions of students in the ground truth ranking can be thought of as occupying the continuum of the interval $[0,1]$ with uniform density. We will usually identify an exam paper as a real number $x\in [0,1]$, i.e., by its rank in the ground truth ranking.\footnote{Notice that the interval $[0,1]$ is used only to represent the rank of a student and not some kind of absolute cardinal quality. In our analysis, the only information we infer from two students with ranks $x$ and $y$ with $x<y$ is that $x$ is better than $y$; we make no additional assumption about the difference in quality between the two students. This comes in contrast to assumptions by \citet{RJ14}, who assume that cardinal scores are part of the ground truth as well.}
Furthermore, we will assume that in each of the $k$ bundles to which exam paper $x$ belongs, the remaining $k-1$ exam papers are selected uniformly at random with replacement from the student population. Our assumption of infinitely many students allows us to ignore subtleties such as the requirement that all students in a bundle should be distinct and also different than the student that acts as the grader of the bundle (the probability that this requirement will not be satisfied in some bundle is zero).\footnote{Admittedly, this analysis is non-rigorous. A formal analysis should assume a finite number of students, take into account all dependencies between random variables that we neglect here, and conclude that these dependencies vanish as the number of students approaches infinity. Such a rigorous analysis is presented in Appendix \ref{app:sec:formal}.}

In our theoretical modelling of imperfect grading, we make further simplifying assumptions. In particular, we ignore the fact that grading behaviour is correlated to student quality, and instead assume {\em independence} of the two characteristics. The ground truth is selected uniformly at random among all possible rankings of all students. Equivalently, this can be thought of as selecting independently the quality of each student uniformly at random from a given interval, and then sorting the students in non-increasing order in terms of these qualities. 
Grading behaviour of the students is independent of quality. When a student receives a bundle of exam papers, she draws a random ranking of them according to a probability distribution that characterizes the grading behaviour of {\em all} students participating in the exam.
In particular, the behaviour of each grader is characterized by a $k\times k$ {\em noise matrix} $P =(p_{i,j})_{i,j\in [k]}$, where $p_{i,j}$ denotes the probability that the exam paper with correct rank $j$ among the $k$ exam papers in a bundle is ranked at position $i$ by the grader.

Clearly, a noise matrix is doubly stochastic, i.e., the sum of the entries in any column and any row is equal to $1$. Observe that the corresponding noise matrix for perfect grading is the $k\times k$ identity matrix. We will often use the term {\em noise model} as a synonym of the term noise matrix. Note that a noise matrix provides only aggregate information for all students of a population. Furthermore, this information is actually rough, as it is not hard to see that a doubly stochastic matrix may correspond to many different probability distributions over rankings. 

Consider an aggregation rule that uses an ordering $\succ$ of the types defined by bundles of size $k$ and is applied to partial rankings provided by graders whose behaviour follows the noise model $P$. Let us focus on computing the expected number of pairwise relations in the ground truth ranking that are correctly recovered in the outcome of the rule. It suffices to consider every pair of exam papers $x,y\in [0,1]$ with $x<y$ (i.e., exam paper $x$ has a better rank in the ground truth compared to exam paper $y$) and add one point if $x$ has a better type than $y$ according to the ordering $\succ$, and half a point if both exam papers have the same type. In this last case, the tie is resolved uniformly at random and the probability that the correct pairwise relation will be recovered is $1/2$. Hence, denoting by $C$ the expected\footnote{Here, the expectation is taken over the randomness in the assignment of exam papers to bundles, in the student grading, as well as in the resolution of ties.} fraction of pairwise relations recovered by the rule (we will refine this notation in a while), and by $x\rhd \sigma$ the event that exam paper $x$ gets type $\sigma$ after grading, we have
\begin{align*}
C &= \int_0^1{\int_x^1{\left(\sum_{\sigma,\sigma':\sigma\succ\sigma'}{\Pr[x \rhd \sigma \text{ and } y \rhd \sigma']}+\frac{1}{2}\sum_{\sigma}{\Pr[x \rhd \sigma \text{ and } y \rhd \sigma]}\right)\ud y}\ud x}\\
&= \sum_{\sigma,\sigma':\sigma\succ\sigma'}{\int_0^1{\int_x^1{\Pr[x \rhd \sigma \text{ and } y \rhd \sigma']\ud y}\ud x}}+\frac{1}{2}\sum_{\sigma}{\int_0^1{\int_x^1{\Pr[x \rhd \sigma \text{ and } y \rhd \sigma]\ud y}\ud x}}
\end{align*}
The first sum runs over all pairs of different types $\sigma, \sigma'$ of $\calT_k$ with order $\sigma\succ\sigma'$ and the second sum runs over all types. The scary (at first glance) double integral can be hidden under the notation $W(\sigma,\sigma')$ to obtain
\begin{align}\label{eq:sum-of-weights}
C &= \sum_{\sigma,\sigma':\sigma\succ\sigma'}{W(\sigma,\sigma')}+\frac{1}{2}\sum_\sigma{W(\sigma,\sigma)}.
\end{align}
We will use the term {\em weight} to refer to the quantity $W(\sigma,\sigma')$. Our assumption for an infinite number of students nullifies any 
dependencies between the types
that exam papers $x$ and $y$ get after grading. So, the events $x\rhd \sigma$ and $y\rhd \sigma'$ are independent and the definition of the weight $W(\sigma,\sigma')$ becomes
\begin{align}\label{eq:weight}
W(\sigma,\sigma') &= \int_0^1{\int_x^1{\Pr[x \rhd \sigma]\cdot \Pr[y \rhd \sigma']\ud y}\ud x}.
\end{align}

Let us now compute the probability that exam paper $x$ gets type $\sigma=(\sigma_1, ..., \sigma_k)$. By considering all ways to distribute the entries of the type vector as ranks of an exam paper by the graders that handle it (ignoring symmetries), there are 
$$N(\sigma)=\frac{k!}{d_1!\cdot ... \cdot d_k!}$$
ways that the exam paper can get type $\sigma$, where $d_i$ is the number of graders that have the exam paper ranked $i$-th. Again, due to our assumption for infinitely many students and the uniform inclusion of them into bundles, the quality of each exam paper included in a bundle does not affect the quality of other exam papers (in the same or different bundles). Clearly, the grading by different students is performed without dependencies either. Denoting by $\EEE(x,\sigma_i)$ the event that exam paper $x$ is ranked $\sigma_i$-th in a bundle, the probability that $x$ is of type $\sigma$ is
\begin{align*} 
\Pr[x \rhd \sigma] &= N(\sigma) \prod_{i=1}^k{ \Pr[\EEE(x,\sigma_i)] }.
\end{align*}
To compute $\Pr[\EEE(x,\sigma_i)]$, it suffices to consider all possible true ranks that exam paper $x$ may have in a bundle and account for the probability of having such a rank and being ranked $\sigma_i$-th by the grader that is handling the bundle. Let us denote by $\EEE^*(x,j)$ the event that the true rank of $x$ in a bundle is $j$. Then,
\begin{align*} 
\Pr[x \rhd \sigma] &= N(\sigma) \prod_{i=1}^k \sum_{j=1}^k {  p_{\sigma_i,j} \Pr[\EEE^*(x,j)] }.
\end{align*}
Now, the probability $\Pr[\EEE^*(x,j)]$ is equal to the number of ways we can choose $j-1$ exam papers to be ahead of $x$, times the probability that all of them will indeed be ahead of $x$ in the bundle, times the probability that the rest $k-j$ exam papers in the bundle will have true ranks worse than $j$. We use $L_k$ to denote the set of all $k$-entry vectors $\ell=(\ell_1, ..., \ell_k)$ with $\ell_i\in [k]$ and, for compactness of notation, we abbreviate $\sum_{i=1}^k{\ell_i}$ by $|\ell|_1$. We have
\begin{align}\label{eq:simple}
\Pr[x \rhd \sigma] &= N(\sigma) \prod_{i=1}^k{ \sum_{j=1}^k{ p_{\sigma_i,j} {k-1 \choose j-1}x^{j-1}(1-x)^{k-j} }} \\\nonumber
&=  N(\sigma) \sum_{\ell\in L_k}{\prod_{i=1}^k{p_{\sigma_i,\ell_i} {k-1 \choose \ell_i-1}x^{\ell_i-1}(1-x)^{k-\ell_i} }} \\\label{eq:theta}
&= N(\sigma) \sum_{\ell\in L_k}{\bigg(\prod_{i=1}^k{ p_{\sigma_i,\ell_i} {k-1 \choose \ell_i-1}} \bigg)  x^{|\ell|_1-k}(1-x)^{k^2-|\ell|_1}}, 
\end{align}
where the second equality is obtained by exchanging the sum and product operators. Using the fact that $(1-x)^m = \sum_{j=0}^m{ {m \choose j}(-1)^jx^j}$ for $m=k^2-|\ell|_1$, we obtain
\begin{align}\nonumber
\Pr[x \rhd \sigma] &= N(\sigma) \sum_{\ell\in L_k}{\bigg( \prod_{i=1}^k{ p_{\sigma_i,\ell_i} {k-1 \choose \ell_i-1}} \bigg)  x^{|\ell|_1-k} \sum_{j=0}^{k^2-|\ell|_1}  { {k^2-|\ell|_1 \choose j}(-1)^jx^j }} \\ \label{eq:prob-x-sigma-is-a-poly}
&= N(\sigma) \sum_{\ell\in L_k}{\sum_{j=0}^{k^2-|\ell|_1}{\bigg( \prod_{i=1}^k{ p_{\sigma_i,\ell_i} {k-1 \choose \ell_i-1}} \bigg) { {k^2-|\ell|_1 \choose j}(-1)^j  x^{|\ell|_1-k+j} } }}.
\end{align}
Interestingly, $\Pr[x \rhd \sigma]$ is a univariate polynomial of degree $k^2-k$. Then, the double integral in equation (\ref{eq:weight}) can be computed analytically. The computation is tedious but straightforward; see Appendix \ref{app:sec:compute-integral}.

\subsection{Computing optimal type-ordering aggregation rules}\label{subsec:optimal}
The approach in Section \ref{subsec:framework} suggests a general way of evaluating the performance of any type-ordering aggregation rule. In order to compute the expected number of correctly recovered pairwise relations, it suffices to use equations (\ref{eq:sum-of-weights}), (\ref{eq:weight}), and (\ref{eq:prob-x-sigma-is-a-poly}). Equation (\ref{eq:prob-x-sigma-is-a-poly}) can be used to obtain $\Pr[x\rhd \sigma]$, which is then used in equation (\ref{eq:weight}) to compute the weights (for any possible pair of types $\sigma$ and $\sigma'$). Finally, equation (\ref{eq:sum-of-weights}) returns the expected number of correctly recovered pairwise relations. 

Of course, the expected number of correctly recovered pairwise relations is not the only performance objective one would like to measure. For example, we could simply ignore exam papers that are very close to each other in the ground truth ranking. The ground truth ranking is mostly a modelling assumption and it should not be very restrictive in the evaluation of an aggregation rule. So, we could just measure the expected number of correctly recovered pairwise relations between pairs of exam papers with ranks in the ground truth that differ by at least $a\%$ (for small values such as $5\%$). Another possibility would be to ignore pairwise relations between pairs of exam papers that have both very low rank in the ground truth. For example, why is it important to recover correctly the pairwise relation between the students that have true ranks $80\%$ and $95\%$? A general objective in this direction would be to measure the correctly recovered relations between pairs of exam papers that involve one with true rank in the top $a\%$ (e.g., $20\%$). 

Our theoretical framework can be easily extended to handle such cases using many different performance objectives. In general, a {\em bivariate} performance objective is defined by a bivariate function $f:[0,1]^2\rightarrow [0,1]$ which returns the importance of measuring a correctly recovered relation between two students $x$ and $y$ with $x\leq y$. In the presentation of our framework in Section \ref{subsec:framework}, we have assumed such a function with $f(x,y)=1$ for every pair of students. The two scenarios of the previous paragraph can be captured by the function
(i) $f(x,y)=1$ when $y-x\geq a\%$ and $f(x,y)=0$ otherwise, and (ii) $f(x,y)=1$ when $x\leq a\%$ (and $x\leq y$) and $f(x,y)=0$ otherwise. Many other performance objectives can be defined including ones in which the function $f$ returns fractional values between $0$ and $1$.

The only modification in the computation of Section \ref{subsec:framework} is in the computation of the weights which should now become 
\begin{align}\label{eq:new-weight}
W(\sigma,\sigma') &= \int_0^1{\int_x^1{f(x,y)\cdot \Pr[x \rhd \sigma]\cdot \Pr[y \rhd \sigma']\ud y}\ud x}.
\end{align}
In order to capture the generality of the scenarios considered, we overload the notation for the performance measure $C$ to specify the bundle size $k$, the aggregation rule $\succ$, the noise matrix $P$ describing the grading behaviour, and the bivariate performance objective $f$. 

\begin{theorem}\label{thm:sum-of-weights-refined}
Consider a type-ordering aggregation rule $\succ$ that is applied on $k$-sized partial rankings from an infinite population of students with grading behaviour that follows a noise matrix $P$. Then, the fraction of correctly recovered pairwise relations that satisfy the performance objective given by the bivariate function $f$ is 
\begin{align}\label{eq:sum-of-weights-refined}
C(k,\succ,P,f) &= \sum_{\sigma,\sigma':\sigma\succ\sigma'}{W(\sigma,\sigma')}+\frac{1}{2}\sum_\sigma{W(\sigma,\sigma)},
\end{align}
where $W(\sigma,\sigma')$ is given by (\ref{eq:new-weight}) and $\Pr[x\rhd \sigma]$ is in turn given by (\ref{eq:prob-x-sigma-is-a-poly}).
\end{theorem}

Note that the weights do not depend on the aggregation rule at all. They depend on the grading behaviour and the bivariate performance objective. Instead, the aggregation rule determines only the particular weights that should be summed up in order to compute $C(k,\succ,P,f)$. This means that, once we have information about the bundle size, the grading behaviour, and the desired bivariate performance objective, we can seek for the type-ordering aggregation rule that is optimal for this particular scenario. All we have to do is to compute the type-ordering aggregation rule $\succ$ that maximizes $C(k,\succ,P,f)$ which, actually, translates to computing an ordering of the types so that the leftmost summation in the definition (\ref{eq:sum-of-weights-refined}) is maximized.

It is not hard to see that the problem is equivalent to solving a maximization variant of the \textit{feedback arc set} (FAS) problem. On input a complete directed graph $G=(V,E,w)$ with non-negative edge weights, the objective of FAS is to find an ordering $\succ$ of the nodes of $G$  such that $\sum_{u,v:u\succ v}{w(u,v)}$ (i.e., the total weight of ``consistently directed'' edges with respect to $\succ$) is maximized. In our case, the input is a complete directed graph that has a node for each type $\sigma\in \calT_k$. A directed edge from a node corresponding to type $\sigma$ to a node corresponding to type $\sigma'$ has weight $W(\sigma,\sigma')$.  The next statement should now be obvious.

\begin{theorem}\label{thm:hardness}
Computing the optimal type-ordering aggregation rule for a scenario involving an infinite population of students, specific bundle size, grading behaviour, and desired bivariate performance objective is equivalent to solving feedback arc set on an edge-weighted complete directed graph.
\end{theorem}

FAS is NP-hard even in its very simple variant on unweighted tournaments \citep{A06}. The particular weighted version we consider here admits a PTAS \citep{KS07}. Unfortunately, the solutions that such a PTAS can guarantee in reasonable time are quite far from optimality and the resulting type-ordering aggregation rule will consequently have highly suboptimal performance. Fortunately, the FAS instances that we had to solve in order to compute optimal rules have a very nice structure for all the scenarios considered. This structure allows us to compute the optimal FAS solution (almost) exactly by a straightforward algorithm that we present in the following. We strongly believe that this nice property holds in {\em any} scenario that can appear in practice.

Let us assume that we would like to solve FAS on an edge-weighted complete directed graph $G=(V,E,w)$ and to compute an ordering of the nodes of $V$ so that the total weight of edges in the direction that is consistent to the ordering is as high as possible. First observe that if two opposite directed edges have the same weight, the ordering of its endpoints does not affect the contribution of the consistently directed edge. So, the decision about the relative order of such {\em non-critical} node pairs can be postponed until the very end of the algorithm and any decision about them will be just fine. Now, consider two nodes $u$ and $v$ of $G$ such that $w(u,v)>w(v,u)$; then, the consistently directed edge that we would like to have in the final solution is $(u,v)$. We will call such pairs of nodes {\em critical} pairs. Decisions about the ordering of critical pairs of nodes have to be taken first. An ideal situation would be if after deciding the critical node pairs, we came up with a partial ordering of all nodes that participate in at least one critical pair. The ordering could then be completed by appropriate decisions about non-critical node pairs. And, luckily, this process would have resulted in an optimal solution for FAS since every pair of nodes would have the maximum possible contribution to the objective. Of course, things are not as easy in general since the decisions about critical pairs may lead to cycles of nodes, which cannot be part of the final ordering.

Our algorithm proceeds as follows. It takes as input an edge-weighted complete directed graph $G=(\calT_k,E,W)$ with $\calT_k$ as the node set and weight $W(\sigma,\sigma')$ (computed using (\ref{eq:new-weight})) for every directed edge from type $\sigma$ to type $\sigma'$. Our algorithm builds an auxiliary unweighted directed graph $H=(\calT_k,A)$ again over the types. For every critical pair of types $\sigma,\sigma'$ with $W(\sigma,\sigma')>W(\sigma',\sigma)$, the auxiliary graph has a directed edge from type $\sigma$ to type $\sigma'$. The next step is to compute all strongly connected components of $H$; two types $\sigma$ and $\sigma'$ belong to the same strongly connected component if $H$ contains a directed path from $\sigma$ to $\sigma'$ and a directed path from $\sigma'$ to $\sigma$. This computation is easily performed by computing breadth first search trees rooted at every node of $H$. After this step, the ordering of the types in different strongly connected components is irrevocably decided. In order to decide the ordering of types within the same strongly connected component, we use brute force on the corresponding subgraph of $G$. If the size of a strongly connected component is so large that brute forcing is prohibitive, we just order the types within the component according to a Borda ordering (breaking ties uniformly at random). As a final straightforward step, we decide the order of non-critical node pairs.

The approach to use Borda ordering when brute forcing is very costly in terms of running time, might give the impression that the outcome of the above algorithm is always very close to a Borda ordering. Surprisingly, our algorithm returns Borda orderings (or orderings that are very close to Borda) very rarely. One such situation is presented in the next section where we show that Borda is indeed the optimal type-ordering aggregation rule in all scenarios that involve perfect graders. For imperfect graders, brute forcing has been proved extremely useful as the vast majority of strongly connected components are small. We report statistical information from the size distribution of strongly connected components in Section \ref{sec:exp} (see Table \ref{tab:component-size} in Section \ref{subsec:optimal-exp}).

\subsection{Borda is optimal for perfect graders}\label{subsec:borda-proof}
We will now exploit our theoretical framework to obtain our first concrete result.

\begin{theorem}\label{thm:borda-perfect}
For every scenario that involves an infinite population of perfect graders, specific bundle size, and a bivariate performance objective, Borda (with any tie-breaking rule) is the optimal type-ordering aggregation rule.
\end{theorem}

\begin{proof}
Assume that we have a scenario with a bundle size of $k$, perfect grading (i.e., a $k\times k$ identity noise matrix), and the bivariate function $f$ that represents the performance objective. 

We first compute the probability that exam paper $x$ gets type $\sigma$ using (\ref{eq:simple}) and the fact that $p_{\sigma_i,\ell}=1$ if $\sigma_i=\ell$ and $p_{\sigma_i,\ell}=0$ otherwise. Hence, 
\begin{align*}
\Pr[x \rhd \sigma] &= N(\sigma) \prod_{i=1}^{k}{{k-1\choose \sigma_i-1}x^{\sigma_i-1}(1-x)^{k-\sigma_i}}\\
&= N(\sigma) \left(\prod_{i=1}^k{k-1\choose \sigma_i-1}\right) x^{k^2-B(\sigma)}(1-x)^{B(\sigma)-k}.
\end{align*}

Now, consider two exam papers with ranks $x$ and $y$ in the ground truth such that $x<y$ and let $\sigma$ and $\sigma'$ be two types. Using the above equality, we obtain 
\begin{align}\label{eq:RHS}
\frac{\Pr[x \rhd \sigma]\Pr[y \rhd \sigma']}{\Pr[x \rhd \sigma']\Pr[y \rhd \sigma]}=\left(\frac{y(1-x)}{x(1-y)}\right)^{B(\sigma)-B(\sigma')}.
\end{align}
Since $y>x$, it is also $1-x>1-y$ and the right hand side of the last equation is above, equal, or below $1$ if and only if the quantity $B(\sigma)-B(\sigma')$ is positive, zero, or negative. Hence, the quantity $$\Pr[x \rhd \sigma]\Pr[y \rhd \sigma'] - \Pr[x \rhd \sigma']\Pr[y \rhd \sigma]$$ and the Borda score difference $B(\sigma)-B(\sigma')$ between the two types $\sigma$ and $\sigma'$ have the same sign. Now, let $\sgn:\mathbb{R}\rightarrow\{-1,0,1\}$ be the signum function. We have that 
\begin{align*}
&\sgn\left(W(\sigma,\sigma')-W(\sigma',\sigma)\right) \\
&= \sgn\left(\int_{0}^{1}{\int_{x}^{1}{f(x,y)(\Pr[x \rhd \sigma]\Pr[y \rhd \sigma'] - \Pr[x \rhd \sigma']\Pr[y \rhd \sigma])\ud y}\ud x}\right)\\
&= \sgn\left(B(\sigma)-B(\sigma')\right)
\end{align*}
This implies that any Borda ordering $\succ$ of the types maximizes the quantity $\sum_{\sigma,\sigma':\sigma\succ\sigma'}{W(\sigma,\sigma')}$ and, consequently, the quantity $C(k,\succ,I,f)$; the theorem follows.
\end{proof}

The statement of Theorem~\ref{thm:borda-perfect} is rather surprising as Borda is among the simplest type-ordering aggregation rules. For example, when $k=6$, Borda classifies the exam papers into only $31$ different levels (based on their Borda scores) while there are type-ordering rules that exploit a more refined classification of the papers into $462$ different levels; the gap is much higher for larger values of $k$. Theorem~\ref{thm:borda-perfect} essentially says that this extra power is not at all necessary and Borda is always the best choice when perfect grading is used.


\section{Validation of our framework}\label{sec:exp}
In this section, we present our field experiments and simulation results that validate the theoretical framework we developed in Section~\ref{sec:rules}.

\subsection{Building realistic noise models using field experiments}\label{subsec:real-exp}
We have run two field experiments with the students that attended the course on Computational Complexity in the Department of Computer Engineering and Informatics of the University of Patras during the Spring 2015 and Spring 2016 semesters. This is a course that the first author teaches during the last few years and usually includes an optional mid term exam. As it is typically the case in Greek universities, cardinal integer and half-integer scores between $0$ and $10$ are used in such exams and they represent how correct the answers of the students to the exam questions are. Hence, these cardinal scores represent the success of the students in the exam in absolute terms.

In our experiments, our goal has been to investigate how effective the students can be in ordinal grading. For this purpose, we created hypothetical exams with three questions and prepared several answers for them. In particular, for the 2015 experiment, we prepared 16 different answers to question 1, 12 answers to question 2, and 8 answers to question 3. Combinations of these answers into all different ways resulted in a pool of 1536 different exam papers. We created bundles of size $6$ from this pool. Each student was given a bundle of exam papers which was asked to rank (for a bonus grade). Note that the selection of papers in each bundle was not arbitrary. The answers for the questions belonged to different levels of correctness and included excellent ones, almost excellent ones with a minor issue not fully resolved, answers in the right direction but with sloppy write-up, completely incorrect answers, no answer at all, etc. Specifically, we had 7, 6, and 5 different levels of correctness for the answers in questions 1, 2, and 3, respectively. For the 2016 experiment, these numbers were only slightly modified. When bundles were formed, we imposed the following constraint for any pair of exam papers $A$ and $B$ in a bundle: if the correctness level of paper $A$ in an answer is strictly higher than that in paper $B$, then paper $B$ cannot have a strictly higher correctness level than $A$ in any other answer. Furthermore, there was at least one question for which the answers (in $A$ and $B$) had different levels of correctness. This guaranteed a strict ranking of the exam papers in each bundle and, furthermore, that this ranking would be well-defined and independent of any assumptions about the importance of the different questions. 

In this ranking exercise, each student was given a bundle of $6$ exam papers and returned a ranking of them. In addition, the students participated in the traditional mid term exam. This allowed us to quantify the correlation between their grading behaviour and their success in the traditional exam. So, the outcome of each experiment is a list consisting, for each student, of a ranking of the exam papers in her bundle (as a permutation of the correct ranking) together with her performance in the exam. These data are depicted in Tables~\ref{tab:2015} (for the 2015 experiment), \ref{tab:2016a}, and~\ref{tab:2016b} (for the 2016 experiment) in Appendix~\ref{app:sec:data}. Figures \ref{fig:bubbles}(a) and \ref{fig:bubbles}(b) show the correlation between grading error (Kendall-tau distance of the ranking returned by each grader from the correct ranking of the exam papers in her bundle) and quality for the $136$ and $241$ students that participated in the mid term exams in 2015 and 2016, respectively. Observe that, in 2015, the grading performance of the majority of the non-excellent students seems to be uniformly distributed between average and excellent, with just a few under-performing outliers, whereas the picture is more clear in 2016 and the grading performance has improved. An explanation for this grading behaviour is that, even though the students have participated in many exams like the mid term in the past and have a very good idea of what they are expected to do, in 2015, it was the very first time they were asked to rank. In contrast, before the 2016 exam, we made the data we collected in the previous year available in order to help the students prepare for the ordinal grading task as well.

\begin{figure*}[ht]
\centering
\begin{subfigure}{0.44\textwidth}
  \includegraphics[trim=20 200 30 235, clip=true, scale=0.33]{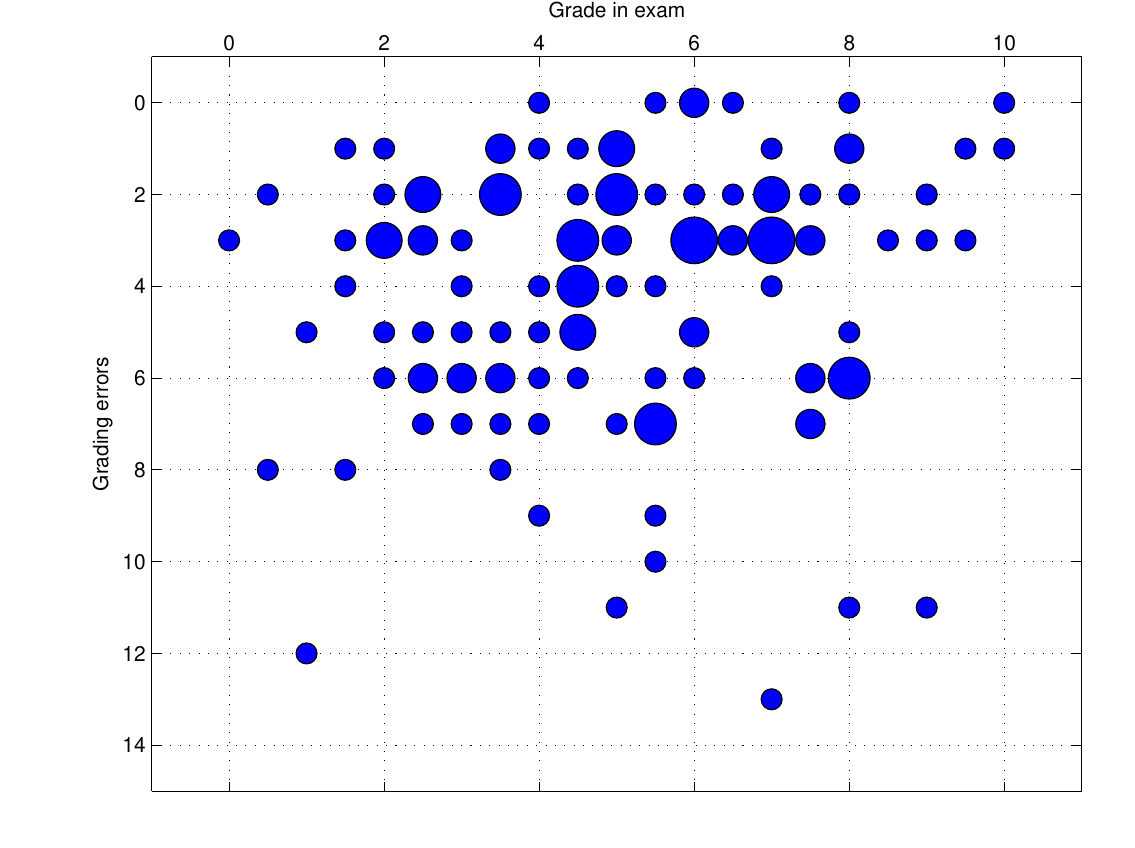}
  \caption{Realistic noise 2015}
\end{subfigure}
\begin{subfigure}{0.44\textwidth}
  \includegraphics[trim=20 200 30 180, clip=true, scale=0.33]{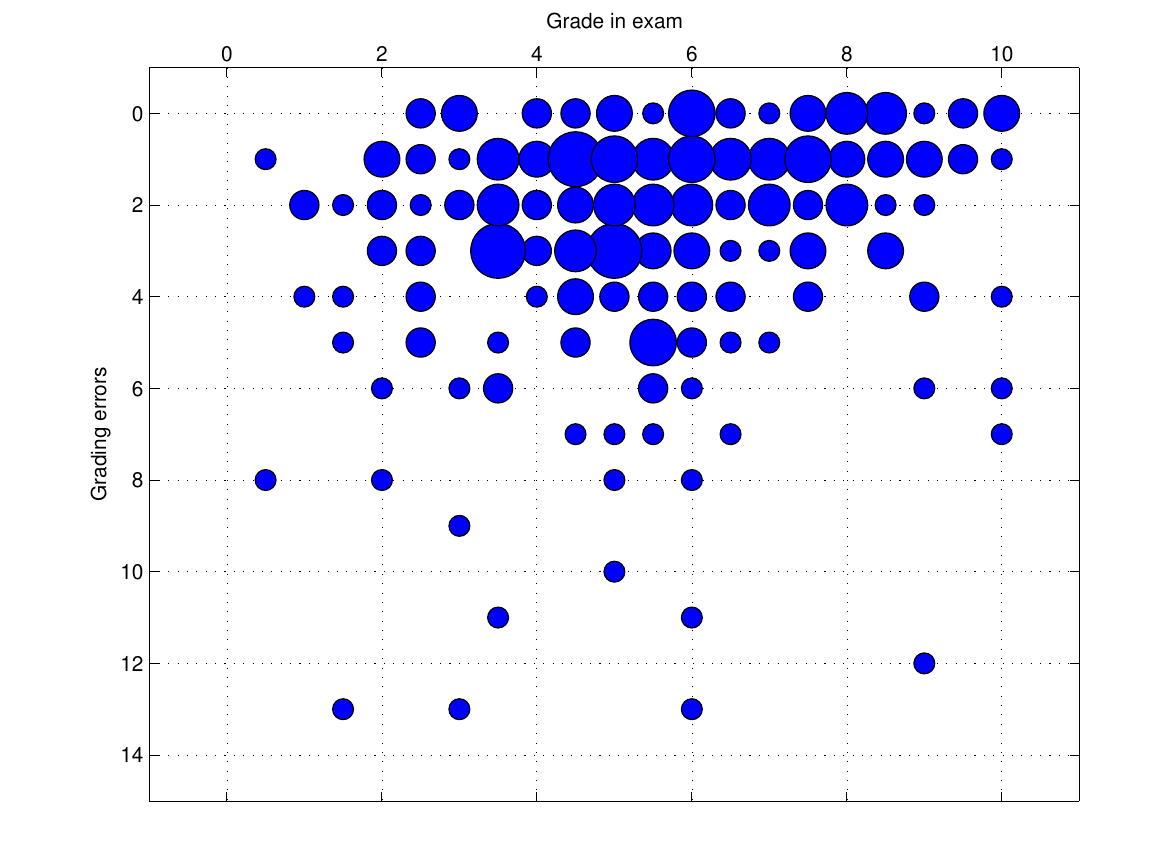}
  \caption{Realistic noise 2016}
\end{subfigure}
\begin{subfigure}{0.44\textwidth}
  \includegraphics[trim=20 200 30 180, clip=true, scale=0.33]{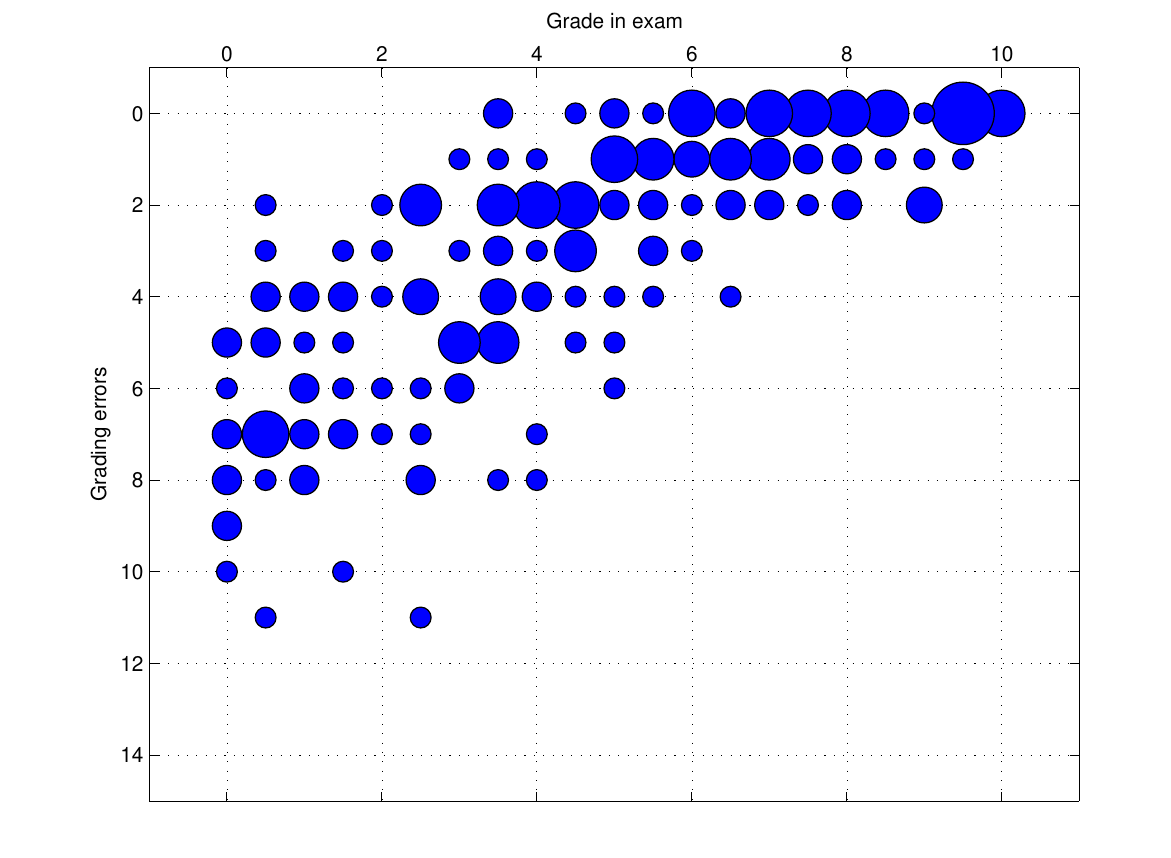}
  \caption{Mallows noise}
\end{subfigure}  
\begin{subfigure}{0.44\textwidth}
  \includegraphics[trim=20 200 30 180, clip=true, scale=0.33]{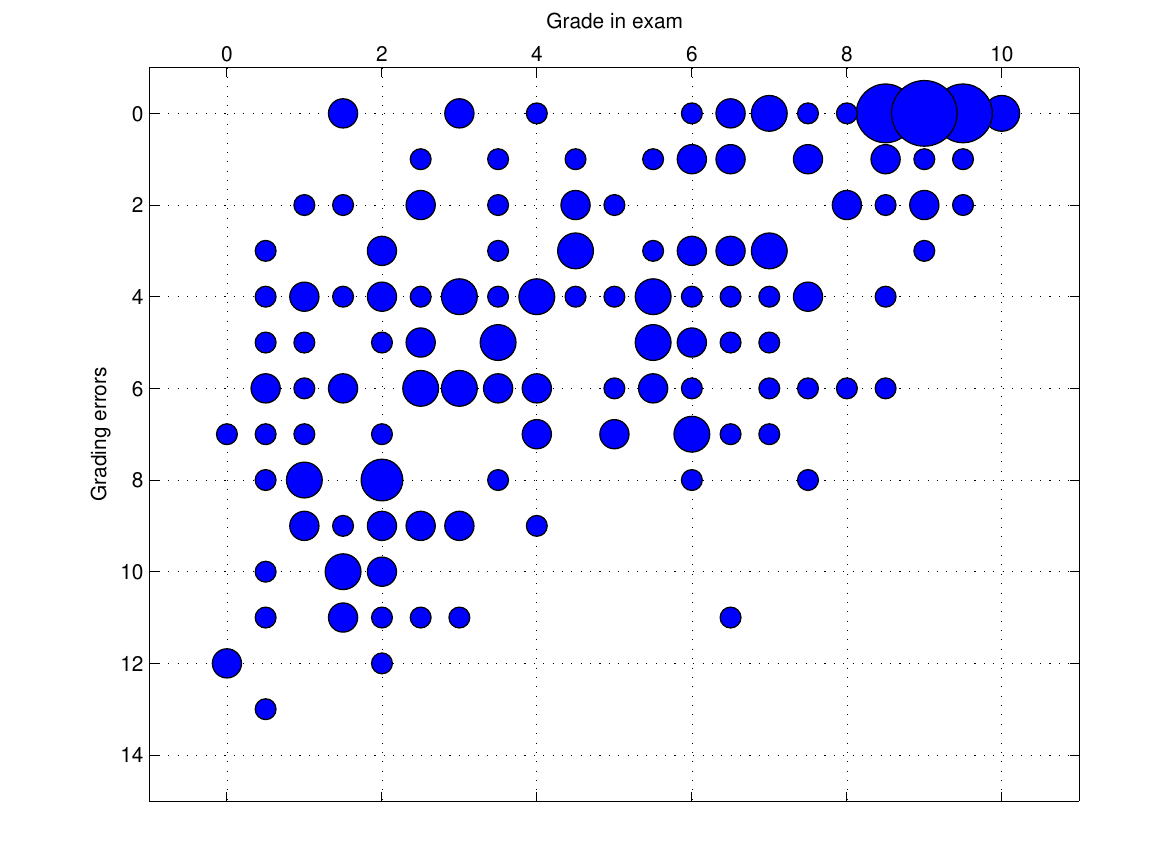}
  \caption{RUM noise}
\end{subfigure} 
\caption{Correlation between (cardinal) grade of students in the traditional exam and grading error (Kentall-tau distance from the correct ranking). Data refer (a) to the $136$ students that participated in our field experiment in 2015, (b) to $241$ students that participated in our field experiment in 2016, (c) to $200$ random students/graders drawn from the Mallows distribution, and (d) to $200$ random students/graders drawn from the RUM distribution. Each bubble corresponds to a number of students that is proportional to its area. The data in (a) and (b) have been obtained by processing the data in Tables~\ref{tab:2015},~\ref{tab:2016a}, and~\ref{tab:2016b}.}
\label{fig:bubbles}
\end{figure*}

For comparison, we have also plotted this correlation for randomly chosen students for the Mallows and RUM populations in Figures \ref{fig:bubbles}(c) and \ref{fig:bubbles}(d). Both figures show data about $200$ students as a representative number in between the number of participants in the two field experiments. For the Mallows population, the correlation between student quality (to be thought of as equivalent to the success in the traditional exam) and grading performance is clear. However, for the RUM population, the correlation seems to be more noisy.

The data depicted in Tables~\ref{tab:2015}, \ref{tab:2016a}, and~\ref{tab:2016b} have been used extensively in simulated exams with two ``realistic'' populations. Essentially, each student in the corresponding experiments serves as the support of the distribution of the grading behaviour of the realistic populations. For sampling students from these populations, we draw {\em pairs} (consisting of a ranking of exam papers and the corresponding quality) from the tables uniformly at random and independently, with the quality values slightly perturbed (randomly) so that a strict ground truth ranking of the sampled students is obtained.

The information about the grading behaviour of students in our two experiments has also been distilled into the noise matrices
\begin{align} \label{eq:real-2015}
P_{2015} = \left[
\begin{array}{cccccc}
0.4632 & 0.2573	& 0.1029 & 0.0588 &	0.0588  & 0.0588  \\
0.2059 & 0.3162	& 0.2279 & 0.1103 &	0.0662  & 0.0735 \\ 
0.1617 & 0.1912	& 0.2574 & 0.2059 &   0.1324 & 0.0515  \\
0.1029 & 0.1176	& 0.1912 & 0.2426 &	0.2794 & 0.0662  \\
0.0441  & 0.0661	& 0.1397 & 0.2206 &	0.3015 & 0.2279 \\
0.0221  & 0.0514	& 0.0808 & 0.1618 &	0.1618 & 0.5221 
\end{array}  \right]
\end{align}
and
\begin{align} \label{eq:real-2016}
P_{2016} = \left[
\begin{array}{cccccc}
0.6224 & 0.2199 & 0.0788 & 0.0373 & 0.0124 & 0.0290 \\
0.1826 & 0.4896 & 0.1867 & 0.1037 & 0.0249 & 0.0124 \\
0.0664 & 0.1494 & 0.4647 & 0.1992 & 0.0788 & 0.0415 \\
0.0664 & 0.0664 & 0.1411 & 0.4315 & 0.2116 & 0.0830\\
0.0456 & 0.0498 & 0.0913 & 0.1618 & 0.4730 & 0.1784 \\
0.0166 & 0.0249 & 0.0373 & 0.0664 & 0.1992 & 0.6556 \\
\end{array}  \right]
\end{align}
which are used when applying our theoretical framework.
We will use the terms {\em realistic 2015} and {\em realistic 2016} to refer to the noise model represented by matrices $P_{2015}$ and $P_{2016}$, respectively. The information in the matrices was obtained by measuring the frequency that the $i$-th ranked exam paper by students should be correctly ranked at position $j$. For example, in 2015, $28$ out of $136$ students ranked third an exam paper in their bundles which should have been ranked fourth; thus, cell $(3,4)$ in $P_{2015}$ contains the value $28/136\approx 0.2059$. 

We have also implemented the two processes that define Mallows and RUM graders (see Section~\ref{subsec:grading}) for bundles of size $6$, and use them in simulations. By sampling $10^9$ Mallows and RUM students with uniform qualities and simulating their grading behaviour, we have computed the corresponding noise matrices  
\begin{align}\label{eq:mallows}
P_{\text{mallows}} &= \left[
\begin{array}{cccccc}
0.6337	& 0.1753 &	0.0824 &0.0494 & 0.0339	& 0.0253 \\
0.1753	& 0.5112 &	0.1549 &0.0768 & 0.0479	& 0.0339 \\
0.0824  & 0.1549 &	0.4865 &0.1500 & 0.0768 & 0.0494 \\
0.0494	& 0.0768 &	0.1500 &0.4865 & 0.1549	& 0.0824 \\
0.0339	& 0.0479 &	0.0768 &0.1549 & 0.5112	& 0.1753 \\
0.0253	& 0.0339 &	0.0494 &0.0824 & 0.1753	& 0.6337
\end{array}  \right]
\end{align}
and
\begin{align}\label{eq:rum} 
P_{\text{rum}} &= \left[
\begin{array}{cccccc}
0.5046  &  0.1587  &  0.0963  &  0.0824  &  0.0793  &  0.0788\\
0.1587  &  0.4048  &  0.1709  &  0.1026  &  0.0836  &  0.0793\\
0.0963  &  0.1709  &  0.3746  &  0.1732  &  0.1026  &  0.0823\\
0.0824  &  0.1026  &  0.1732  &  0.3746  &  0.1709  &  0.0963\\
0.0793  &  0.0836  &  0.1026  &  0.1709  &  0.4048  &  0.1586\\
0.0788  &  0.0793  &  0.0824  &  0.0963  &  0.1586  &  0.5046
\end{array}  \right]
\end{align}

The noise matrices $P_{2015}$, $P_{2016}$, $P_{\text{mallows}}$, and $P_{\text{rum}}$ are used in the computation of the optimal type-ordering aggregation rules for the corresponding student populations according to the methodology developed in Sections \ref{subsec:framework} and \ref{subsec:optimal}. We stress again that these noise matrices do not include any information about the correlation between the grading behaviour and the quality of the student that acts as grader. This is a feature that our theoretical framework completely neglects. In contrast, this correlation is implemented in our simulations. Surprisingly, as we will see in the next section, our theory leads to very accurate performance predictions, in spite of its several simplifying assumptions compared to practice.

\subsection{On the accuracy of theoretical performance predictions}\label{subsec:optimal-exp}
We have applied the theoretical framework that we developed in Sections~\ref{subsec:framework} and~\ref{subsec:optimal} in order to obtain the optimal type-ordering aggregation rules for several scenarios together with theoretical predictions regarding their performance. In all scenarios, we use the same bundle size of $k=6$ and distinguish between the realistic 2015, the realistic 2016, the Mallows, and the RUM noise models by using the corresponding matrices $P_{2015}$, $P_{2016}$, $P_{\text{mallows}}$ and $P_{\text{rum}}$ defined in equations (\ref{eq:real-2015}), (\ref{eq:real-2016}), (\ref{eq:mallows}) and (\ref{eq:rum}), respectively. As bivariate performance objectives, we have considered the following:
\begin{itemize}
\item all2all: the total number of all correctly recovered pairwise relations. The corresponding bivariate function is defined as $f(x,y)=1$ for $x,y\in [0,1]$ with $x\leq y$ and $f(x,y)=0$ otherwise;
\item th-10\% and th-50\%: the total number of correctly recovered relations between pairs that include an exam paper that is ranked in the top $10\%$ and top $50\%$ in the ground truth, respectively, i.e., $f(x,y)=1$ if $x\leq 0.1$ and $x\leq 0.5$ and $x\leq y$, respectively;
\item acc-2\% and acc-5\%: the total number of correctly recovered relations between pairs with positions that differ by at least $2\%$ and $5\%$ in the ground truth, respectively, i.e., $f(x,y)=1$ if $y-x\geq 0.02$ and $y-x\geq 0.05$, respectively.
\end{itemize}

For each scenario, we use (\ref{eq:new-weight}) and (\ref{eq:prob-x-sigma-is-a-poly}) to compute the weight $W(\sigma,\sigma')$ for every pair of types $\sigma$ and $\sigma'$ from $\mathcal{T}_6$. Then, following Theorem~\ref{thm:hardness}, we solve the corresponding instance of FAS (as described in Section~\ref{subsec:optimal}) to compute the type-ordering aggregation rule that is optimal (of course, under the simplifying assumptions of our theoretical framework) for the particular scenario. The theoretical prediction of performance is then given by (\ref{eq:sum-of-weights-refined}) from Theorem~\ref{thm:sum-of-weights-refined}.

Computations required for the application of our theoretical framework (i.e., those described in Appendix \ref{app:sec:compute-integral} for the computation of the weights as well as the FAS algorithm described in Section \ref{subsec:optimal-exp}) have been automated. All the computational results that we report in the following have been obtained using an Intel 12-core i7 machine with 32Gb of RAM running Windows 7. Our methods have been implemented in C using the GNU Multiple Precision Arithmetic Library (GMP) and in Matlab R2013a. In particular, high precision is absolutely necessary in order to compute the weights even for bundles of size $6$ since, by inspecting equations (\ref{eq:new-weight}) and (\ref{eq:prob-x-sigma-is-a-poly}) carefully (see also Appendix~\ref{app:sec:compute-integral} for a detailed discussion on the computation of the weights), we can see that there are products with more than $30$ factors and factorials of integers up to $30$ that are involved in the computations.

In all scenarios we have considered, the algorithm for solving FAS is fast. This is due to the fact that the strongly connected components have small size. In all cases, among the 462 different types that we can have for bundles of size $6$, more than 97\% of them form singleton components and the maximum component size never exceeded 50 (for the RUM model). Brute forcing has been used to order the types in strongly connected components of size up to $10$. For larger components, Borda orderings have been used as described in Section~\ref{subsec:optimal}. The distribution of the strongly connected components for the scenarios we considered is depicted in Table \ref{tab:component-size}.

\begin{table}[h]
\centering
\begin{tabular}{ccccccc}
\noalign{\hrule height 1.5pt}
\multicolumn{2}{c}{size}	&	1		& 	3--7	&	8--11	&	$\geq$ 12	&	max \\\noalign{\hrule height 1.5pt}
\multirow{4}{*}{\parbox[t]{1.5cm}{\centering realistic\\2015}} & 	all2all &	448		&	13		&	1		&		0		&	10	\\
& 	th-50\% &	460		&	2		& 	0		&		0		&	3	\\
& 	acc-2\% &	449		&	12		& 	1		&		0		&	10	\\
& 	acc-5\% & 	451		&	10		&	1		&		0		&	10	\\\hline
\multirow{4}{*}{\parbox[t]{1.5cm}{\centering realistic\\2016}} 
& 	all2all &	458		&	4		&	0		&		0		&	5	\\
& 	th-50\% &	460		&	2		& 	0		&		0		&	3	\\
& 	acc-2\% &	458		&	4		& 	0		&		0		&	5	\\
& 	acc-5\% & 	460		&	2		&	0		&		0		&	4	\\\hline
\multirow{4}{*}{mallows} 
& 	all2all &	453		&	6		&	2		&		1		&	20	\\
& 	th-50\% &	459		&	3		& 	0		&		0		&	4	\\
& 	acc-2\% &	449		&	10		& 	2		&		1		&	20	\\
& 	acc-5\% & 	449		&	12		&	0		&		1		&	20	\\\hline
\multirow{4}{*}{rum} 
& 	all2all &	443		&	10		&	2		&		7		&	50	\\
& 	th-50\% &	448		&	11		& 	2		&		1		&	17	\\
& 	acc-2\% &	439		&	14		& 	4		&		5		&	50	\\
& 	acc-5\% & 	435		&	16		&	10		&		1		&	50	\\
\noalign{\hrule height 1.5pt}
\end{tabular}
\caption{Distribution of the size of strongly connected components. Results about th-10\% are not shown; curiously, all strongly connected components are singletons in these cases.}
\label{tab:component-size}
\end{table} 

In parallel to the application of our theoretical framework (see again Figure \ref{fig:approach} which summarizes our overall approach), we have also performed extensive simulations for all scenarios considered. For each scenario, we have simulated exams with $10\,000$ students (as explained in Section~\ref{subsec:real-exp}), using the optimal type-ordering aggregation rule, that was obtained by applying our theoretical framework for the scenario, as discussed above. Tables \ref{tab:realistic-data} and \ref{tab:synthetic-data} contain the average values (from 1000 simulated exams) of the performance measure for each simulated grading scenario in columns labelled as ``simulation''. The columns labelled ``theory'' contain the theoretical performance predictions for the same aggregation rule and scenario. Data for perfect grading scenarios are reported in Table~\ref{tab:synthetic-data}, where Borda is the optimal type-ordering aggregation rule.

In contrast to the simplifying assumptions of our theoretical framework, correlation of grading behaviour and performance in the exam is a key feature in our simulations. Therefore, the information contained in Tables \ref{tab:realistic-data} and \ref{tab:synthetic-data} is rather surprising and shows that, in spite of our assumptions, our theory provides extremely accurate predictions for the performance of type-ordering aggregation rules in practice. Note that the values in Tables \ref{tab:realistic-data} and \ref{tab:synthetic-data} are percentages and we never observed differences beyond the second decimal point between the theoretically predicted value and the simulated one.\footnote{Even though we have consistently used exams with $10\,000$ students in all our simulations, data with smaller exams are also very close to the theoretically predicted values. For example, in simulations with 1000 exams with 1000 students in the realistic 2016 grading scenario with the all2all performance objective, the optimal rule and  Borda have average performance percentage of $85.76$ and $85.07$, compared to $85.69$ and $85.02$ in Table~\ref{tab:realistic-data}. Higher differences are observed for much smaller exams (respectively, $86.23$ and $85.55$ for $100$-student exams).} Also, note that the number of $10\,000$ students in our simulations is much lower than the vision for the most popular courses that will be offered by MOOCs in the near future; the predictions become even more accurate for higher numbers of students.

Borda has been used in all imperfect grading scenarios for comparison purposes. The optimal type-ordering aggregation rule can have a performance that is $3.5\%$ better than Borda (e.g., in the th-10\% scenario with RUM graders). However, in many cases, Borda is closer to optimality.

\begin{table}[ht]
\centering
\begin{tabular}{ccccccccc}
\noalign{\hrule height 1pt}\hline
noise   & \multicolumn{4}{c}{realistic grading 2015}                              & \multicolumn{4}{c}{realistic grading 2016}                              \\\hline
setting & \multicolumn{2}{c}{theory} & \multicolumn{2}{c}{simulation} & \multicolumn{2}{c}{theory} & \multicolumn{2}{c}{simulation} \\
method  & opt         & borda        & opt                 & borda                & opt         & borda        & opt                 & borda                \\\noalign{\hrule height 1.5pt}
all2all		&	80.01	&	79.57	&	80.09	&	79.57	&	85.70	&	85.02	&	85.69	&	85.02	\\
th-10\%	&	87.61	&	87.18	&	87.60	&	87.17	&	91.71	&	90.02	&	91.69	&	90.01	\\
th-50\%	&	83.62	&	83.43	&	83.62	&	83.43	&	88.64	&	88.06	&	88.63	&	88.06	\\
acc-2\%	&	81.27	&	80.73	&	81.27	&	80.74	&	87.08	&	86.39	&	87.08	&	86.38	\\
acc-5\%	&	82.97	&	82.42	&	82.97	&	82.42	&	89.01	&	88.31	&	89.01	&	88.30	\\
\noalign{\hrule height 1pt}\hline
\end{tabular}
\caption{Performance of optimal type-ordering aggregation rules as well as Borda for the two realistic grading scenarios of 2015 and 2016 with respect to the five different objectives. The values presented are theoretical predictions (theory) and average simulation measurements from 1000 exams with $10\,000$ students.}
\label{tab:realistic-data}
\end{table}

\begin{table}[ht]
\centering
\begin{tabular}{ccccccccccc}	
\noalign{\hrule height 1.5pt}
noise & \multicolumn{2}{c}{perfect grading} & \multicolumn{4}{c}{mallows grading} & \multicolumn{4}{c}{rum grading}\\\hline
setting & theory & sim. & \multicolumn{2}{c}{theory} & \multicolumn{2}{c}{simulation} & \multicolumn{2}{c}{theory} & \multicolumn{2}{c}{simulation}\\
method & borda & borda & opt & borda & opt & borda & opt & borda & opt & borda \\\noalign{\hrule height 1.5pt}
all2all	&	92.01	&	92.02	&	85.15	&	84.38	&	85.16	&	84.39	&	77.89	&	76.79	&	77.89	&	76.81	\\
th-10\%	&	96.94	&	96.95	&	92.05	&	90.52	&	92.07	&	90.54	&	87.11	&	83.59	&	87.13	&	83.62	\\
th-50\%	&	94.13	&	94.14	&	88.39	&	87.8	&	88.4	&	87.81	&	81.27	&	80.32	&	81.28	&	80.33	\\
acc-2\%	&	93.57	&	93.57	&	86.52	&	85.72	&	86.52	&	85.73	&	78.99	&	77.85	&	78.99	&	77.86	\\
acc-5\%	&	95.47	&	95.47	&	88.42	&	87.61	&	88.42	&	87.62	&	80.57	&	79.40	&	80.57	&	79.41	\\
\noalign{\hrule height 1.5pt}
\end{tabular}
\caption{Performance of optimal type-ordering aggregation rules as well as Borda for scenarios with perfect, Mallows, and RUM grading with respect to the five different objectives. The values presented are theoretical predictions (theory) and average simulation measurements from 1000 exams with $10\,000$ students.}
\label{tab:synthetic-data}
\end{table}

Figure \ref{fig:clouds-borda-vs-opt} reports detailed information for all simulations, for the all2all, th-10\% and acc-5\% scenarios. Clearly, the performance of the aggregation rules for all objectives that we considered is sharply concentrated around the average values shown in Tables~\ref{tab:realistic-data} and~\ref{tab:synthetic-data}; note that the size of the $x$ and $y$-axis that are depicted in all subfigures are at most $3\%$ wide (besides in subfigure \ref{fig:clouds-borda-vs-opt}(k) for RUM grading with the th-10\% bivariate performance objective, which has axes that are $6\%$ wide). Again, Borda is used for comparison purposes.

\begin{figure*}[p]
\centering
\begin{subfigure}{0.3\textwidth}
   \includegraphics[trim= 50 200 50 250, clip=true, scale=0.25]{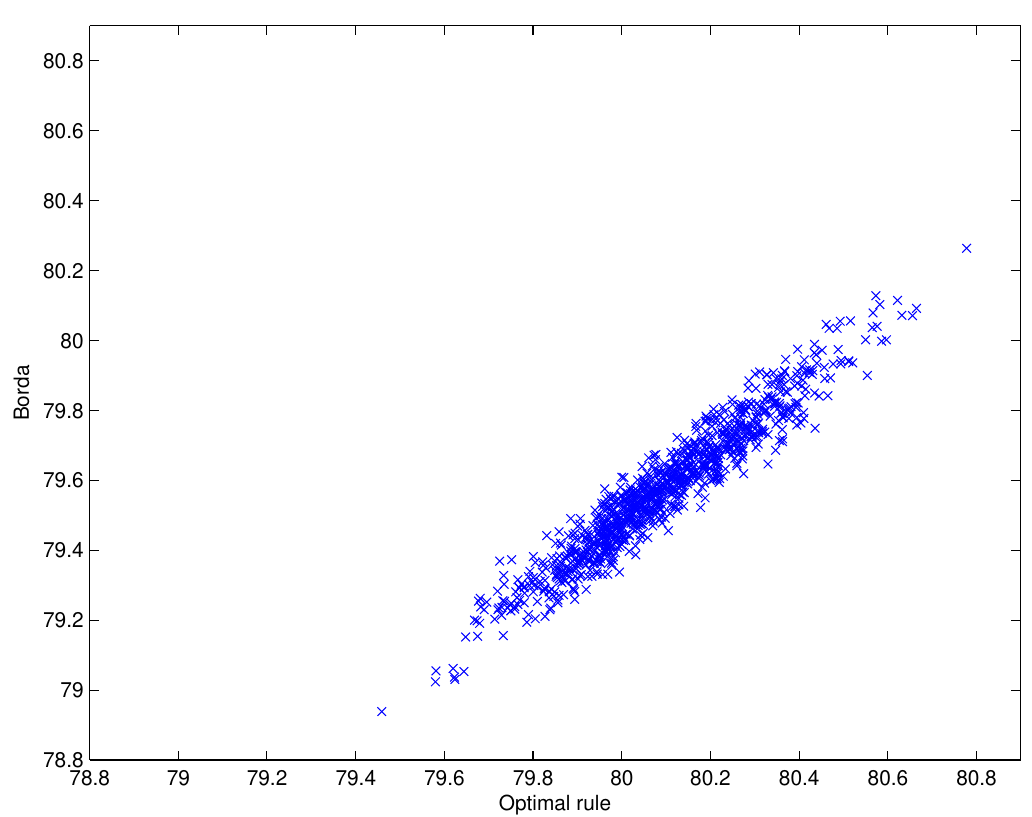} 
   \caption{realistic 2015, all2all}
\end{subfigure}
\begin{subfigure}{0.3\textwidth}
   \includegraphics[trim= 50 200 50 250, clip=true, scale=0.25]{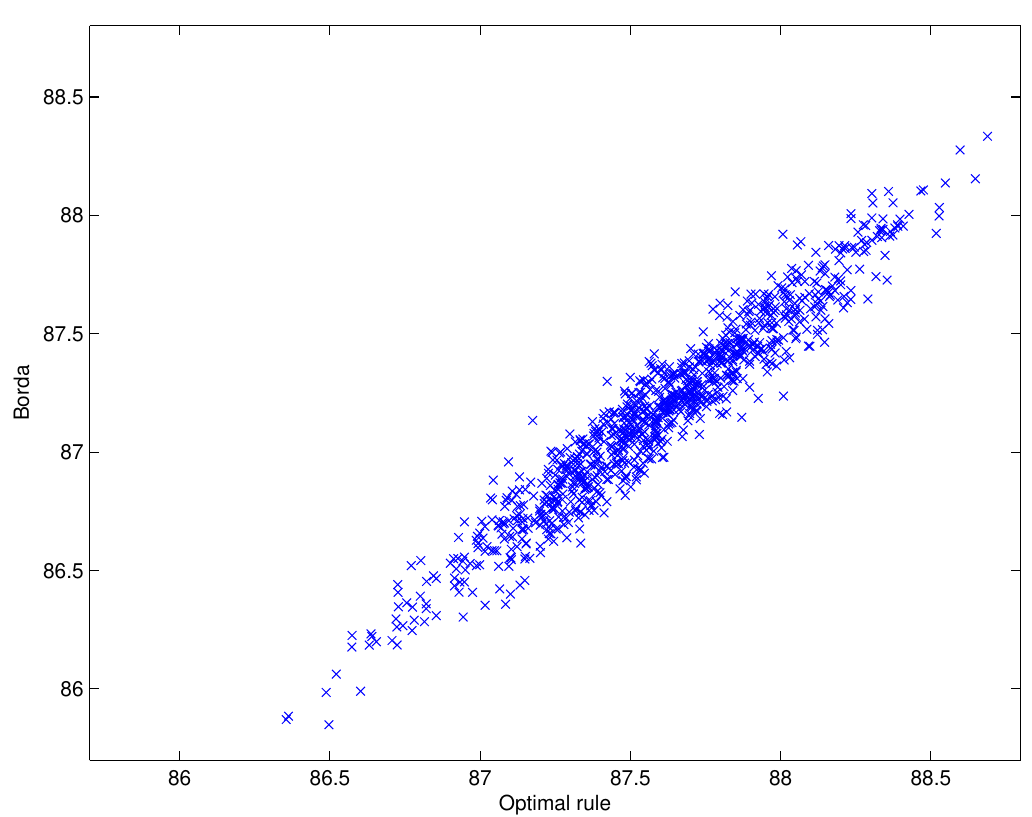} 
   \caption{realistic 2015, th-10\%}
\end{subfigure}
\begin{subfigure}{0.3\textwidth}
   \includegraphics[trim= 50 200 50 250, clip=true, scale=0.25]{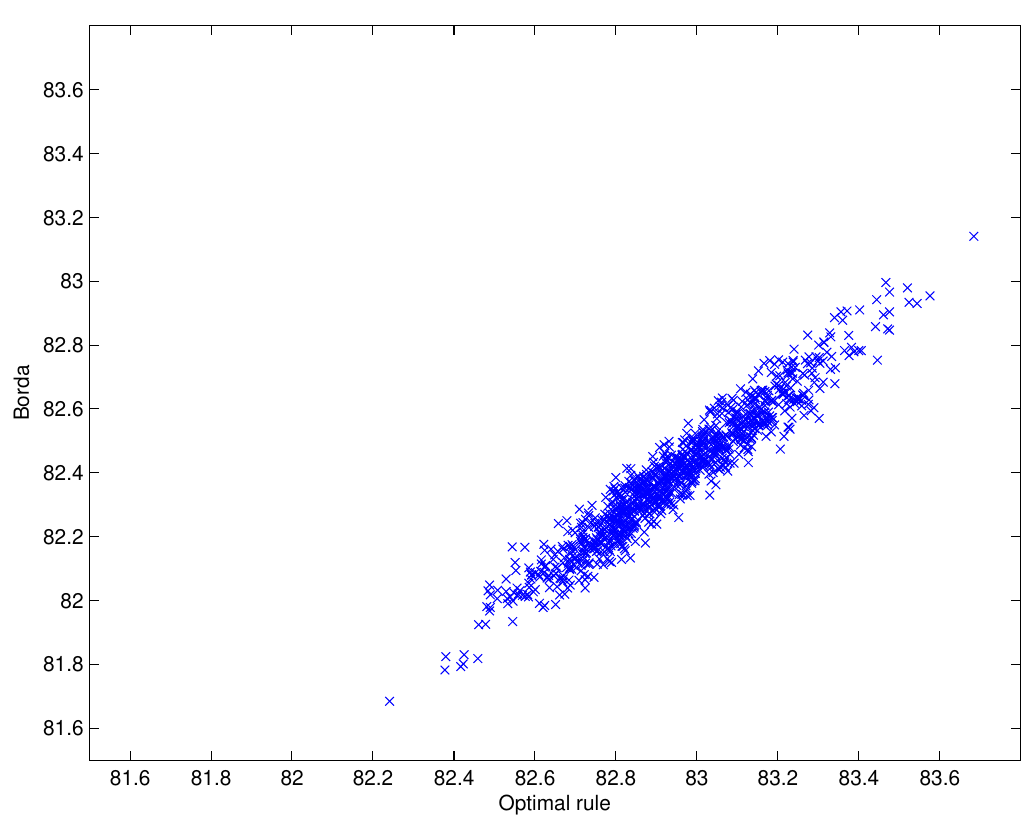}
   \caption{realistic 2015, acc-5\%}
\end{subfigure}

\begin{subfigure}{0.3\textwidth}
  \includegraphics[trim=50 200 50 200, clip=true, scale=0.25]{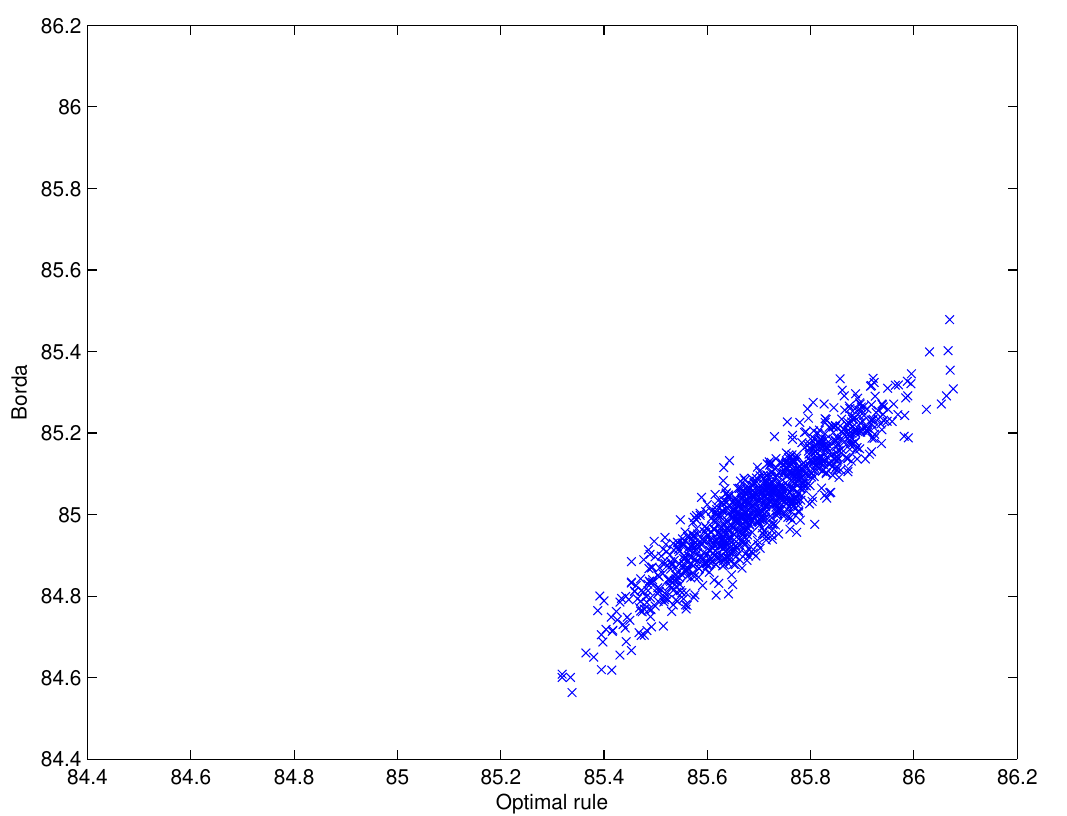}
   \caption{realistic 2016, all2all}
\end{subfigure}
\begin{subfigure}{0.3\textwidth}
     \includegraphics[trim=50 200 50 200, clip=true, scale=0.25]{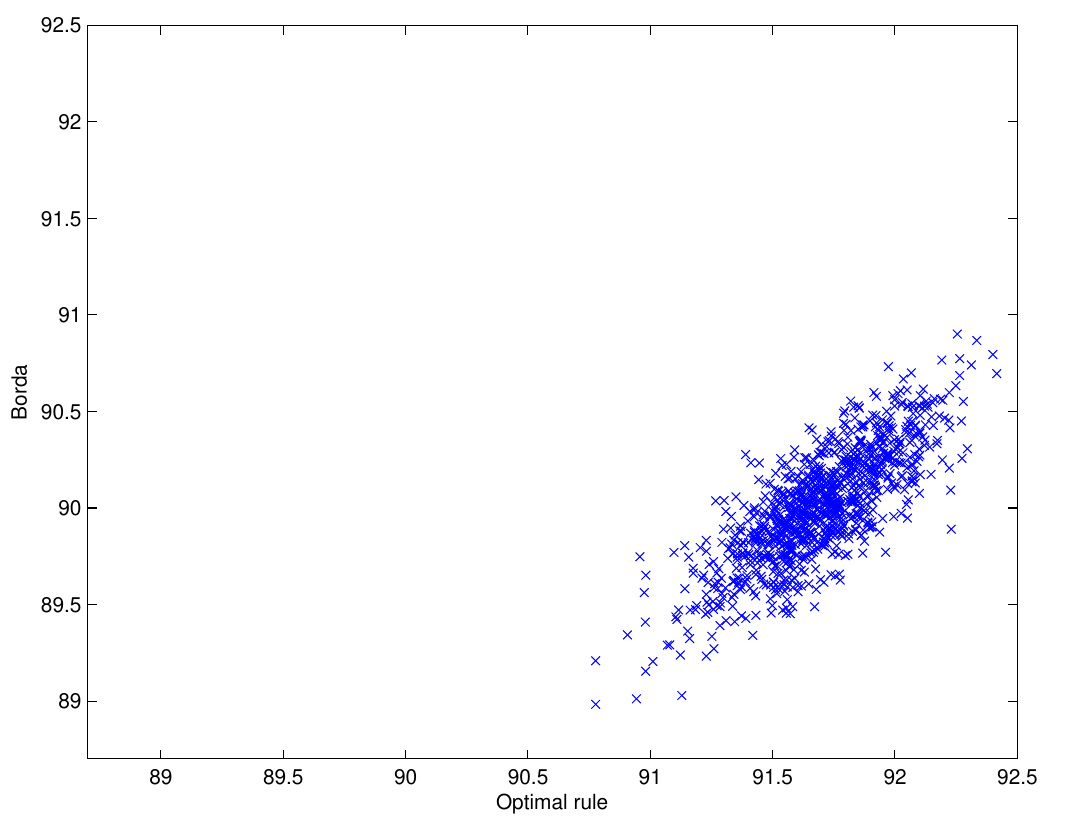}
   \caption{realistic 2016, th-10\%}
\end{subfigure}
\begin{subfigure}{0.3\textwidth}
     \includegraphics[trim=50 200 50 200, clip=true, scale=0.25]{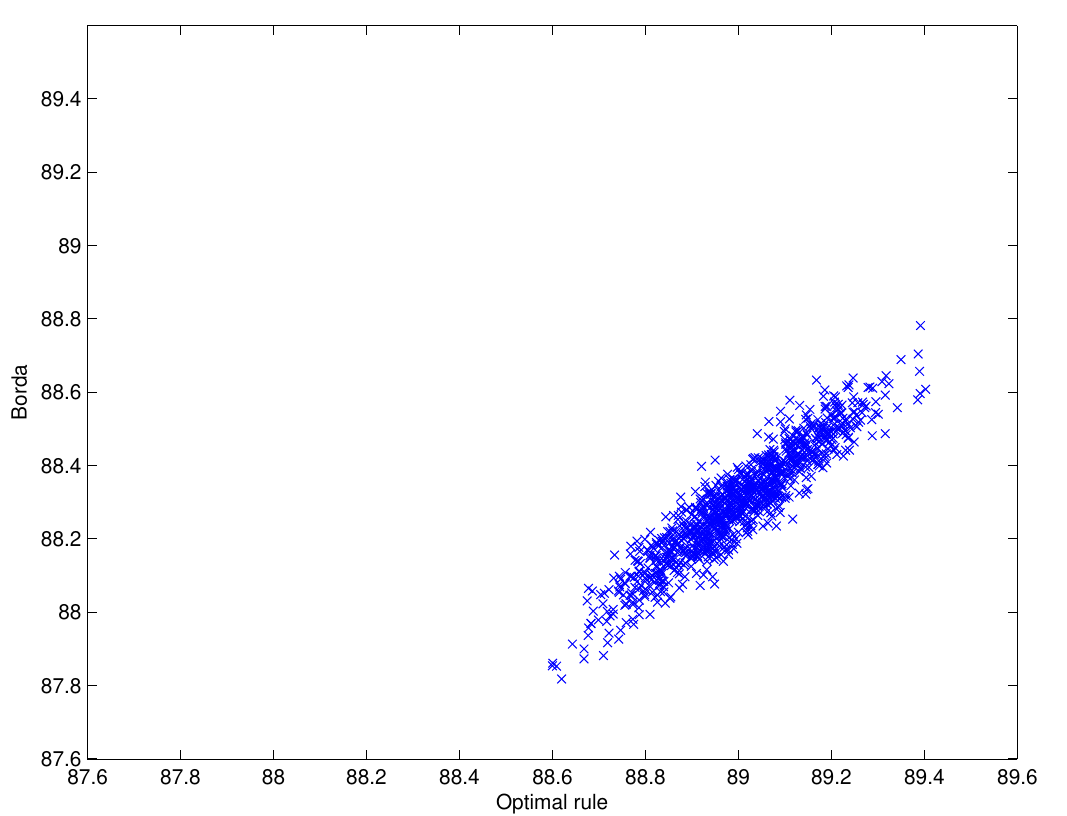}
   \caption{realistic 2016, acc-5\%}
\end{subfigure}

\begin{subfigure}{0.3\textwidth}
   \includegraphics[trim= 50 200 50 250, clip=true, scale=0.25]{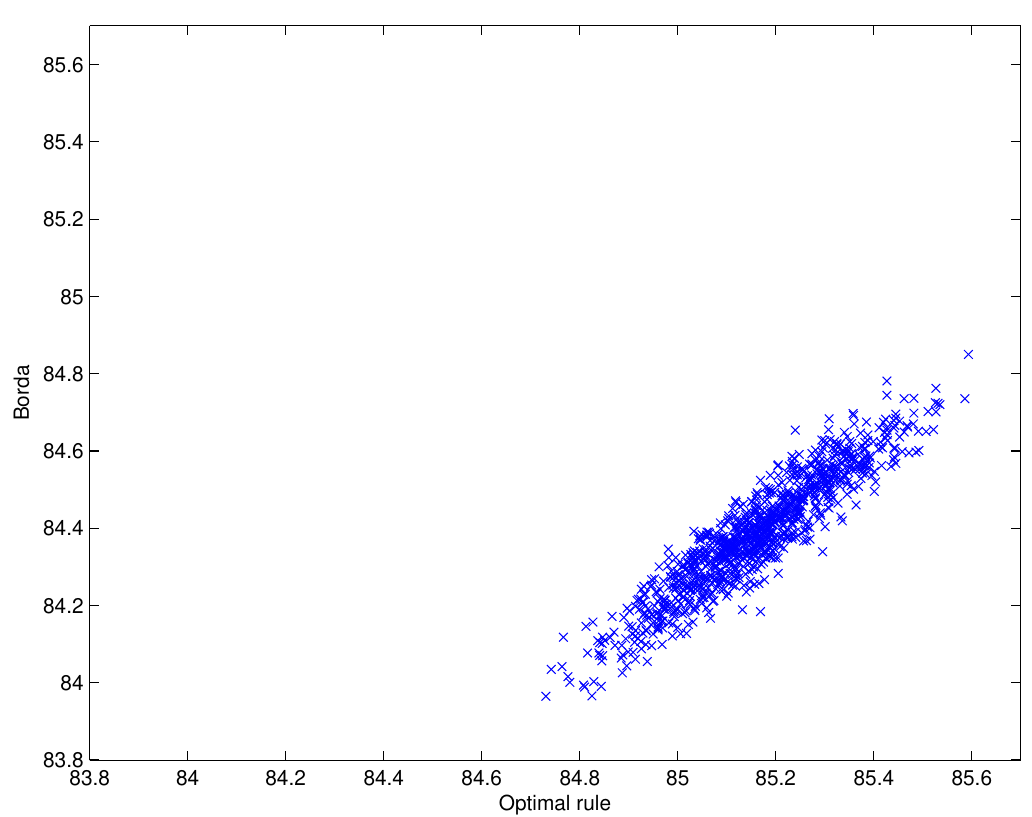} 
   \caption{mallows, all2all}
\end{subfigure}
\begin{subfigure}{0.3\textwidth}
   \includegraphics[trim= 50 200 50 250, clip=true, scale=0.25]{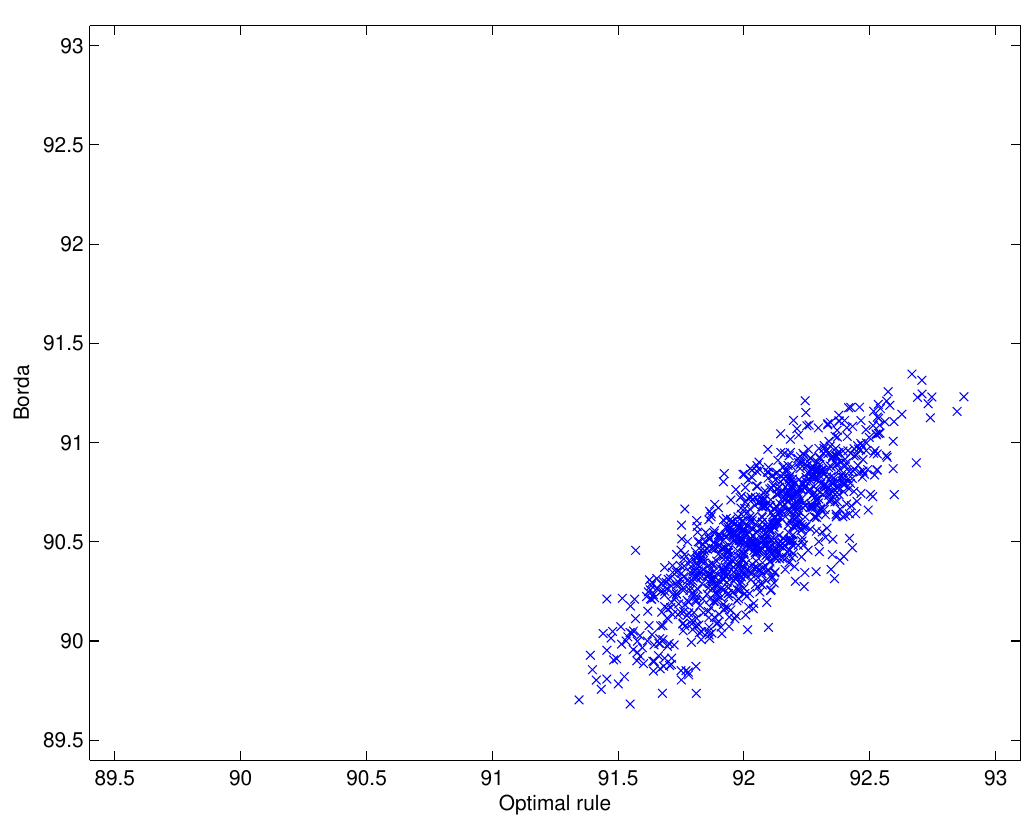} 
   \caption{mallows, th-10\%}
\end{subfigure}
\begin{subfigure}{0.3\textwidth}
   \includegraphics[trim= 50 200 50 250, clip=true, scale=0.25]{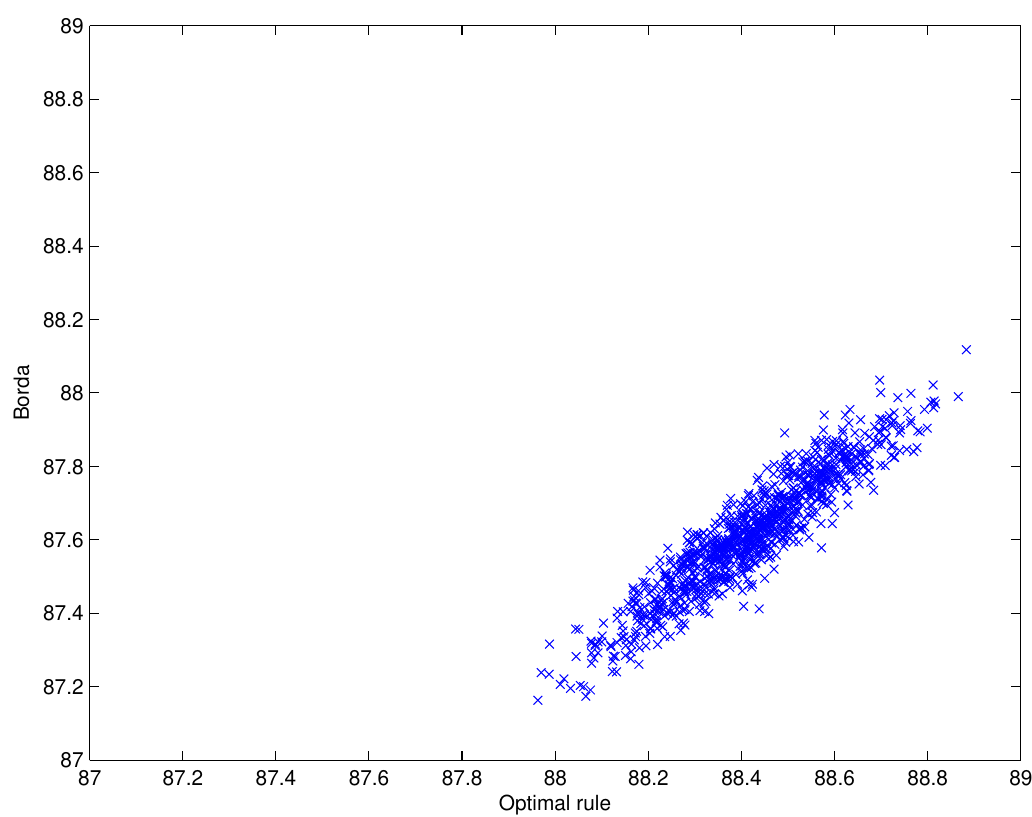} 
   \caption{mallows, acc-5\%}
\end{subfigure}

\begin{subfigure}{0.3\textwidth}
   \includegraphics[trim= 50 200 50 200, clip=true, scale=0.25]{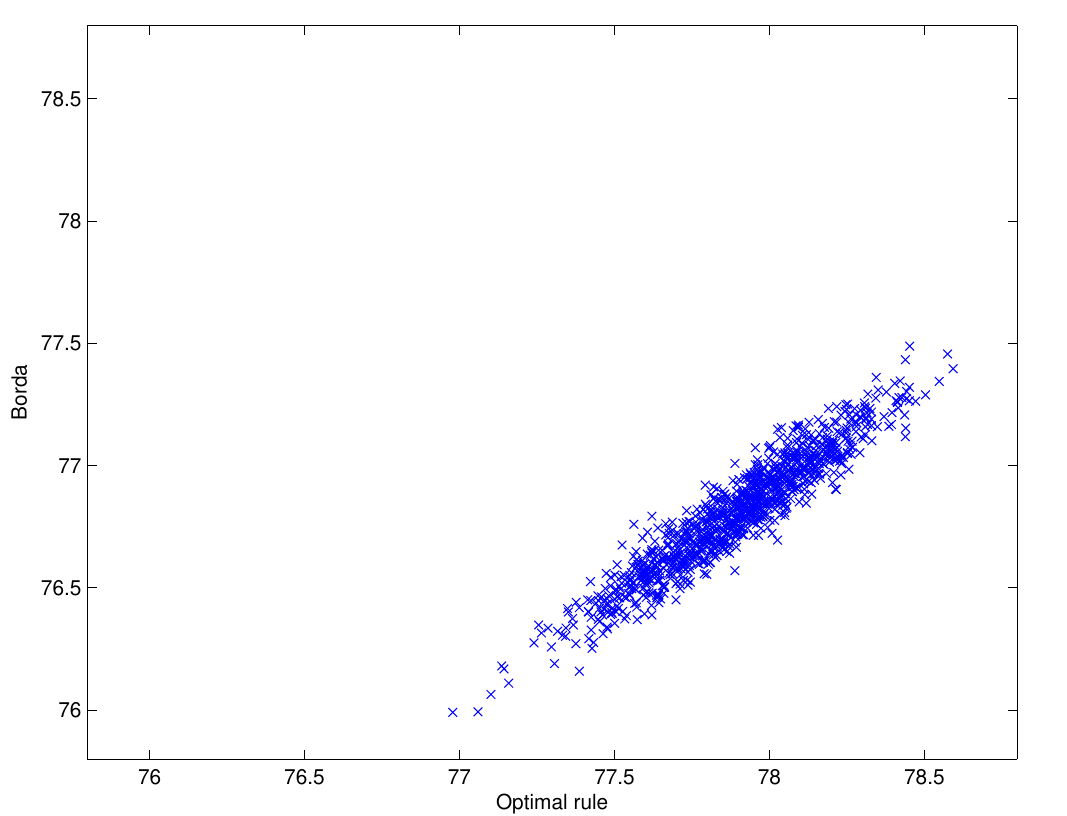} 
   \caption{rum, all2all}
\end{subfigure}
\begin{subfigure}{0.3\textwidth}
   \includegraphics[trim= 50 200 50 200, clip=true, scale=0.25]{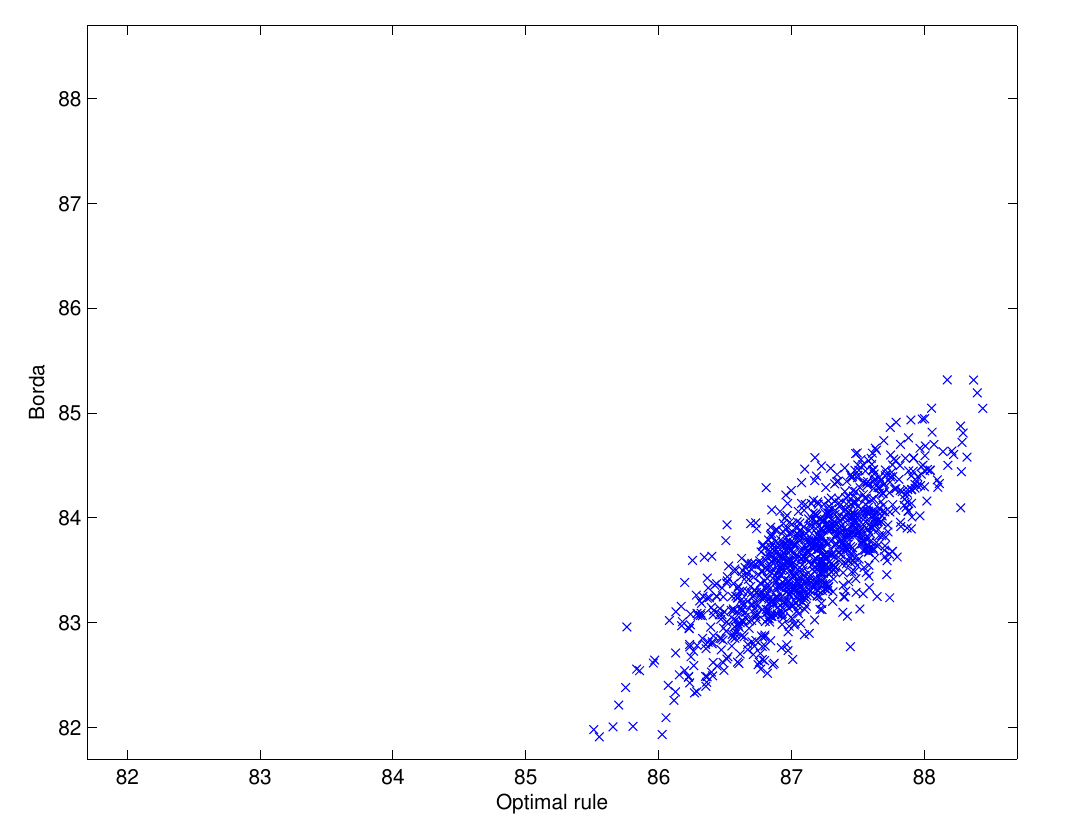} 
   \caption{rum, th-10\%}
\end{subfigure}
\begin{subfigure}{0.3\textwidth}
   \includegraphics[trim= 50 200 50 200, clip=true, scale=0.25]{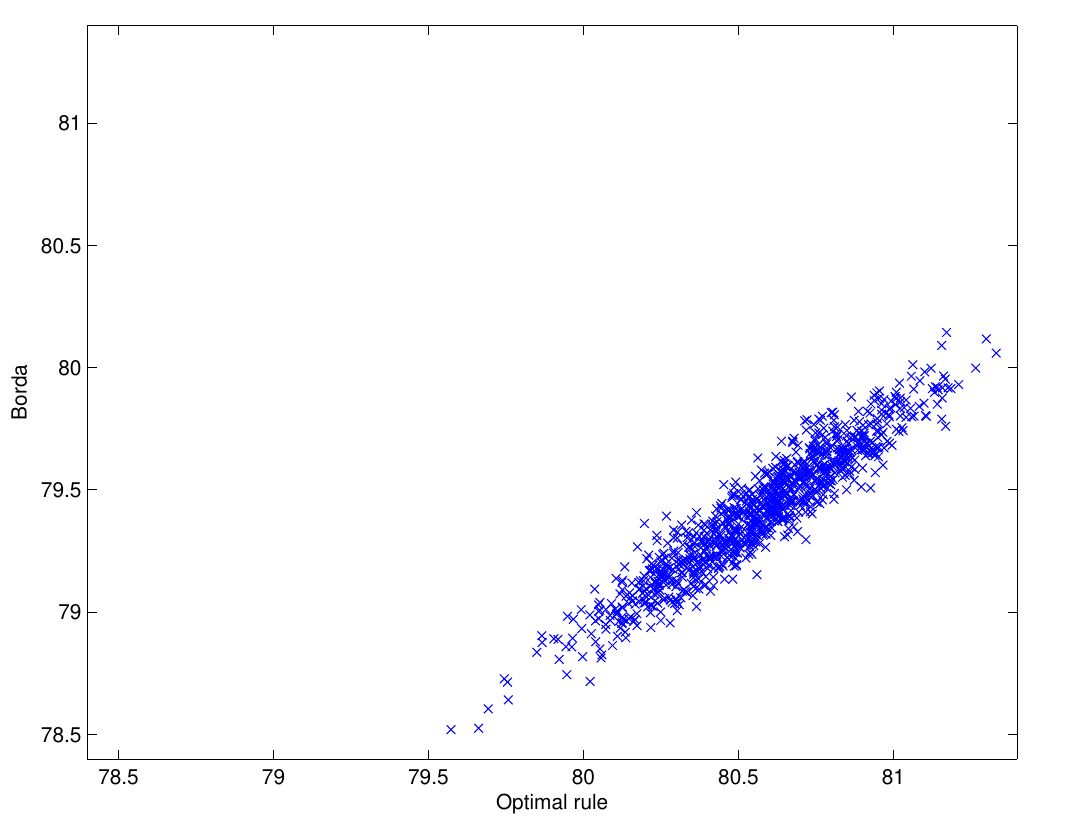}
   \caption{rum, acc-5\%}
\end{subfigure}
\caption{Performance of the optimal type-ordering aggregation rule and Borda for the realistic 2015, 2016, Mallows and RUM grading scenarios for the all2all, acc-5\%, and th-10\% objectives. Each point (among the 1000 points in each cloud) corresponds to a simulated exam with the participation of $10\,000$ students with the corresponding grading behaviour.}
\label{fig:clouds-borda-vs-opt}
\end{figure*}

A final comment on the performance of the optimal type-ordering aggregation rules is that they are {\em extremely robust}. Even though they have been optimized with respect to a particular bivariate performance objective, they perform very well with respect to other objectives as well. Figure \ref{fig:displacements} shows measurements of properties that cannot be expressed as bivariate performance objectives. Each plot shows data about Borda and the optimal (under the all2all objective) type-ordering aggregation rules in scenarios with perfect, realistic, Mallows, and RUM grading. Borda in the perfect grading scenario has the best performance with respect to these objectives as well. Actually, its performance in this scenario can serve as the optimistic barrier for every type-ordering aggregation rule in any (imperfect) grading scenario. More interestingly, Borda has performance that is very close to the optimal rule for realistic grading (the corresponding curves almost coincide in Figures \ref{fig:displacements}(a) and \ref{fig:displacements}(c)) and is slightly worse for Mallows and RUM grading. In fact, these results are in sync to those in Tables \ref{tab:realistic-data} and \ref{tab:synthetic-data}, and Figure \ref{fig:clouds-borda-vs-opt}.

\begin{figure*}[ht]
\centering
\begin{subfigure}{0.3\textwidth}
   \includegraphics[trim= 50 200 50 208, clip=true, scale=0.24]{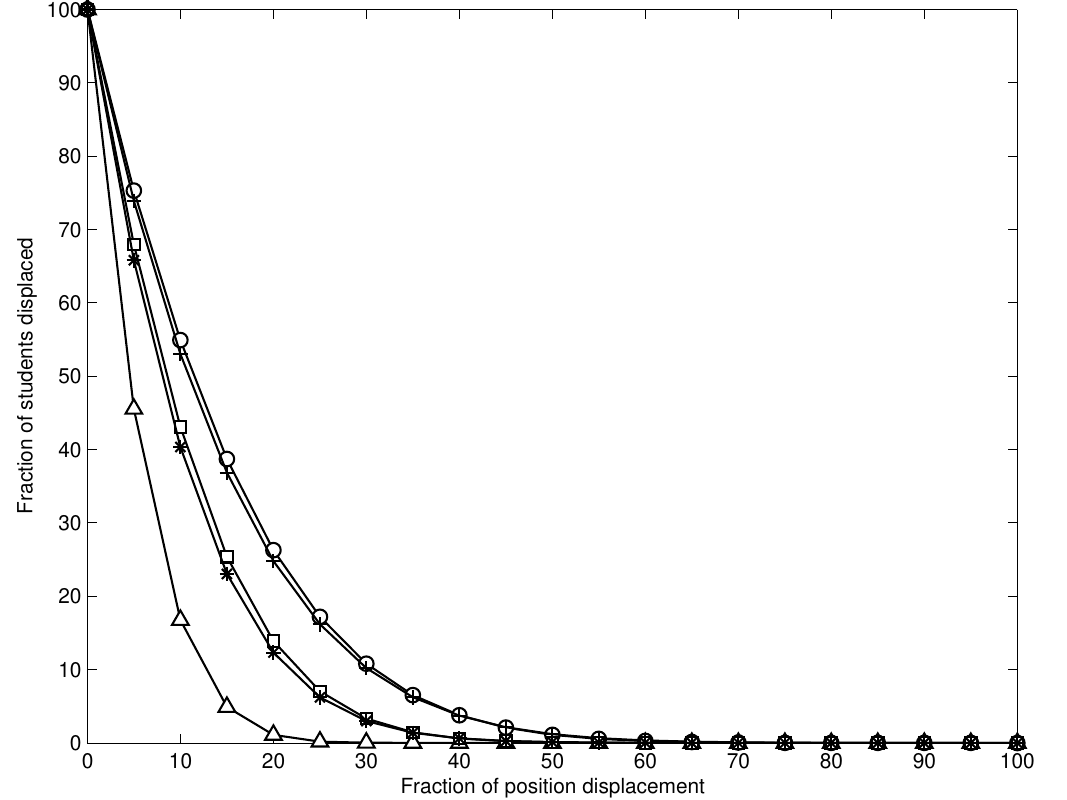} 
   \caption{Displacement, realistic \\ \ }
\end{subfigure}
\begin{subfigure}{0.3\textwidth}
   \includegraphics[trim= 50 200 50 208, clip=true, scale=0.24]{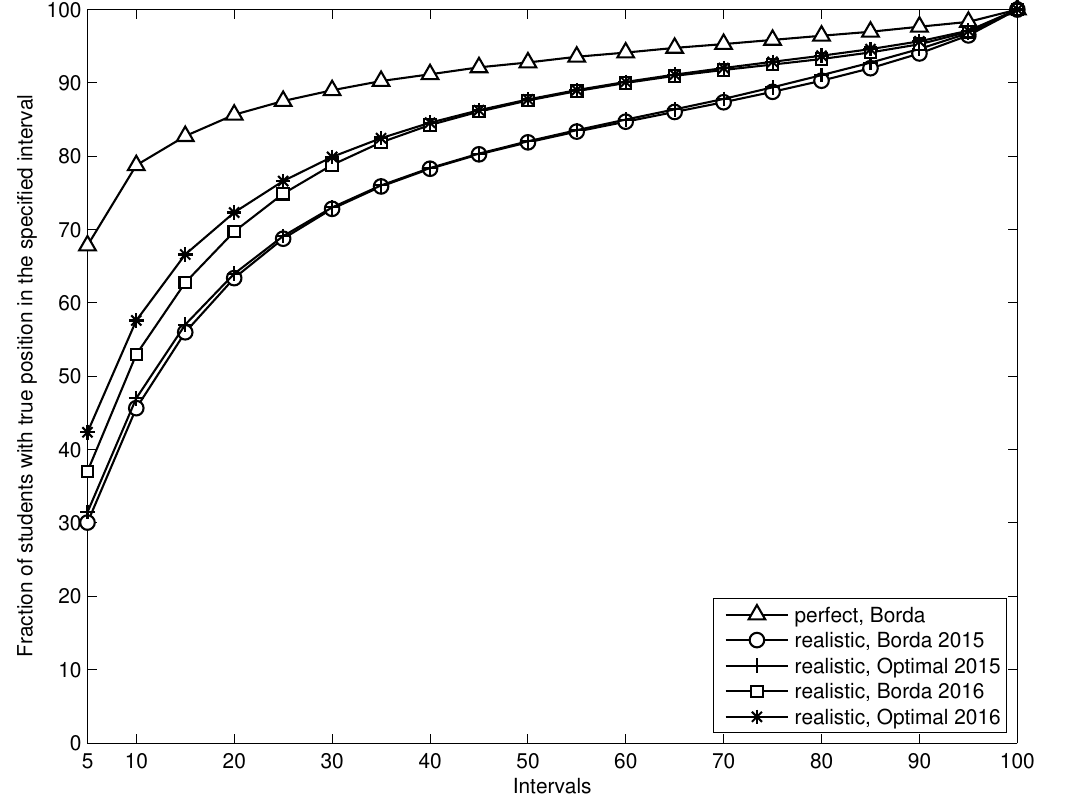}  
   \caption{Interval displacement,\\ realistic }
\end{subfigure}
\begin{subfigure}{0.3\textwidth}
   \includegraphics[trim=50 200 50 260, clip=true, scale=0.24]{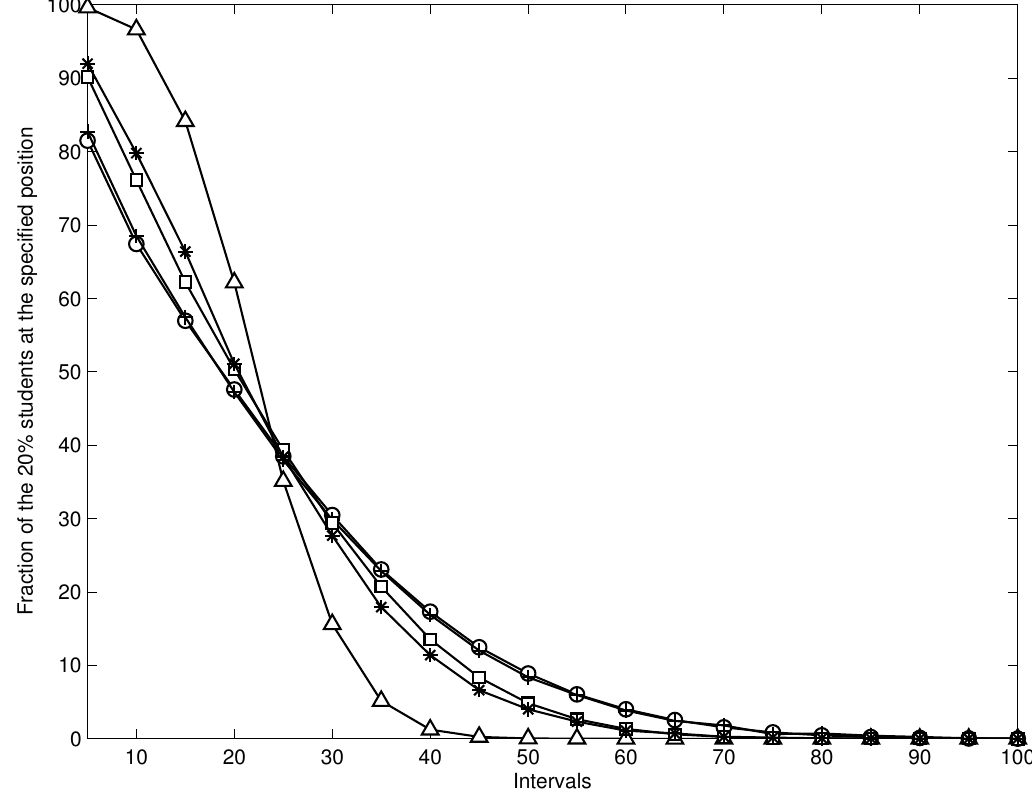}  
   \caption{Distribution of top 20\%,\\ realistic}
\end{subfigure}

\begin{subfigure}{0.3\textwidth}
   \includegraphics[trim= 50 200 50 208, clip=true, scale=0.24]{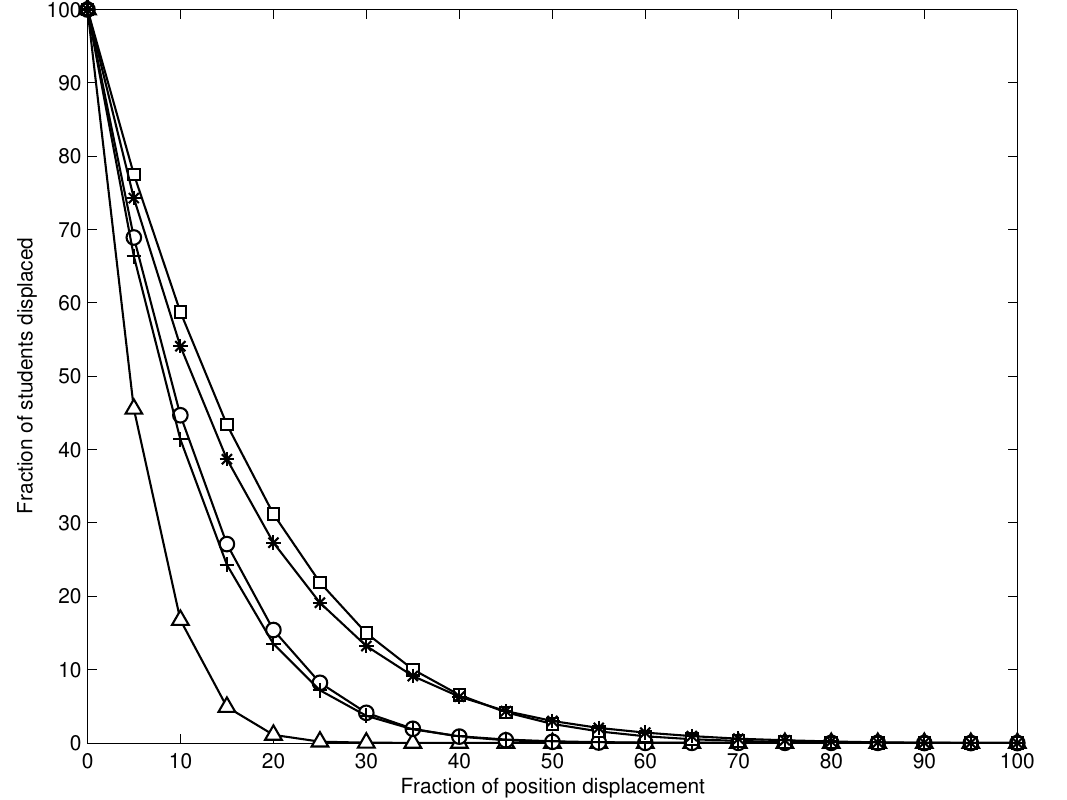} 
   \caption{Displacement, synthetic \\ \ }
\end{subfigure}
\begin{subfigure}{0.3\textwidth}
   \includegraphics[trim= 50 200 50 208, clip=true, scale=0.24]{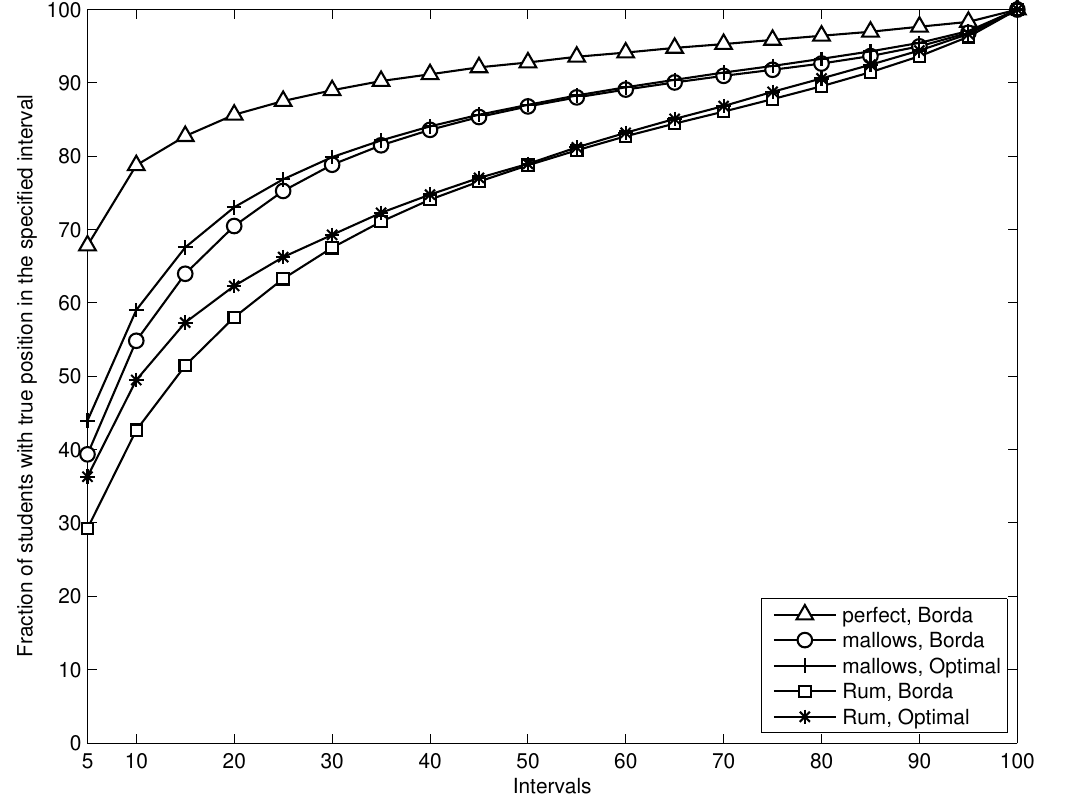}  
   \caption{Interval displacement,\\ synthetic}
\end{subfigure}
\begin{subfigure}{0.3\textwidth}
   \includegraphics[trim= 50 200 50 260, clip=true, scale=0.24]{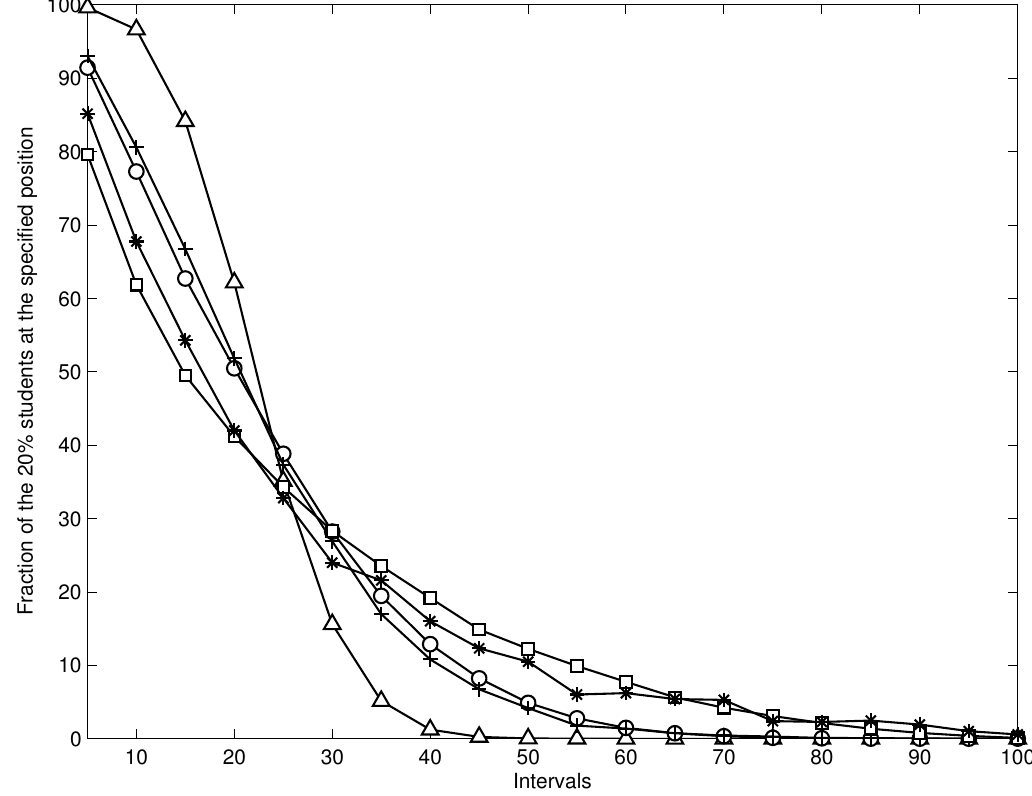}  
   \caption{Distribution of top 20\%,\\ synthetic}
\end{subfigure}
\caption{Robustness of optimal type-ordering aggregation rules and Borda with realistic graders in Subfigures (a), (b) and (c), and perfect, Mallows, and RUM graders in Subfigures (d), (e) and (f). The information depicted is the result of 1000 executions on exams with $10\,000$ participating students. Displacement: a point $(x, y)$ represents the fact that $y\%$ of the students have been displaced by at least $x\%$ from their true position. Interval displacement: a point $(x, y)$ represents the fact that $y\%$ of the students in the interval $[0, x\%]$ are the same both in the ground truth and in the rankings produced by the aggregation rules. Distribution of top $20\%$: a point $(x, y)$ represents the fact that $y\%$ of the top $20\%$ of the students are positioned in the interval $[(x - 5)\%, x\%]$.}
\label{fig:displacements}
\end{figure*}

\subsection{The effect of inaccuracies in the noise model}\label{subsec:apx}
The two realistic noise models that we built in Section \ref{subsec:real-exp} are, by definition, approximations of the students in our home institution. Besides limitations that have to do with our modelling assumptions, they have the obvious drawback that they have been built using a very small fraction of our students, i.e., 136 students in 2015 and the slightly increased number of 241 students in 2016. So far, the reader should have been convinced that the type-ordering aggregation rules we have built are indeed optimal for a large population that inherits the quality and grading performance of this small fraction of students; this has been the focus of the application of our theoretical framework and of our simulations with realistic grading. What is far from clear is whether these aggregation rules will perform equally well for the whole population of the students in our home institution. To see the importance of this question, imagine it in the planetary scale that MOOCs envision. Can we make safe predictions for huge student populations by sampling a tiny fraction of them, building a noise model as we did in Section \ref{subsec:real-exp}, and then selecting the optimal type-ordering aggregation rules as we did in Section \ref{subsec:optimal-exp}?

We give a positive answer to this question by considering Mallows and RUM grading scenarios. With Mallows and RUM, we have the luxury of two well-defined noise models for the grading behaviour of a huge student population which we have used in order to compute optimal type-ordering aggregation rules. This information will be used only for assessing the approach presented in the following. Now, we will pretend that no information about grading behaviour is available and all we can do is to apply (actually, to simulate) field experiments like the ones we presented in Section \ref{subsec:real-exp} on tiny fractions of the students in order to come up with noise matrices. In this way, we will compute  approximations of the true noise models.

We have followed this approach using samples of Mallows and RUM graders of size 100 and 1000; recall that we have used $10^9$ samples to compute the actual Mallows and RUM noise model matrices. The noise model matrices we obtained are as follows:
$$P_{\text{mallows}}^{\text{100}} = \left[
\begin{array}{cccccc}
    0.59  &  0.19  &  0.07  &  0.08  &  0.06  &  0.01 \\
    0.19  &  0.44  &  0.18  &  0.09  &  0.04  &  0.06 \\
    0.10  &  0.19  &  0.43  &  0.19  &  0.07  &  0.02 \\
    0.05  &  0.05  &  0.15  &  0.45  &  0.19  &  0.11 \\
    0.06  &  0.10  &  0.09  &  0.14  &  0.46  &  0.15 \\
    0.01  &  0.03  &  0.08  &  0.05  &  0.18  &  0.65
\end{array} \right]$$
and
$$P_{\text{mallows}}^{\text{1000}} = \left[
\begin{array}{cccccc}
    0.639  &  0.186  &  0.066  &  0.058  &  0.031  &  0.020 \\
    0.193  &  0.534  &  0.150  &  0.055  &  0.032  &  0.036 \\
    0.073  &  0.149  &  0.501  &  0.147  &  0.076  &  0.054 \\
    0.039  &  0.075  &  0.155  &  0.497  &  0.147  &  0.087 \\
    0.033  &  0.038  &  0.071  &  0.163  &  0.517  &  0.178 \\
    0.023  &  0.018  &  0.057  &  0.080  &  0.197  &  0.625
\end{array} \right]$$
for the Mallows model, and
$$P_{\text{rum}}^{\text{100}} = \left[
\begin{array}{cccccc}
    0.49  &  0.19  &   0.08   &  0.06  &  0.11  &   0.07 \\
    0.22  &  0.36  &   0.25  &  0.05  &  0.04   &   0.08 \\
    0.09  &  0.12  &   0.31  &  0.22  &  0.14   &  0.12 \\
    0.04  &  0.11  &   0.23  & 0.35   &  0.13   & 0.14 \\
    0.08  &  0.12  &   0.08  &  0.20  &  0.37   &  0.15 \\
    0.08  &  0.10  &   0.05  &  0.12  &  0.21   & 0.44
\end{array} \right]$$
and
$$P_{\text{rum}}^{\text{1000}} = \left[
\begin{array}{cccccc}
    0.506 &   0.154 &   0.080 &   0.095 &   0.075 &   0.090 \\
    0.156 &   0.401 &   0.186 &   0.093 &   0.084 &   0.080  \\
    0.088  &  0.163 &   0.385 &   0.173 &   0.107 &   0.084 \\
    0.088 &   0.110 &   0.159 &   0.374 &   0.183 &   0.086 \\
    0.076 &   0.088 &   0.100 &   0.190 &   0.386 &   0.160 \\
    0.086 &   0.084 &   0.090 &   0.075 &   0.165 &   0.500
\end{array} \right]$$
for the RUM model. The matrices have been used to compute the optimal type-ordering aggregation rules for the five bivariate performance objectives using our theoretical framework from Sections~\ref{subsec:framework} and~\ref{subsec:optimal}.

Interestingly, the instances of FAS that we had to solve were slightly harder now. In particular, for the $100$-sample Mallows noise model and the $100$-sample RUM noise model, we had strongly connected components of size up to $26$ and $89$, respectively. Still, our methodology was applied smoothly and allowed us to compute optimal rules. Recall that (see Sections~\ref{subsec:optimal} and~\ref{subsec:optimal-exp}) for strongly connected components of size larger than $10$, we use  Borda orderings of the types within the component, instead of computing the optimal ordering by brute forcing (which is prohibitive for so large components). Hence, an important question is how close to optimality are the type-ordering aggregation rules that we come up with, when we inevitably resort to Borda orderings. To answer this, we compute an upper bound on the performance of the optimal rules by considering the maximum weight edge between any pair of types that are part of a strongly connected component. Note that this gives an upper bound on the contribution of the strongly connected components to the total weight of the optimal solution of the FAS instance  since, in general, taking the maximum weight edges may lead to cycles. Then, we can see how close the performance of our rules is to this upper bound. Table \ref{tab:error-difference} contains this information; it should be clear from the almost zero values reported there that the type-ordering aggregation rules that we obtain are extremely close to being optimal.

\begin{table}
\centering
\begin{tabular}{ lcccc c }
\noalign{\hrule height 1.5pt}
model           & all2all  & acc-2\%  & acc-5\% & th-10\% & th-50\%    \\\noalign{\hrule height 1.5pt}
realistic 2015 	& 	$0$    & $0$      & $0$     & $0$     & $0$        \\ 
realistic 2016  & $0$     & $0$     & $0$     & $0$     & $0$        \\ 
mallows         & $6\cdot 10^{-6}$ & $6\cdot 10^{-6}$  & $6\cdot 10^{-6}$  & $0$  & $0$       \\ 
100-mallows  & $5\cdot 10^{-5}$ &$5\cdot 10^{-5}$	  & $5\cdot 10^{-5}$   & $0$  & $0$        \\ 
1000-mallows & $0$              & $8\cdot 10^{-6}$    & $1\cdot 10^{-5}$   & $0$  & $0$        \\ 
rum          & $2\cdot 10^{-4}$ & $2\cdot 10^{-4}$    & $1\cdot 10^{-4}$   & $0$  &  $2\cdot 10^{-5}$ \\ 
100-rum     & $1\cdot 10^{-4}$ & $2\cdot 10^{-4}$     & $7\cdot 10^{-4}$   & $0$  & $0$        \\ 
1000-rum    & $6\cdot 10^{-4}$ & $6\cdot 10^{-4}$     & $2\cdot 10^{-4}$   & $0$  &  $1\cdot 10^{-5}$ \\ 
\noalign{\hrule height 1.5pt}
\end{tabular}
 \caption{Upper bound error for all scenarios. The numbers depicted correspond to the difference between the performance of the respective type-ordering aggregation rule that is computed using Borda orderings within the strongly connected components of size larger than 10 and the theoretical upper bound on the performance of the optimal rule, which is computed by accounting for the maximum weight between any two types that are part of a big strongly connected component. The value $0$ indicates that there were no such big components and the rule is actually optimal. As one can observe, even in cases where we did have big components, the error is extremely close to $0$.}
\label{tab:error-difference}
\end{table}

Table \ref{tab:learning} shows the theoretical prediction values of the $100$- and $1000$-sample approximation type-ordering aggregation rules. The performance of the type-ordering aggregation rules that were computed using the $100$-sample approximations are already amazingly close to those for the actual models. For the rules that we computed using the $1000$-sample approximation, it is almost impossible to distinguish them from the actual ones, in terms of performance. 

\begin{table}[h!]
\centering
\begin{tabular}{ccccccc}	
\noalign{\hrule height 1.5pt}
\# samples  & \multicolumn{2}{c}{100} & \multicolumn{2}{c}{1000}& \multicolumn{2}{c}{$10^9$} \\\hline
setting & mallows & rum & mallows & rum & mallows & rum  \\\noalign{\hrule height 1.5pt}
all2all	&	84.95	&	77.51	&	85.14	&	77.85	&	85.15	&	77.89	\\
th-10\%	&	91.82	&	86.58	&	92.05	&	87.08	&	92.05	&	87.11	\\
th-50\%	&	88.21	&	80.84	&	88.39	&	81.25	&	88.39	&	81.27	\\
acc-2\%	&	86.31	&	78.59	&	86.51	&	78.95	&	86.52	&	78.99	\\
acc-5\%	&	88.19	&	80.21	&	88.41	&	80.51	&	88.42	&	80.57	\\
\noalign{\hrule height 1.5pt}
\end{tabular}
\caption{Theoretical performance prediction of the optimal type-ordering aggregation rules for the $100$- and $1000$-sample approximations of the Mallows and RUM model.}
\label{tab:learning}
\end{table}

A more refined graphical representation of these findings is given in Figure \ref{fig:clouds-apx} (best viewed in color). Each plot contains a blue and a red cloud of 1000 points, each corresponding to a single simulated exam with $10\,000$ students. The blue points (respectively, red points) show the performance of the optimal rule for the $1000$-sample (respectively, $100$-sample) Mallows and RUM approximation versus the Mallows- and RUM-optimal rule in subfigures (a)--(c) and (d)--(f), respectively. In all cases, the blue cloud almost coincides with the diagonal in each plot, indicating an optimal approximation of the optimal rule. The red cloud is distinct but still very close. To realize how close the two clouds are, for the case of the Mallows model, almost the whole cloud of points for Borda (from Figures \ref{fig:clouds-borda-vs-opt}(g)--(i)) would be located outside the plot area of Figure \ref{fig:clouds-apx}(a)--(c) (if we attempted to plot it).

\begin{figure*}[ht]
\centering
\begin{subfigure}{0.3\textwidth}
   \includegraphics[trim= 50 200 50 260, clip=true, scale=0.28]{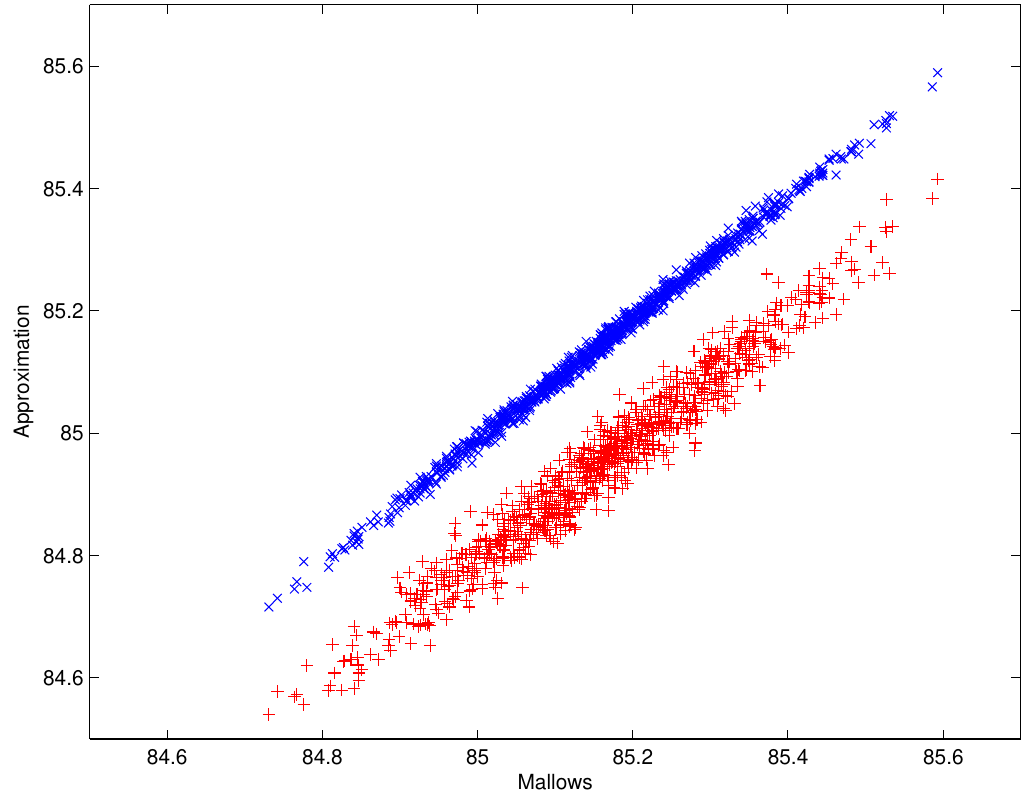} 
   \caption{all2all, Mallows}
\end{subfigure}
\begin{subfigure}{0.3\textwidth}
   \includegraphics[trim= 50 200 50 260, clip=true, scale=0.28]{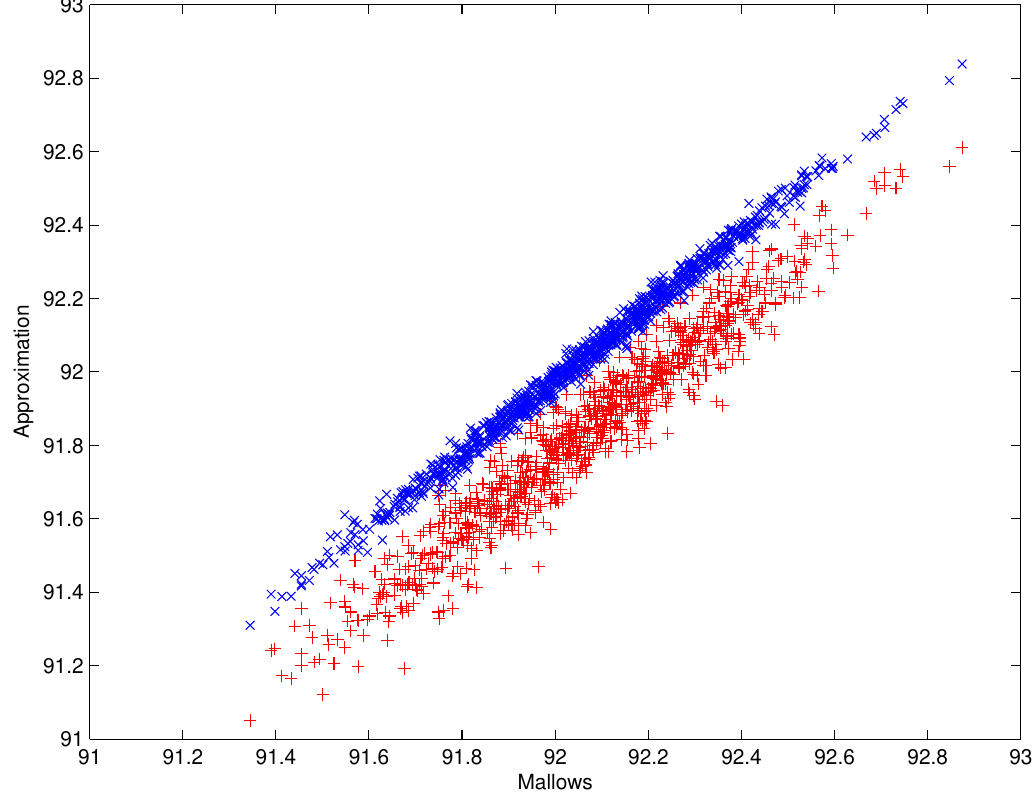} 
   \caption{th-10\%, Mallows}
\end{subfigure}
\begin{subfigure}{0.3\textwidth}
   \includegraphics[trim= 50 200 50 260, clip=true, scale=0.28]{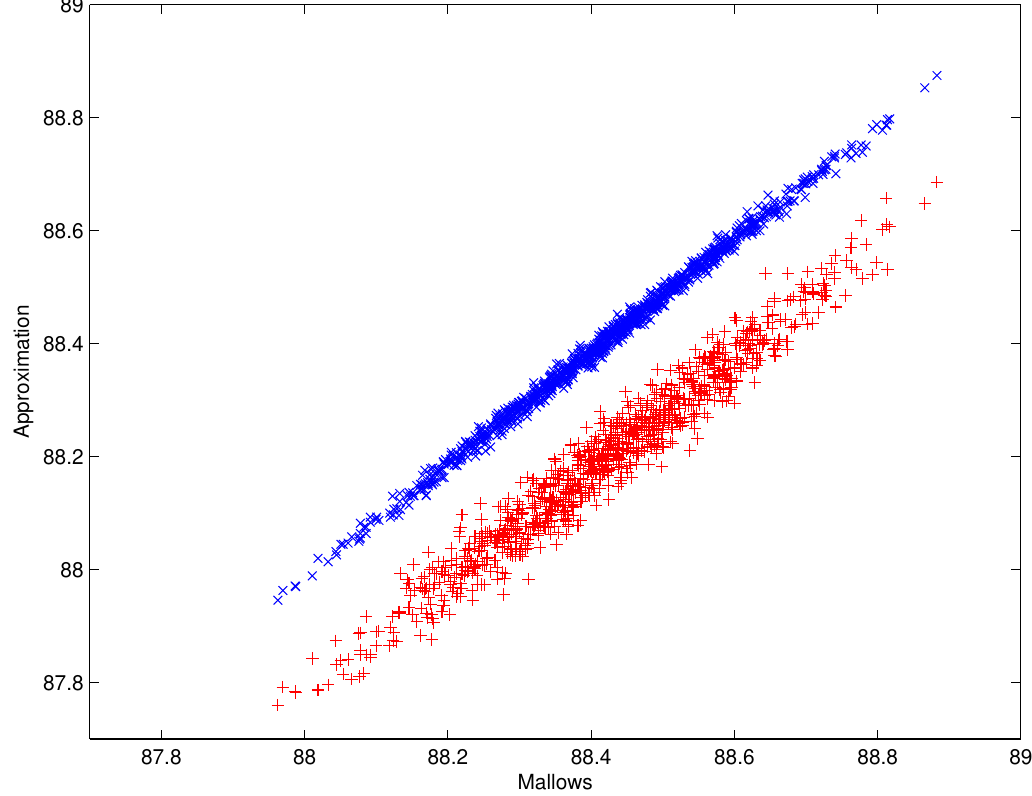} 
   \caption{acc-5\%, Mallows}
\end{subfigure}
\begin{subfigure}{0.3\textwidth}
   \includegraphics[trim= 50 200 50 200, clip=true, scale=0.28]{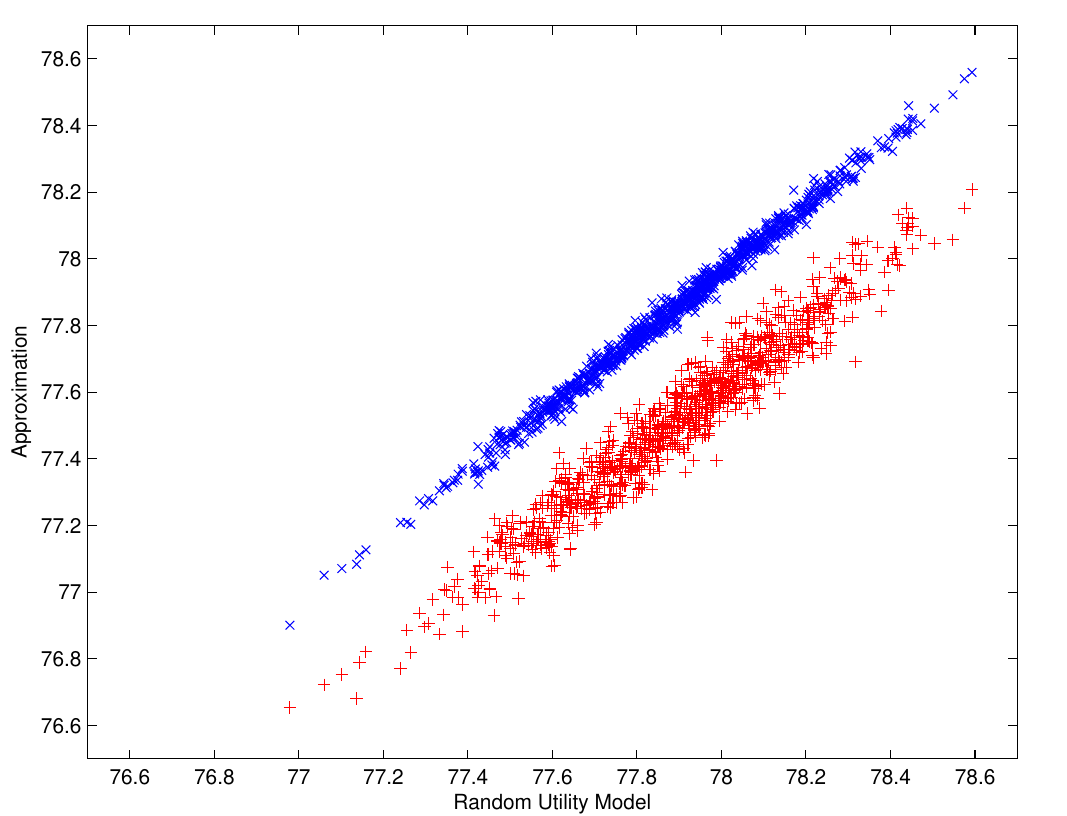} 
   \caption{all2all, RUM}
\end{subfigure}
\begin{subfigure}{0.3\textwidth}
   \includegraphics[trim= 50 200 50 200, clip=true, scale=0.28]{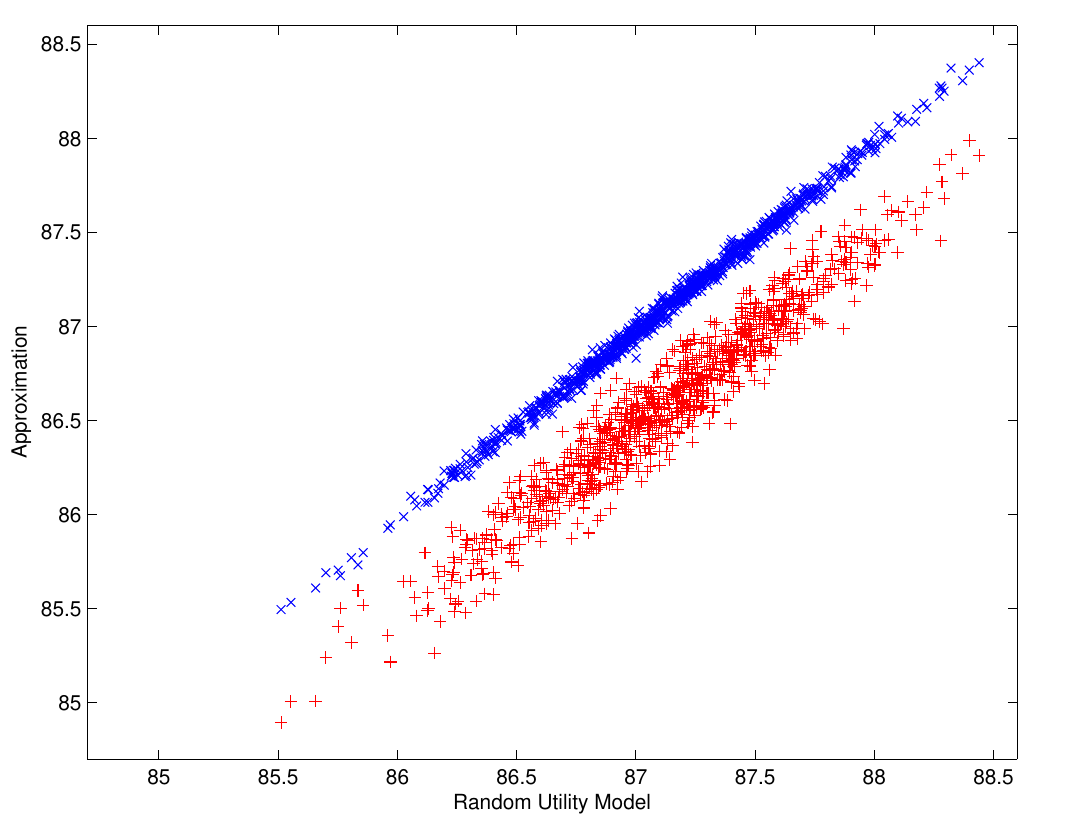} 
   \caption{th-10\%, RUM}
\end{subfigure}
\begin{subfigure}{0.3\textwidth}
   \includegraphics[trim= 50 200 50 200, clip=true, scale=0.28]{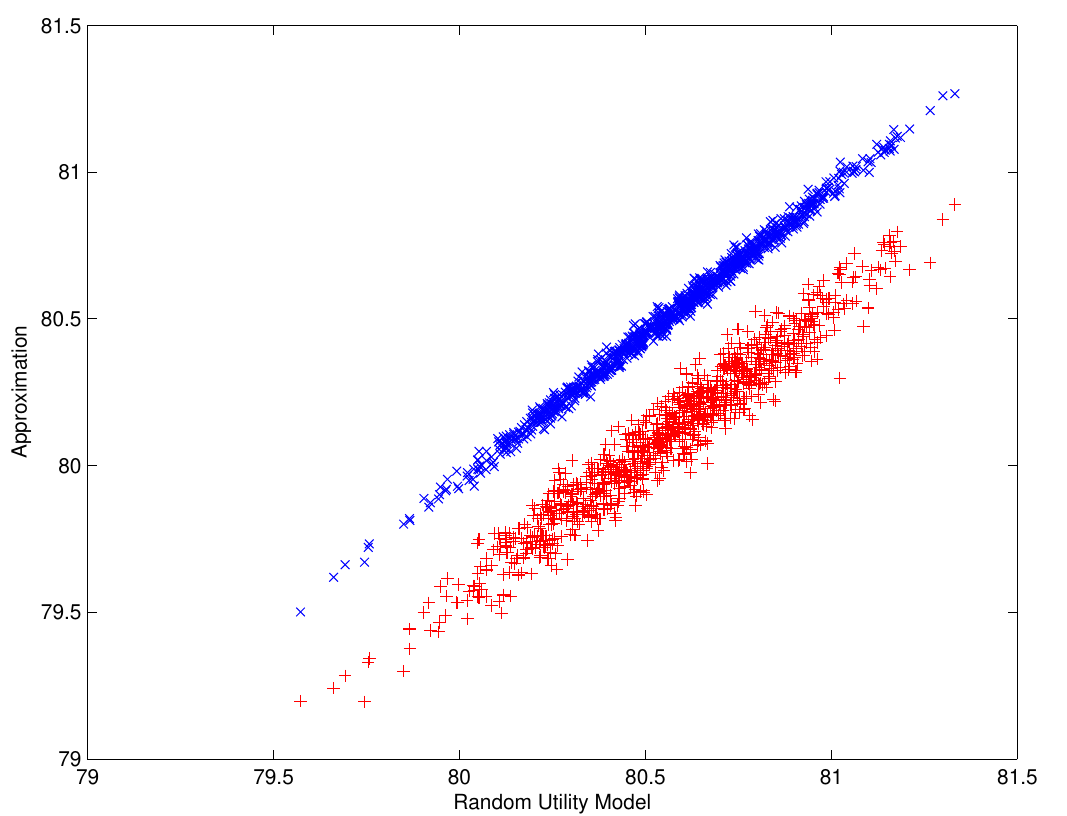} 
   \caption{acc-5\%, RUM}
\end{subfigure}
\caption{A comparison of the optimal rule for Mallows and RUM and their approximations with respect to the objectives (a) all2all, (b) acc-5\%, and (c) th-10\%. The depicted results are from 1000 executions of the type-ordering aggregation rules on simulated exams with $10\,000$ students.}
\label{fig:clouds-apx}
\end{figure*}

We conclude the presentation of our simulation results with a comparison of the type orderings of the optimal aggregation rules for Mallows and RUM and their $100$- and $1000$-sample approximations; these are presented in Tables \ref{tab:rules} and \ref{tab:rum-rules}. Therein, we can see that the optimal rules for the $1000$-sample noise models according to the all2all performance objective are very close (but not identical) to the Mallows- and RUM-optimal rules. The optimal rules for the $100$-sample noise models are substantially different (these differences are more apparent in lower positions of the orderings which cannot be included here). An interesting characteristic of optimal rules for Mallows, RUM and their approximations is that the orderings of types are non-monotonic. For example, type $(1,1,1,1,1,5)$ is always ahead of $(1,1,1,1,1,2)$ for the Mallows model and its approximations. This justifies our decision to study type-ordering aggregation rules and ignore positional scoring rules; clearly, no positional scoring rule would come up with non-monotonic orderings of types.

\begin{table}
\centering
\begin{tabular}{cccccc}
\noalign{\hrule height 1.5pt}
pos. & mallows & $100$-sample & $1000$-sample \\\noalign{\hrule height 1.5pt}
1& $(1,1,1,1,1,1)$ & $(1,1,1,1,1,1)$ & $(1,1,1,1,1,1)$	\\
2& $(1,1,1,1,1,6)$ & $(1,1,1,1,1,5)$ & $(1,1,1,1,1,6)$	\\
3& $(1,1,1,1,1,5)$ & $(1,1,1,1,1,2)$ & $(1,1,1,1,1,5)$	\\
4& $(1,1,1,1,1,2)$ & $(1,1,1,1,1,4)$ & $(1,1,1,1,1,2)$	\\
5& $(1,1,1,1,1,4)$ & $(1,1,1,1,1,3)$ & $(1,1,1,1,1,4)$	\\
6& $(1,1,1,1,1,3)$ & $(1,1,1,1,1,6)$ & $(1,1,1,1,1,3)$	\\
7& $(1,1,1,1,2,6)$ & $(1,1,1,1,2,2)$ & $(1,1,1,1,2,6)$	\\
8& $(1,1,1,1,2,2)$ & $(1,1,1,1,2,5)$ & $(1,1,1,1,2,2)$	\\
9& $(1,1,1,1,6,6)$ & $(1,1,1,1,5,5)$ & $(1,1,1,1,2,5)$	\\
10& $(1,1,1,1,2,5)$ & $(1,1,1,1,2,4)$ & $(1,1,1,1,6,6)$	\\
11& $(1,1,1,1,5,6)$ & $(1,1,1,1,2,3)$ & $(1,1,1,1,5,6)$	\\
12& $(1,1,1,1,2,4)$ & $(1,1,1,1,3,5)$ & $(1,1,1,1,5,5)$	\\
13& $(1,1,1,1,2,3)$ & $(1,1,1,1,4,5)$ & $(1,1,1,1,2,4)$	\\
14& $(1,1,1,1,5,5)$ & $(1,1,1,1,3,3)$ & $(1,1,1,1,2,3)$	\\
\noalign{\hrule height 1.5pt}
\end{tabular}
\caption{The first 14 types in the type ordering of the optimal rules for Mallows and its 100-sample and 1000-sample approximations, according to the all2all performance objective.}
\label{tab:rules}
\end{table}

\begin{table}[h]
\centering
\begin{tabular}{cccccc}
\noalign{\hrule height 1.5pt}
pos. & rum & $100$-sample & $1000$-sample \\\noalign{\hrule height 1.5pt}
1 	&	$(1,1,1,1,1,1)$	&	$(1,1,1,1,1,1)$		& $(1,1,1,1,1,1)$ \\
2 	&	$(1,1,1,1,1,6)$	&	$(1,1,1,1,1,6)$		&$(1,1,1,1,1,6)$ \\
3 	&	$(1,1,1,1,1,5)$	&	$(1,1,1,1,1,2)$		&$(1,1,1,1,1,5)$ \\
4 	&	$(1,1,1,1,1,4)$	&	$(1,1,1,1,1,5)$		&$(1,1,1,1,1,4)$ \\
5 	&	$(1,1,1,1,1,2)$	&	$(1,1,1,1,1,3)$		&$(1,1,1,1,1,2)$ \\
6 	&	$(1,1,1,1,1,3)$	&	$(1,1,1,1,2,6)$		&$(1,1,1,1,1,3)$ \\
7 	&	$(1,1,1,1,6,6)$	&	$(1,1,1,1,2,2)$		&$(1,1,1,1,6,6)$ \\
8 	&	$(1,1,1,1,5,6)$	&	$(1,1,1,1,6,6)$		&$(1,1,1,1,2,6)$ \\
9 	&	$(1,1,1,1,2,6)$	&	$(1,1,1,1,1,4)$		&$(1,1,1,1,5,6)$ \\
10 	&	$(1,1,1,1,5,5)$	&	$(1,1,1,1,2,5)$		&$(1,1,1,1,4,6)$ \\
11	&	$(1,1,1,1,2,5)$	&	$(1,1,1,1,5,6)$		&$(1,1,1,1,2,5)$ \\
12 	&	$(1,1,1,1,4,6)$	& 	$(1,1,1,1,2,3)$		&$(1,1,1,1,5,5)$ \\
13 	&	$(1,1,1,1,2,2)$	& 	$(1,1,1,1,5,5)$		&$(1,1,1,1,3,6)$ \\
14 	&	$(1,1,1,1,3,6)$	& 	$(1,1,1,1,3,6)$		&$(1,1,1,1,2,2)$ \\
\noalign{\hrule height 1.5pt}
\end{tabular}
\caption{The first 14 types in the type ordering of the optimal rules for RUM and its 100-sample and 1000-sample approximations, according to the all2all performance objective.}
\label{tab:rum-rules}
\end{table}

\section{Open problems and further research}\label{sec:open}
In this paper, we have developed a theoretical framework for performance prediction and optimization over a class of rank aggregation rules for ordinal peer grading in MOOCs. Our work reveals many challenging future research directions regarding ordinal peer grading and the deployment of our methods to real MOOCs. An obvious first direction is to develop an analogous framework for broader classes of aggregation rules. This framework will be most useful if it allows for selecting the optimal aggregation rule for a given scenario, as we have managed to do for type-ordering aggregation rules here.

In the deployment of ordinal peer grading in real MOOCs, a few professional graders may be actually available; in technical terms, this implies a partial knowledge of the ground truth~\citep{GWL16}. How should this partial knowledge be combined with rank aggregation of students' grading in order to get an even better final ranking? This question seems to suggest nice extensions to our theory. Another issue that we have completely neglected here is related to the common student drop out after their participation in an exam but before its grading. Even though we do not believe that such situations invalidate our methods, such issues have to be taken seriously into account before deciding which rank aggregation rules to deploy in real systems.

Finally, a thread of interesting research questions is related to incentives; e.g., see \citet{KLMP15} and \citet{ALMRW16}. Classical impossibilities in social choice theory imply that students may grade strategically in order to improve their own position in the final outcome. Can this strategic behaviour be taken into account when deciding the optimal rank aggregation rule? Our approach might be possible to adapt to strategic graders but this would require challenging technical work.

\bibliographystyle{named}
\bibliography{ordinal.bib}

\appendix

\section{Computing the weights}\label{app:sec:compute-integral}
We now elaborate on how to analytically compute the weight $W(\sigma,\sigma')$; the following computations are implemented in Algorithm~\ref{alg:weights}. Recall that $W(\sigma,\sigma')$ is given by equation (\ref{eq:new-weight}), i.e., 
\begin{align*}
W(\sigma,\sigma') = \int_0^1{\int_x^1{f(x,y)\Pr[x \rhd \sigma]\cdot \Pr[y \rhd \sigma']\ud y}\ud x}, 
\end{align*}
and that the performance objective bivariate function $f$ indicates whether a correctly recovered pairwise relation between two students $x$ and $y$ (with $x<y$) should be accounted for or not. All performance objectives we consider in this paper can generically be described by such a function $f$ with $f(x,y) = 1$ when $x \in [\alpha,\beta]$ and $y \in [x +\gamma,\delta]$ for appropriate values of $\alpha, \beta, \gamma, \delta \in [0,1]$, and $f(x,y)=0$ otherwise. In particular, all2all can be expressed with the tuple $(\alpha,\beta,\gamma,\delta)=(0,1,0,1)$, th-10\% and th-50\% with the tuples $(0,0.1,0,1)$ and
$(0,0.5,0,1)$, and acc-2\% and acc-5\% with the tuples $(0,0.98,0.02,1)$ and $(0,0.95,0.05,1)$. 

Then, $W(\sigma,\sigma')$ is equal to 
\begin{align} \nonumber
W(\sigma,\sigma') &= \int_\alpha^\beta{\int_{x+\gamma}^\delta{ \Pr[x \rhd \sigma]\cdot \Pr[y \rhd \sigma']\ud y}\ud x} \\\label{eq:weight-parameterized}
&= \int_\alpha^\beta{ \Pr[x \rhd \sigma] \int_{x+\gamma}^\delta{ \Pr[y \rhd \sigma']\ud y} \ud x}.
\end{align}

Now, recall that $\Pr[x \rhd \sigma]$ is given by equation (\ref{eq:prob-x-sigma-is-a-poly}) and is a univariate polynomial of degree $k^2-k$. Hence, it can be written as
\begin{align}\label{eq:x-poly-form}
\Pr[x \rhd \sigma] &= \sum_{s=0}^{k^2-k}{ c_s(\sigma) x^s },
\end{align} 
where the coefficients $c_s(\sigma)$ with $s=0, ..., k^2-k$ are computed by equation (\ref{eq:prob-x-sigma-is-a-poly}), and have been included for completeness in the first part of Algorithm~\ref{alg:weights}. 

The inner integral in equation (\ref{eq:weight-parameterized}) is computed as follows:
\begin{align*}
\int_{x+\gamma}^\delta{ \Pr[y \rhd \sigma'] }\ud y &= \int_{x+\gamma}^\delta{  \sum_{s=0}^{k^2-k}{ c_s(\sigma') y^s } }\ud y \\
&= \sum_{s=0}^{k^2-k}{ c_s(\sigma') \int_{x+\gamma}^\delta{ y^s } \ud y } \\
&= \sum_{s=0}^{k^2-k}{ \frac{c_s(\sigma')}{s+1} \left( \delta^{s+1} - (x+\gamma)^{s+1} \right) } \\
&= \sum_{s=0}^{k^2-k}{ \frac{c_s(\sigma')\delta^{s+1}}{s+1} } - \sum_{s=1}^{k^2-k+1}{ \frac{c_s(\sigma')(x+\gamma)^s}{s} }. 
\end{align*}
Using the fact that $(z+w)^m = \sum_{i=0}^m{ {m \choose i}z^{m-i}w^i}$ for $z=\gamma$, $w=x$, and $m=s$, we obtain
\begin{align} \label{eq:y-integral}
\int_{x+\gamma}^\delta{  \Pr[y \rhd \sigma'] }\ud y = \sum_{s=0}^{k^2-k}{ \frac{c_s(\sigma')\delta^{s+1}}{s+1} } - \sum_{s=1}^{k^2-k+1}{c_s(\sigma')\sum_{i=0}^s{ \frac{1}{s}{s \choose i}\gamma^{s-i} x^i }}.
\end{align}
Observe that the inner integral is also a univariate polynomial of degree $k^2-k+1$ and it can be written as 
\begin{align}\label{eq:y-poly-form}
\int_{z+\gamma}^\delta{  \Pr[y \rhd \sigma'] }\ud y = \sum_{t=0}^{k^2-k+1}{ d_t(\sigma) x^t }, 
\end{align}
where the coefficients $d_t(\sigma')$ with $t=0, ..., k^2-k+1$ are computed by equation (\ref{eq:y-integral}) at the second part of Algorithm~\ref{alg:weights}. 

By substituting equations (\ref{eq:x-poly-form}) and (\ref{eq:y-poly-form}) in equation (\ref{eq:weight-parameterized}), we obtain
\begin{align*}
W(\sigma,\sigma') &= \int_\alpha^\beta{ \Pr[x \rhd \sigma] \bigg( \int_{x+\gamma}^\delta{ \Pr[y \rhd \sigma']\ud y} \bigg) \ud x} \\
&= \int_\alpha^\beta{ \sum_{s=0}^{k^2-k}{ c_s(\sigma) x^s} \sum_{t=0}^{k^2-k+1}{ d_t(\sigma') x^t } }\ud z \\
&= \sum_{s=0}^{k^2-k}\sum_{t=0}^{k^2-k+1}{ c_s(\sigma) d_t(\sigma') \int_\alpha^\beta {x^{s+t} }\ud z } \\
&= \sum_{s=0}^{k^2-k}\sum_{t=0}^{k^2-k+1}{ \frac{c_s(\sigma) d_t(\sigma')}{s+t+1} \bigg( \beta^{s+t+1} - \alpha^{s+t+1} \bigg) }.
\end{align*}
This computation is described in the last part of Algorithm~\ref{alg:weights}.

\begin{algorithm}[t]
    \SetAlgoLined
    \tcp{Compute the coefficient vector $\cc(\sigma)$}
    \For{$s:=0 \dots k^2-k$}{
        set $c_s(\sigma) := 0$
    }
    \For{$\ell_1 := 1 \dots k$}{
    $\ddots$    \\
    \For{$\ell_k := 1 \dots k$}{
       set $|\ell|_1  := \sum_{i=1}^k{\ell_i}$\\
       \For{$j := 0 \dots k^2-|\ell|_1$}{ 
           set $c_{|\ell|_1-k+j}(\sigma) := c_{|\ell|_1-k+j}(\sigma) + N(\sigma) \bigg( \prod_{i=1}^k{ p_{\sigma_i,\ell_i} {k-1 \choose \ell_i-1}} \bigg) {k^2-|\ell|_1 \choose j}(-1)^j$
       }
    }
    }    
    \tcp{Compute the coefficient vector $\dd(\sigma')$}
    set $d_0(\sigma') := \sum_{s=0}^{k^2-k}{ \frac{c_s(\sigma')\delta^{s+1}}{s+1} }$ \\
    \For{$t :=1 \dots k^2-k+1$}{
        set $d_t(\sigma') := 0$
    }
    \For{$s := 1 \dots k^2-k+1$}{
       \For{$i := 0 \dots s$}{
           set $d_i(\sigma') := d_i(\sigma') - \frac{c_s(\sigma')}{s}{s \choose i}\gamma^{s-i}$
       }
    }
    \tcp{Compute $W(\sigma,\sigma')$}
    set $W(\sigma,\sigma') := 0$ \\
    \For{$s := 0 \dots k^2-k$}{
    \For{$t := 0 \dots k^2-k+1$}{
        set $W(\sigma,\sigma') := W(\sigma,\sigma') + \frac{c_s(\sigma) d_t(\sigma')}{s+t+1} \left( \beta^{s+t+1} - \alpha^{s+t+1} \right)$
    }    
    }
\caption{Computing $W(\sigma,\sigma')$}
\label{alg:weights}
\end{algorithm}

\section{Formal analysis of type-ordering aggregation rules}\label{app:sec:formal}
Our assumptions about infinite number of students make our analysis in Section~\ref{subsec:framework} non-rigorous. We now present a rigorous analysis that handles formally all subtleties involved. We denote by $n$ the number of students and by $k$ the bundle size. 

In our non-rigorous analysis, exam papers are represented by their fractional true ranks in $[0,1]$. Here, as the number of exam papers is considered to be finite, we adjust this notation as follows. The integer $\chi\in [n]$ will denote both an exam paper and its true rank (i.e., exam paper $\chi$ is the $\chi$-th best paper in the ground truth). For an exam paper $\chi$ and type $\sigma\in \mathcal{T}_k$, we will use ${\Pr}_n[\chi\rhd \sigma]$ to denote the probability that $\chi$ gets type $\sigma$. Then, the notation $\Pr[x\rhd \sigma]$ that is used in Section~\ref{subsec:framework} can be thought of as the limit, as $n$ approaches infinity, of the probability that exam paper $\chi$ (with $x=\chi/n$) gets type $\sigma$.

Recall that for a type $\sigma\in {\mathcal{T}_k}$, the quantity $N(\sigma)$ denotes the number of different ways the graders can give type $\sigma$ to a given exam paper. It can be easily seen that $N(\sigma)\leq k!$. The notation $L_k$ is again used to denote the set of all $k$-entry vectors $\ell=(\ell_1, ..., \ell_k)$ with $\ell_i\in [k]$. We use the abbreviation $|\ell|_1=\sum_{t=1}^k{\ell_t}$. Finally, $p$ is the noise matrix.

For a real $z\in [0,1]$, define 
\begin{align*}
\theta_\sigma(z) &= 
N(\sigma) \sum_{\ell\in L_k}{\bigg(\prod_{i=1}^k{ p_{\sigma_i,\ell_i} {k-1 \choose \ell_i-1}} \bigg)  z^{|\ell|_1-k}(1-z)^{k^2-|\ell|_1}}.
\end{align*}
Notice that $\theta_\sigma(x)$ is the equivalent expression (\ref{eq:theta}) for $\Pr[x\rhd \sigma]$ in our non-rigorous analysis. Clearly, $\theta_\sigma(z)\leq 1$ for every $z\in [0,1]$. 

Here, we will focus on exam paper $\chi\in [n]$ and will show that $\theta_\sigma(\chi/n)$ is an approximation for ${\Pr}_n[\chi\rhd \sigma]$, which becomes sharp as $n$ approaches infinity. This approximation is stated in Corollary~\ref{cor:approx}, which is obtained through the two next Lemmas~\ref{lem:approx-via-x-rhd-sigma-given-delta} and~\ref{lem:approx-x-rhd-sigma-given-delta}. 

We denote by $\Delta$ the event that the exam papers that are contained in the $k$ bundles in which $\chi$ appears are all different. So, when $\Delta$ is true, we can view the $k(k-1)$ exam papers that appear in bundles together with $\chi$ as selected uniformly at random without replacement among all exam papers besides $\chi$.

Also, we denote by $\eta(k)$ a sufficiently large quantity that depends only on $k$. Setting $\eta(k)=k^{k^2+2k+4}$ is enough for our proof below. Note that we have made no particular attempt to optimize $\eta(k)$.

\begin{lemma}\label{lem:approx-via-x-rhd-sigma-given-delta}
For every exam paper $\chi\in [n]$ and type $\sigma\in \mathcal{T}_k$, it holds that $$|{\Pr}_n[\chi \rhd \sigma] - {\Pr}_n[\chi\rhd \sigma|\Delta]|\leq \eta(k)/n.$$
\end{lemma}

\begin{proof}
Using the law of total probability, we have
\begin{align*}
{\Pr}_n[\chi\rhd \sigma] &= {\Pr}_n[\chi\rhd \sigma|\Delta]\cdot{\Pr}_n[\Delta]+{\Pr}_n[\chi\rhd \sigma|\overline{\Delta}]\cdot{\Pr}_n[\overline{\Delta}]\\
&= {\Pr}_n[\chi\rhd \sigma|\Delta]+{\Pr}_n[\overline{\Delta}]\cdot\left({\Pr}_n[\chi\rhd \sigma|\overline{\Delta}]-{\Pr}_n[\chi\rhd \sigma|\Delta]\right),
\end{align*}
which implies that
\begin{align}\label{eq:bound-by-overline-delta}
|{\Pr}_n[\chi\rhd \sigma]-{\Pr}_n[\chi\rhd \sigma|\Delta]| &\leq {\Pr}_n[\overline{\Delta}].
\end{align}

It remains to bound ${\Pr}_n[\overline{\Delta}]$. Consider the random process of forming the bundles (recall the discussion in Section \ref{subsec:bundles} and specifically footnote 3). The process consists of $k$ rounds. For $i=1, ..., k$, the $i$-th exam paper in each bundle is decided in round $i$. We denote by $B_i$ the bundle that receives exam paper $\chi$ in round $i$. Without loss of generality, we assume that, in each round $j$, the $j$-th exam paper in the bundles is decided as follows. By definition, if $j=i$, the $j$-th exam paper in bundle $B_i$ is exam paper $\chi$. If $i\not=j$, the $j$-th exam paper in bundle $B_i$ is selected uniformly at random among all exam papers besides exam paper $\chi$, the exam papers that have been included in bundle $B_i$ in rounds $1$, $2$, ..., $j-1$, and the $j$-th exam paper of bundles $B_1$, $B_2$, ..., $B_{i-1}$. Hence, if $i\not=j$, the $j$-th exam paper in bundle $B_i$ is selected uniformly at random among $n-i-j+2$ exam papers.

The number of distinct exam papers that have been included in bundles before deciding the $j$-th exam paper of bundle $B_i$ is $(j-1)(k-1)+i$ if $i<j$ and $(j-1)(k-1)+i-1$ if $i>j$. 
Hence, 
\begin{align*}
{\Pr}_n[\Delta] &= \prod_{j=1}^k{\prod_{i=1}^{j-1}{\frac{n-(j-1)(k-1)-i}{n-i-j+2}}\prod_{i=j+1}^k{\frac{n-(j-1)(k-1)-i+1}{n-i-j+2}}}\\
&\geq  \prod_{j=1}^k{\prod_{i=1}^{j-1}{\frac{n-k^2}{n}}\prod_{i=j+1}^k{\frac{n-k^2}{n}}} \geq \left(1-\frac{k^2}{n}\right)^{k^2} \geq 1-\frac{k^4}{n}\geq 1-\eta(k)/n.
\end{align*}
The second last inequality follows by Bernoulli inequality. Hence, ${\Pr}_n[\overline{\Delta}]\leq \eta(k)/n$ and the lemma follows due to inequality (\ref{eq:bound-by-overline-delta}).
\end{proof}

\begin{lemma}\label{lem:approx-x-rhd-sigma-given-delta} 
For every exam paper $\chi\in [n]$ and type $\sigma\in \mathcal{T}_k$, it holds that $$|{\Pr}_n[\chi \rhd \sigma|\Delta]- \theta_\sigma(\chi/n)|\leq \eta(k)/n.$$
\end{lemma}

\begin{proof}
Let $B_1$, $B_2$, ..., $B_k$ be the $k$ bundles which contain exam paper $\chi$. In order to compute ${\Pr}_n[\chi\rhd \sigma|\Delta]$, we will compute the probability that $\chi$ will be ranked $\sigma_i$-th in bundle $B_i$ for $i=1, ..., k$ and, due to symmetry, we will multiply by $N(\sigma)$ in order to account for all possible different ways to get type $\sigma$. We denote by $\mathcal{E}_i$ the event that exam paper $\chi$ is ranked $\sigma_i$-th by the grader of bundle $B_i$. For a vector $\ell=(\ell_1, ..., \ell_k)\in L_k$, we denote by $\mathcal{Z}_i$ the event that $\chi$ has true rank $\ell_i$ among the exam papers in bundle $B_i$ (i.e., $\chi$ is the $\ell_i$-th best among the exam papers in bundle $B_i$). Then,
\begin{align}\nonumber
{\Pr}_n[\chi\rhd \sigma|\Delta] &= N(\sigma)\cdot {\Pr}_n[\cap_{i=1}^k{\mathcal{E}_i}|\Delta]\\\label{eq:x-sigma}
&= N(\sigma) \cdot \sum_{\ell\in L_k}{{\Pr}_n[\cap_{i=1}^k{\mathcal{Z}_i}|\Delta]\cdot {\Pr}_n[\cap_{i=1}^k{\mathcal{E}_i}|\cap_{i=1}^k{\mathcal{Z}_i},\Delta]}.
\end{align}
Now, observe that 
\begin{align}\label{eq:graders}
{\Pr}_n[\cap_{i=1}^k{\mathcal{E}_i}|\cap_{i=1}^k{\mathcal{Z}_i},\Delta] &= \prod_{i=1}^k{{\Pr}_n[\mathcal{E}_i|\mathcal{Z}_i]} = \prod_{i=1}^k{p_{\sigma_i,\ell_i}}
\end{align}
by the definition of the noise matrix. Furthermore, using the chain rule, 
\begin{align}\label{eq:prob-true}
{\Pr}_n[\cap_{i=1}^k{\mathcal{Z}_i}|\Delta] &= \prod_{i=1}^k{{\Pr}_n[\mathcal{Z}_i|\mathcal{Z}_1, ..., \mathcal{Z}_{i-1},\Delta]}.
\end{align}

Now, the conditions $\mathcal{Z}_1$, ..., $\mathcal{Z}_{i-1}$ indicate that among the exam papers in bundles $B_1$, ..., $B_{i-1}$, exactly $\sum_{t=1}^{i-1}{\ell_t}-i+1$ have better true rank than $\chi$ and exactly $k(i-1)-\sum_{t=1}^{i-1}{\ell_t}$ have worse true rank than $\chi$. Assuming $\Delta$ and $\mathcal{Z}_1$, ..., $\mathcal{Z}_{i-1}$, the $k-1$ exam papers in bundle $B_i$ (besides $\chi$) are selected uniformly at random without replacement between all exam papers that have not been included in bundles $B_1$, ..., $B_{i-1}$. Then, the probability that exactly $\ell_i-1$ among the $k-1$ exam papers of bundle $B_i$ have better rank than exam paper $\chi$ is
\begin{align}\nonumber
&{\Pr}_n[\mathcal{Z}_i|\mathcal{Z}_1, ..., \mathcal{Z}_{i-1},\Delta]\\\label{eq:products}
&= {k-1 \choose \ell_i-1} \prod_{j=1}^{\ell_i-1}{\frac{\chi+i-j-1-\sum_{t=1}^{i-1}{\ell_t}}{n-j-(k-1)(i-1)}}\prod_{j=\ell_i+1}^{k}{\frac{n+1-\chi-(i-1)k-j+\sum_{t=1}^{i-1}{\ell_t}}{n+1-(k-1)(i-1)-\ell_i-j}}.
\end{align}

We bound the fractions in the above expression using the property $\frac{\alpha}{\beta}\leq \frac{\alpha+\gamma}{\beta+\gamma}$ when $0\leq \alpha\leq \beta$ and $\gamma>0$ and the facts that $i,j,\ell_i\in [k]$. We have
\begin{align}\label{eq:upper-x}
\frac{\chi+i-j-1-\sum_{t=1}^{i-1}{\ell_t}}{n-j-(k-1)(i-1)} \leq \frac{\chi+(i-1)k-\sum_{t=1}^{i-1}{\ell_t}}{n} \leq \min\left\{1,\frac{\chi}{n}+\frac{k^2}{n}\right\}
\end{align}
and
\begin{align}\label{eq:upper-n-x}
\frac{n+1-\chi-(i-1)k-j+\sum_{t=1}^{i-1}{\ell_t}}{n+1-(k-1)(i-1)-\ell_i-j} &\leq \frac{n-\chi+\sum_{t=1}^{i}{\ell_t}}{n} \leq \min\left\{1,1-\frac{\chi}{n}+\frac{k^2}{n}\right\}.
\end{align}
Using equalities (\ref{eq:prob-true}) and (\ref{eq:products}) and inequalities (\ref{eq:upper-x}) and (\ref{eq:upper-n-x}), we get
\begin{align}\nonumber
{\Pr}_n[\cap_{i=1}^k{\mathcal{Z}_i}|\Delta] &\leq \prod_{i=1}^k{{k-1 \choose \ell_i-1}\left(\min\left\{1,\frac{\chi}{n}+\frac{k^2}{n}\right\}\right)^{\ell_i-1}\left(\min\left\{1,1-\frac{\chi}{n}+\frac{k^2}{n}\right\}\right)^{k-\ell_i}}\\\label{eq:mins}
&= \left(\min\left\{1,\frac{\chi}{n}+\frac{k^2}{n}\right\}\right)^{|\ell|_1-k}\left(\min\left\{1,1-\frac{\chi}{n}+\frac{k^2}{n}\right\}\right)^{k^2-|\ell|_1}\prod_{i=1}^k{{k-1 \choose \ell_i-1}}.
\end{align}
Using the variation of Bernoulli inequality, which states that $(\alpha+\beta)^\gamma\leq \alpha^\gamma+\beta\gamma$ when $\alpha,\beta>0$, $\alpha+\beta\leq 1$, and $\gamma$ is a non-negative integer, we have 
\begin{align}\nonumber
&\left(\min\left\{1,\frac{\chi}{n}+\frac{k^2}{n}\right\}\right)^{|\ell|_1-k}\left(\min\left\{1,1-\frac{\chi}{n}+\frac{k^2}{n}\right\}\right)^{k^2-|\ell|_1}\\\nonumber
&\leq \min\left\{1,\left(\frac{\chi}{n}\right)^{|\ell|_1-k}+\frac{k^2}{n}(|\ell|_1-k)\right\} \cdot \min\left\{1,\left(1-\frac{\chi}{n}\right)^{k^2-|\ell|_1}+\frac{k^2}{n}(k^2-|\ell|_1)\right\}\\\label{eq:upper-mins}
&\leq \left(\frac{\chi}{n}\right)^{|\ell|_1-k}\cdot \left(1-\frac{\chi}{n}\right)^{k^2-|\ell|_1}+\frac{k^4}{n}.
\end{align}
Putting (\ref{eq:x-sigma}), (\ref{eq:graders}), (\ref{eq:mins}), and (\ref{eq:upper-mins}) together and using the definition of $\theta_\sigma(\chi/n)$, we obtain that
\begin{align}\nonumber
{\Pr}_n[\chi\rhd\sigma|\Delta] &\leq \theta_\sigma(\chi/n)+\frac{k^4}{n}\cdot N(\sigma)\sum_{\ell\in L_k}{\prod_{i=1}^k{{k-1\choose \ell_i-1}p_{\sigma_i,\ell_i}}}\\\label{eq:upper}
&\leq \theta_\sigma(\chi/n)+k^{k^2+2k+4}/n \leq \theta_\sigma(\chi/n)+\eta(k)/n.
\end{align}

Similarly, we bound the fractions in the RHS of (\ref{eq:products}) from below using the facts that $i,j,\ell_i\in [k]$. We have
\begin{align}\label{eq:lower-x}
\frac{\chi+i-j-1-\sum_{t=1}^{i-1}{\ell_t}}{n-j-(k-1)(i-1)} \geq \frac{\chi+i-j-1-\sum_{t=1}^{i-1}{\ell_t}}{n}  \geq \max\left\{0,\frac{\chi}{n}-\frac{k^2}{n}\right\}
\end{align}
and
\begin{align}\label{eq:lower-n-x}
\frac{n+1-\chi-(i-1)k-j+\sum_{t=1}^{i-1}{\ell_t}}{n+1-(k-1)(i-1)-\ell_i-j} &\geq \frac{n+1-\chi-(i-1)k-j+\sum_{t=1}^{i-1}{\ell_t}}{n} \geq \max\left\{0,1-\frac{\chi}{n}-\frac{k^2}{n}\right\}.
\end{align}
Using equalities (\ref{eq:prob-true}) and (\ref{eq:products}) and inequalities (\ref{eq:lower-x}) and (\ref{eq:lower-n-x}), we get
\begin{align}\nonumber
{\Pr}_n[\cap_{i=1}^k{\mathcal{Z}_i}|\Delta] &\geq \prod_{i=1}^k{{k-1 \choose \ell_i-1}\left(\max\left\{0,\frac{\chi}{n}-\frac{k^2}{n}\right\}\right)^{\ell_i-1}\left(\max\left\{0,1-\frac{\chi}{n}-\frac{k^2}{n}\right\}\right)^{k-\ell_i}}\\\label{eq:maxs}
&= \left(\max\left\{0,\frac{\chi}{n}-\frac{k^2}{n}\right\}\right)^{|\ell|_1-k}\left(\max\left\{0,1-\frac{\chi}{n}-\frac{k^2}{n}\right\}\right)^{k^2-|\ell|_1}\prod_{i=1}^k{{k-1 \choose \ell_i-1}}.
\end{align}
Using the variation of Bernoulli inequality, which states that $(\alpha-\beta)^\gamma\geq \alpha^\gamma-\beta\gamma$ when $\alpha,\beta>0$, $\alpha-\beta\geq 0$, and $\gamma$ is a non-negative integer, we have 
\begin{align}\nonumber
&\left(\max\left\{0,\frac{\chi}{n}-\frac{k^2}{n}\right\}\right)^{|\ell|_1-k}\left(\max\left\{0,1-\frac{\chi}{n}-\frac{k^2}{n}\right\}\right)^{k^2-|\ell|_1}\\\nonumber
&\geq \max\left\{0,\left(\frac{\chi}{n}\right)^{|\ell|_1-k}-\frac{k^2}{n}(|\ell|_1-k)\right\} \cdot \max\left\{0,\left(1-\frac{\chi}{n}\right)^{k^2-|\ell|_1}-\frac{k^2}{n}(k^2-|\ell|_1)\right\}\\\label{eq:lower-maxs}
&\geq \left(\frac{\chi}{n}\right)^{|\ell|_1-k}\cdot \left(1-\frac{\chi}{n}\right)^{k^2-|\ell|_1}-\frac{k^4}{n}.
\end{align}
Putting (\ref{eq:x-sigma}), (\ref{eq:graders}), (\ref{eq:maxs}), and (\ref{eq:lower-maxs}) together and using the definition of $\theta_\sigma(\chi/n)$, we obtain that
\begin{align}\nonumber
{\Pr}_n[\chi\rhd\sigma|\Delta] &\geq \theta_\sigma(\chi/n)-\frac{k^4}{n}\cdot N(\sigma)\sum_{\ell\in L_k}{\prod_{i=1}^k{{k-1\choose \ell_i-1}p_{\sigma_i,\ell_i}}}\\\label{eq:lower}
&\geq \theta_\sigma(\chi/n)-k^{k^2+2k+4}/n \geq \theta_\sigma(\chi/n)-\eta(k)/n.
\end{align}
The lemma follows by inequalities (\ref{eq:upper}) and (\ref{eq:lower}).
\end{proof}

Lemmas~\ref{lem:approx-via-x-rhd-sigma-given-delta} and~\ref{lem:approx-x-rhd-sigma-given-delta} imply the following.
\begin{corollary}\label{cor:approx}
For every exam paper $\chi\in [n]$ and type $\sigma\in \mathcal{T}_k$, it holds that $$|{\Pr}_n[\chi\rhd \sigma]-\theta_\sigma(\chi/n)| \leq 2\eta(k)/n.$$
\end{corollary}

We next focus on two exam papers $\chi, \upsilon\in [n]$. We will show that the events that they get types $\sigma,\sigma'\in \mathcal{T}_k$ are almost independent.
\begin{lemma}\label{lem:independent}
For every pair of exam papers $\chi,\upsilon\in [n]$ and pair of types $\sigma,\sigma'\in \mathcal{T}_k$, it holds that $$|{\Pr}_n[\chi \rhd \sigma \mbox{ and } \upsilon \rhd \sigma']-{\Pr}_n[\chi \rhd \sigma]\cdot {\Pr}_n[\upsilon \rhd \sigma']|\leq \eta(k)/n.$$
\end{lemma}

\begin{proof}
Consider the two exam papers $\chi$ and $\upsilon$ and let $\Gamma$ denote the event that no bundle contains both $\chi$ and $\upsilon$. Using Bayes' rule and the law of total probability, we have
\begin{align*}
&{\Pr}_n[\chi \rhd \sigma \mbox{ and } \upsilon \rhd \sigma']\\
&={\Pr}_n[\chi \rhd \sigma]\cdot {\Pr}_n[\upsilon \rhd \sigma'|\chi\rhd \sigma]\\
&= {\Pr}_n[\chi \rhd \sigma]\cdot \left({\Pr}_n[\upsilon \rhd \sigma'|\chi\rhd \sigma,\Gamma]{\Pr}_n[\Gamma]+{\Pr}_n[\upsilon \rhd \sigma'|\chi\rhd \sigma,\overline{\Gamma}]{\Pr}_n[\overline{\Gamma}]\right)\\
&= {\Pr}_n[\chi \rhd \sigma]\cdot \left({\Pr}_n[\upsilon \rhd \sigma'|\Gamma]{\Pr}_n[\Gamma]+{\Pr}_n[\upsilon \rhd \sigma'|\chi\rhd \sigma,\overline{\Gamma}]{\Pr}_n[\overline{\Gamma}]\right)\\
&= {\Pr}_n[\chi \rhd \sigma]\cdot {\Pr}_n[\upsilon \rhd \sigma']+{\Pr}_n[\overline{\Gamma}]\cdot {\Pr}_n[\chi\rhd \sigma]\cdot \left({\Pr}_n[\upsilon \rhd \sigma'|\chi\rhd \sigma,\overline{\Gamma}]-{\Pr}_n[\upsilon\rhd \sigma'|\overline{\Gamma}]\right).
\end{align*}
In the third equality, we have used the fact that, whether exam paper $\upsilon$ gets type $\sigma'$ does not depend on whether exam paper $\chi$ gets type $\sigma$ when no bundle contains both $\chi$ and $\upsilon$, i.e., ${\Pr}_n[\upsilon\rhd \sigma'|\chi\rhd \sigma,\Gamma]={\Pr}_n[\upsilon\rhd \sigma'|\Gamma]$.
Hence,
\begin{align*}
|{\Pr}_n[\chi \rhd \sigma \mbox{ and } \upsilon \rhd \sigma']-{\Pr}_n[\chi \rhd \sigma]\cdot {\Pr}_n[\upsilon \rhd \sigma']| &\leq \Pr[\overline{\Gamma}].
\end{align*}
It remains to show that $\Pr[\overline{\Gamma}] \leq \eta(k)/n$. Let $t\in \{k, k+1, ..., k(k-1)\}$ be the random variable indicating the number of papers different than $\chi$ which appear in the bundles of $\chi$. Given $t$, the probability that paper $\upsilon$ is one of these papers is ${\Pr}_n[\overline{\Gamma}|t]=\frac{t}{n-1}\leq \frac{\eta(k)}{n}$.
Hence, ${\Pr}_n[\overline{\Gamma}]\leq \eta(k)/n$ as well and the lemma follows. 
\end{proof}

Now, we use $C_n$ to denote the expected fraction of pairwise relations between exam papers that are recovered correctly by the type-ordering aggregation rule $\succ$. Using Lemma~\ref{lem:independent} and Corollary~\ref{cor:approx}, we have
\begin{align*}
C_n &= {n \choose 2}^{-1}\sum_{\chi=1}^{n-1}{\sum_{\upsilon=\chi+1}^{n}{\left(\sum_{\sigma,\sigma':\sigma\succ \sigma'}{{\Pr}_{n}[\chi\rhd \sigma \mbox{ and } \upsilon\rhd \sigma']}+\frac{1}{2}\sum_\sigma{{\Pr}_n[\chi\rhd \sigma \mbox{ and } \upsilon\rhd \sigma]}\right)}}\\
&\leq {n \choose 2}^{-1}\sum_{\chi=1}^{n-1}{\sum_{\upsilon=\chi+1}^{n}{\left(\sum_{\sigma,\sigma':\sigma\succ \sigma'}{\left({\Pr}_{n}[\chi\rhd \sigma]\cdot {\Pr}_n[\upsilon\rhd \sigma']+\frac{\eta(k)}{n}\right)}\right.}}\\
&\quad {{\left.+\frac{1}{2}\sum_\sigma{\left({\Pr}_n[\chi\rhd \sigma]\cdot{\Pr}_n[\upsilon\rhd \sigma]+\frac{\eta(k)}{n}\right)}\right)}}\\
&\leq {n \choose 2}^{-1}\sum_{\chi=1}^{n-1}{\sum_{\upsilon=\chi+1}^{n}{\left(\sum_{\sigma,\sigma':\sigma\succ \sigma'}{\left(\theta_\sigma(\chi/n)\cdot \theta_{\sigma'}(\upsilon/n)+5\frac{\eta(k)}{n}\right)}+\frac{1}{2}\sum_\sigma{\left(\theta_\sigma(\chi/n)\cdot \theta_{\sigma}(\upsilon/n)+5\frac{\eta(k)}{n}\right)}\right)}}\\
&= {n \choose 2}^{-1}\sum_{\chi=1}^{n-1}{\sum_{\upsilon=\chi+1}^{n}{\left(\sum_{\sigma,\sigma':\sigma\succ \sigma'}{\theta_\sigma(\chi/n)\cdot \theta_{\sigma'}(\upsilon/n)}+\frac{1}{2}\sum_\sigma{\theta_\sigma(\chi/n)\cdot \theta_{\sigma}(\upsilon/n)}\right)}}+\frac{5|{\mathcal{T}}_k|^2 \eta(k)}{2n}.
\end{align*}
In the second inequality, besides Corollary~\ref{cor:approx}, we have also used the fact that $\theta_\sigma(\chi/n)$ and $\theta_{\sigma'}(\upsilon/n)$ are at most $1$.

Similarly, we can obtain the following lower bound:
\begin{align*}
C_n &\geq  {n \choose 2}^{-1}\sum_{\chi=1}^{n-1}{\sum_{\upsilon=\chi+1}^{n}{\left(\sum_{\sigma,\sigma':\sigma\succ \sigma'}{\theta_\sigma(\chi/n)\cdot \theta_{\sigma'}(\upsilon/n)}+\frac{1}{2}\sum_\sigma{\theta_\sigma(\chi/n)\cdot \theta_{\sigma}(\upsilon/n)}\right)}}-\frac{5|{\mathcal{T}}_k|^2 \eta(k)}{2n}.
\end{align*}
So far, we have shown that $C_n$ takes values that range in an interval of width $5|\mathcal{T}_k|^2\eta(k)/n$. We remark that our experiments indicate that the concentration is much sharper. 

Notice that the quantity $\frac{5|{\mathcal{T}}_k|^2 \eta(k)}{2n}$, which appears in the above upper and lower bounds of $C_n$, approaches $0$ as $n$ tends to infinity, as both $|\mathcal{T}_k|$ and $\eta(k)$ depend only on $k$. Hence,
\begin{align*}
\lim_{n\rightarrow +\infty}{C_n} &= \int_0^1{\int_x^1{\left(\sum_{\sigma,\sigma':\sigma\succ \sigma'}{\theta_\sigma(x)\cdot \theta_{\sigma'}(y)}+\frac{1}{2}\sum_\sigma{\theta_\sigma(x)\cdot \theta_{\sigma}(y)}\right)\ud y}\ud x}\\
&=\sum_{\sigma,\sigma':\sigma\succ \sigma'}{W(\sigma,\sigma')} +\frac{1}{2}\sum_\sigma{W(\sigma,\sigma)},
\end{align*}
where the weights $W$ are exactly as we have defined them in Section~\ref{subsec:framework}. We conclude that the quantity $C$ that we study in our non-rigorous analysis is indeed the limit of the expected number of correctly recovered pairwise relations as the number of students approaches infinity. Extending the analysis for more general bivariate performance objectives $f$ (as we did in Section~\ref{subsec:optimal}) is straightforward.

\section{Experimental data}\label{app:sec:data}
The following tables contain the data collected from our two field experiments. Information for each student consists of an identifier, a half-integer quality (the cardinal grade of the student in the mid term exam) and the ranking provided by the student for the six exam papers in her bundle (assuming that the correct ranking is 1 2 3 4 5 6). Table \ref{tab:2015} contains the data collected in the 2015 field experiment. Due to the larger number of participating students, the data from the 2016 field experiment have been split into the two Tables \ref{tab:2016a} and \ref{tab:2016b}.

\begin{table}[h!]
\centering{\small
\begin{tabular}{ccc ccc ccc ccc}
\noalign{\hrule height 1pt}\hline
\# & qual. & ranking & \# & qual. & ranking & \# & qual. & ranking & \# & qual. & ranking  \\\hline
1	&	10	&	2	1	3	4	5	6	&	35	&	7	&	1	3	2	6	4	5	&	69	&	5	&	5	6	4	1	2	3	&	103	&	3.5	&	2	1	3	4	5	6	\\
2	&	10	&	1	2	3	4	5	6	&	36	&	7	&	6	5	3	4	1	2	&	70	&	5	&	4	5	2	1	3	6	&	104	&	3.5	&	1	2	4	3	5	6	\\
3	&	9.5	&	1	3	2	4	5	6	&	37	&	6.5	&	1	2	5	3	4	6	&	71	&	5	&	2	1	3	4	5	6	&	105	&	3.5	&	4	2	3	6	5	1	\\
4	&	9.5	&	1	2	5	4	3	6	&	38	&	6.5	&	2	1	5	3	4	6	&	72	&	5	&	1	3	2	4	5	6	&	106	&	3	&	4	1	2	5	3	6	\\
5	&	9	&	1	3	2	5	4	6	&	39	&	6.5	&	1	2	6	3	4	5	&	73	&	5	&	2	1	5	3	4	6	&	107	&	3	&	4	1	3	2	6	5	\\
6	&	9	&	3	1	2	5	4	6	&	40	&	6.5	&	1	2	3	4	5	6	&	74	&	4.5	&	1	3	4	2	5	6	&	108	&	3	&	2	3	4	6	5	1	\\
7	&	9	&	2	6	5	4	3	1	&	41	&	6	&	1	4	2	3	5	6	&	75	&	4.5	&	1	2	3	5	4	6	&	109	&	3	&	3	4	5	1	2	6	\\
8	&	8.5	&	3	2	1	4	5	6	&	42	&	6	&	1	2	3	5	6	4	&	76	&	4.5	&	1	5	2	3	4	6	&	110	&	3	&	3	1	5	6	4	2	\\
9	&	8	&	3	4	2	5	1	6	&	43	&	6	&	2	1	5	3	4	6	&	77	&	4.5	&	1	4	3	2	5	6	&	111	&	3	&	1	3	2	6	4	5	\\
10	&	8	&	1	3	2	4	5	6	&	44	&	6	&	5	1	2	4	3	6	&	78	&	4.5	&	3	1	2	5	4	6	&	112	&	2.5	&	2	1	5	3	4	6	\\
11	&	8	&	3	1	2	4	5	6	&	45	&	6	&	1	2	3	4	5	6	&	79	&	4.5	&	2	4	1	3	6	5	&	113	&	2.5	&	3	4	6	1	2	5	\\
12	&	8	&	1	5	6	2	3	4	&	46	&	6	&	1	4	2	3	6	5	&	80	&	4.5	&	3	4	1	5	2	6	&	114	&	2.5	&	1	2	4	3	6	5	\\
13	&	8	&	5	3	4	6	2	1	&	47	&	6	&	1	5	2	3	4	6	&	81	&	4.5	&	1	2	6	4	5	3	&	115	&	2.5	&	1	2	4	5	3	6	\\
14	&	8	&	1	2	3	4	5	6	&	48	&	6	&	1	2	3	6	5	4	&	82	&	4.5	&	2	3	1	5	4	6	&	116	&	2.5	&	1	5	2	6	3	4	\\
15	&	8	&	1	2	3	5	4	6	&	49	&	6	&	1	2	5	6	4	3	&	83	&	4.5	&	3	1	4	2	6	5	&	117	&	2.5	&	1	3	4	2	5	6	\\
16	&	8	&	1	4	3	6	2	5	&	50	&	6	&	3	6	1	2	4	5	&	84	&	4.5	&	3	2	1	5	4	6	&	118	&	2.5	&	1	3	2	6	4	5	\\
17	&	8	&	4	2	3	5	1	6	&	51	&	6	&	1	2	3	4	5	6	&	85	&	4.5	&	3	1	4	2	6	5	&	119	&	2.5	&	3	1	4	6	5	2	\\
18	&	8	&	3	2	4	5	6	1	&	52	&	5.5	&	1	2	3	4	5	6	&	86	&	4.5	&	4	1	2	6	3	5	&	120	&	2.5	&	4	2	3	1	6	5	\\
19	&	7.5	&	1	2	3	6	4	5	&	53	&	5.5	&	4	2	1	3	5	6	&	87	&	4.5	&	4	2	3	5	1	6	&	121	&	2	&	1	3	4	2	5	6	\\
20	&	7.5	&	1	2	6	5	4	3	&	54	&	5.5	&	3	1	5	4	6	2	&	88	&	4	&	1	4	5	2	3	6	&	122	&	2	&	1	3	2	4	5	6	\\
21	&	7.5	&	1	3	6	5	4	2	&	55	&	5.5	&	4	1	3	6	5	2	&	89	&	4	&	1	2	6	5	3	4	&	123	&	2	&	2	1	4	5	3	6	\\
22	&	7.5	&	3	5	4	1	2	6	&	56	&	5.5	&	1	6	2	5	4	3	&	90	&	4	&	3	1	4	6	5	2	&	124	&	2	&	2	1	4	5	3	6	\\
23	&	7.5	&	1	2	6	3	4	5	&	57	&	5.5	&	1	2	4	5	3	6	&	91	&	4	&	1	3	2	4	5	6	&	125	&	2	&	3	4	1	2	6	5	\\
24	&	7.5	&	1	3	2	5	4	6	&	58	&	5.5	&	6	3	4	2	1	5	&	92	&	4	&	4	6	1	5	2	3	&	126	&	2	&	2	3	1	4	6	5	\\
25	&	7.5	&	2	3	5	4	6	1	&	59	&	5.5	&	1	6	5	3	4	2	&	93	&	4	&	1	2	3	4	5	6	&	127	&	2	&	3	2	5	4	1	6	\\
26	&	7	&	1	2	3	5	6	4	&	60	&	5.5	&	4	5	2	1	3	6	&	94	&	4	&	2	4	5	3	6	1	&	128	&	1.5	&	1	3	2	4	5	6	\\
27	&	7	&	1	3	2	5	4	6	&	61	&	5.5	&	1	4	5	6	3	2	&	95	&	3.5	&	2	1	4	3	5	6	&	129	&	1.5	&	1	2	6	3	4	5	\\
28	&	7	&	1	2	6	3	4	5	&	62	&	5	&	1	2	4	3	5	6	&	96	&	3.5	&	2	1	3	4	6	5	&	130	&	1.5	&	2	4	3	1	5	6	\\
29	&	7	&	2	1	5	3	6	4	&	63	&	5	&	2	1	3	4	6	5	&	97	&	3.5	&	1	4	2	3	5	6	&	131	&	1.5	&	5	2	4	3	1	6	\\
30	&	7	&	2	1	3	4	5	6	&	64	&	5	&	2	1	4	3	5	6	&	98	&	3.5	&	2	1	3	5	4	6	&	132	&	1	&	1	2	6	5	3	4	\\
31	&	7	&	1	2	3	6	4	5	&	65	&	5	&	2	1	3	4	6	5	&	99	&	3.5	&	4	1	3	2	6	5	&	133	&	1	&	5	6	3	2	4	1	\\
32	&	7	&	2	3	1	5	4	6	&	66	&	5	&	1	2	3	6	4	5	&	100	&	3.5	&	1	5	4	3	2	6	&	134	&	0.5	&	3	1	2	4	5	6	\\
33	&	7	&	1	2	4	6	3	5	&	67	&	5	&	4	1	2	3	5	6	&	101	&	3.5	&	5	1	4	2	6	3	&	135	&	0.5	&	3	4	5	6	1	2	\\
34	&	7	&	2	3	1	5	4	6	&	68	&	5	&	1	3	2	6	5	4	&	102	&	3.5	&	6	2	1	3	4	5	&	136	&	0	&	2	3	1	4	6	5	\\
\noalign{\hrule height 1pt}\hline\end{tabular}}
\caption{The data collected in our 2015 field experiment (with $136$ students). Each quality/ranking pair corresponds to a student. Quality takes half-integer values between $0$ and $10$. The ranking $1\,2\,3\,4\,5\,6$ is the correct one. Hence, in the ranking $2\,1\,3\,4\,5\,6$ provided by the first student, all pairwise relations are correct besides the one involving the best and the second best exam papers in the bundle of the student.
From these rankings, we can compute the probability $p_{i,j}$ for each $i \in [6], j \in [6]$ by counting the number of times $j$ appears $i$-th in the rankings provided by all students, and then dividing by the total number of students that participated in the field experiment. For example, since $2$ appears $28$ times first, we have that $p_{2,1} = 28/136 \approx 0.2059$; see the noise matrix $P_{2015}$, given by \eqref{eq:real-2015}.
}
\label{tab:2015}
\end{table} 

\begin{table}[h!]
\centering{\small
\begin{tabular}{ccc ccc ccc ccc}
\noalign{\hrule height 1pt}\hline
\# & qual. & ranking & \# & qual. & ranking & \# & qual. & ranking & \# & qual. & ranking  \\\hline
1	&	10	&	1	2	3	5	6	4	&	31	&	8.5	&	3	2	5	4	1	6	&	61	&	7	&	1	3	2	4	5	6	&	91	&	6	&	1	2	3	4	5	6	\\
2	&	10	&	1	3	2	4	5	6	&	32	&	8	&	1	2	3	4	5	6	&	62	&	7	&	1	2	3	6	4	5	&	92	&	6	&	2	5	6	3	4	1	\\
3	&	10	&	1	2	3	6	5	4	&	33	&	8	&	1	4	5	2	3	6	&	63	&	7	&	1	2	3	4	5	6	&	93	&	6	&	1	2	5	3	4	6	\\
4	&	10	&	1	2	3	4	5	6	&	34	&	8	&	3	1	2	5	4	6	&	64	&	7	&	5	1	2	6	4	3	&	94	&	6	&	1	3	6	5	2	4	\\
5	&	10	&	2	1	3	4	5	6	&	35	&	8	&	1	3	4	2	5	6	&	65	&	7	&	1	4	3	2	5	6	&	95	&	6	&	2	1	3	4	5	6	\\
6	&	10	&	1	6	2	4	5	3	&	36	&	8	&	1	5	3	4	2	6	&	66	&	7	&	1	2	3	4	5	6	&	96	&	6	&	1	3	2	4	5	6	\\
7	&	10	&	1	2	3	4	5	6	&	37	&	8	&	3	1	2	4	5	6	&	67	&	7	&	1	2	3	4	6	5	&	97	&	6	&	4	1	5	2	3	6	\\
8	&	9.5	&	1	2	4	3	6	5	&	38	&	8	&	1	2	4	3	5	6	&	68	&	7	&	1	2	4	6	3	5	&	98	&	6	&	1	2	4	3	5	6	\\
9	&	9.5	&	6	4	5	3	1	2	&	39	&	8	&	2	1	3	4	5	6	&	69	&	6.5	&	2	1	4	3	6	5	&	99	&	6	&	3	1	4	2	5	6	\\
10	&	9.5	&	1	2	3	5	4	6	&	40	&	8	&	3	2	1	5	4	6	&	70	&	6.5	&	2	4	1	3	5	6	&	100	&	6	&	1	2	3	5	4	6	\\
11	&	9.5	&	1	2	5	3	4	6	&	41	&	8	&	5	2	3	4	1	6	&	71	&	6.5	&	1	3	5	2	4	6	&	101	&	6	&	2	1	5	3	6	4	\\
12	&	9	&	4	5	2	1	3	6	&	42	&	8	&	1	3	2	4	5	6	&	72	&	6.5	&	5	2	1	3	6	4	&	102	&	6	&	1	4	2	3	5	6	\\
13	&	9	&	1	2	3	4	5	6	&	43	&	7.5	&	2	1	3	5	4	6	&	73	&	6.5	&	1	4	3	2	5	6	&	103	&	6	&	1	2	5	3	4	6	\\
14	&	9	&	1	2	3	5	6	4	&	44	&	7.5	&	1	2	3	4	5	6	&	74	&	6.5	&	1	2	3	5	4	6	&	104	&	6	&	2	1	3	5	4	6	\\
15	&	9	&	1	2	3	4	6	5	&	45	&	7.5	&	1	3	6	5	4	2	&	75	&	6.5	&	1	2	3	4	6	5	&	105	&	6	&	1	2	4	3	5	6	\\
16	&	9	&	2	1	3	6	4	5	&	46	&	7.5	&	4	1	2	5	6	3	&	76	&	6.5	&	1	2	3	4	5	6	&	106	&	6	&	4	5	6	2	3	1	\\
17	&	9	&	5	4	1	3	2	6	&	47	&	7.5	&	3	1	2	4	5	6	&	77	&	6.5	&	1	2	5	4	6	3	&	107	&	5.5	&	6	4	2	5	3	1	\\
18	&	9	&	5	1	2	4	3	6	&	48	&	7.5	&	1	2	3	4	5	6	&	78	&	6.5	&	1	2	3	4	5	6	&	108	&	5.5	&	2	1	3	4	5	6	\\
19	&	9	&	3	1	2	4	6	5	&	49	&	7.5	&	1	3	2	4	5	6	&	79	&	6.5	&	1	3	2	4	5	6	&	109	&	5.5	&	1	3	2	4	5	6	\\
20	&	9	&	2	3	4	1	6	5	&	50	&	7.5	&	1	2	3	4	5	6	&	80	&	6.5	&	1	2	5	3	4	6	&	110	&	5.5	&	1	3	2	4	5	6	\\
21	&	8.5	&	1	2	3	4	6	5	&	51	&	7.5	&	1	2	3	5	6	4	&	81	&	6.5	&	2	1	3	6	4	5	&	111	&	5.5	&	4	1	2	3	5	6	\\
22	&	8.5	&	2	4	3	1	5	6	&	52	&	7.5	&	2	1	3	6	4	5	&	82	&	6	&	1	4	3	2	5	6	&	112	&	5.5	&	1	2	3	4	5	6	\\
23	&	8.5	&	2	3	1	4	5	6	&	53	&	7.5	&	1	2	3	4	5	6	&	83	&	6	&	1	2	3	4	5	6	&	113	&	5.5	&	2	3	1	4	6	5	\\
24	&	8.5	&	1	2	3	4	5	6	&	54	&	7.5	&	1	2	4	5	3	6	&	84	&	6	&	1	3	2	5	4	6	&	114	&	5.5	&	2	1	3	5	4	6	\\
25	&	8.5	&	2	1	4	3	6	5	&	55	&	7.5	&	1	2	3	6	4	5	&	85	&	6	&	3	1	4	2	5	6	&	115	&	5.5	&	1	2	3	4	6	5	\\
26	&	8.5	&	2	3	1	4	5	6	&	56	&	7.5	&	1	2	4	3	5	6	&	86	&	6	&	1	2	3	5	4	6	&	116	&	5.5	&	3	2	1	6	4	5	\\
27	&	8.5	&	2	5	3	1	4	6	&	57	&	7.5	&	2	3	1	4	6	5	&	87	&	6	&	1	5	4	2	3	6	&	117	&	5.5	&	3	1	2	5	4	6	\\
28	&	8.5	&	1	3	5	2	4	6	&	58	&	7	&	2	1	4	3	6	5	&	88	&	6	&	1	2	3	4	5	6	&	118	&	5.5	&	1	2	3	4	6	5	\\
29	&	8.5	&	4	1	3	2	5	6	&	59	&	7	&	2	1	3	4	5	6	&	89	&	6	&	1	5	2	6	4	3	&	119	&	5.5	&	1	2	5	3	6	4	\\
30	&	8.5	&	1	4	3	2	5	6	&	60	&	7	&	1	2	3	4	5	6	&	90	&	6	&	5	1	2	3	4	6	&	120	&	5.5	&	3	2	1	4	5	6	\\
\noalign{\hrule height 1pt}\hline\end{tabular}}
\caption{The data collected in our 2016 field experiment (the table contains data for $120$ out of $241$ students). Each quality/ranking pair corresponds to a student. Quality takes half-integer values between $0$ and $10$. The ranking $1\,2\,3\,4\,5\,6$ is the correct one.}
\label{tab:2016a}
\end{table} 

\begin{table}[h!]
\centering
{\small
\begin{tabular}{ccc ccc ccc ccc}
\noalign{\hrule height 1pt}\hline
\# & qual. & ranking & \# & qual. & ranking & \# & qual. & ranking & \# & qual. & ranking  \\\hline
121	&	5.5	&	4	2	3	1	5	6	&	152	&	5	&	1	2	3	4	5	6	&	182	&	4	&	1	2	3	5	6	4	&	212	&	3	&	1	5	2	4	3	6	\\
122	&	5.5	&	1	3	4	2	6	5	&	153	&	4.5	&	1	4	3	2	5	6	&	183	&	4	&	1	2	3	5	4	6	&	213	&	2.5	&	1	2	3	4	5	6	\\
123	&	5.5	&	1	3	2	5	4	6	&	154	&	4.5	&	1	2	3	4	5	6	&	184	&	4	&	1	2	4	3	5	6	&	214	&	2.5	&	1	2	3	6	5	4	\\
124	&	5.5	&	1	2	3	4	5	6	&	155	&	4.5	&	1	2	4	3	5	6	&	185	&	3.5	&	1	2	3	5	4	6	&	215	&	2.5	&	1	3	2	5	4	6	\\
125	&	5.5	&	1	2	3	4	6	5	&	156	&	4.5	&	4	2	1	3	5	6	&	186	&	3.5	&	1	2	3	4	5	6	&	216	&	2.5	&	3	2	4	1	5	6	\\
126	&	5.5	&	5	1	2	3	4	6	&	157	&	4.5	&	1	2	6	4	3	5	&	187	&	3.5	&	1	2	4	3	5	6	&	217	&	2.5	&	1	2	4	3	5	6	\\
127	&	5.5	&	1	2	6	5	3	4	&	158	&	4.5	&	4	1	2	3	5	6	&	188	&	3.5	&	5	1	4	2	3	6	&	218	&	2.5	&	1	2	3	4	5	6	\\
128	&	5.5	&	2	1	4	3	5	6	&	159	&	4.5	&	1	4	2	3	5	6	&	189	&	3.5	&	1	3	2	4	5	6	&	219	&	2.5	&	4	2	1	3	5	6	\\
129	&	5	&	1	3	2	4	5	6	&	160	&	4.5	&	4	1	2	3	6	5	&	190	&	3.5	&	1	2	4	5	6	3	&	220	&	2.5	&	1	2	3	5	4	6	\\
130	&	5	&	3	1	2	4	5	6	&	161	&	4.5	&	2	4	5	1	3	6	&	191	&	3.5	&	2	1	3	6	5	4	&	221	&	2.5	&	1	2	3	4	6	5	\\
131	&	5	&	1	2	3	4	6	5	&	162	&	4.5	&	6	3	5	4	2	1	&	192	&	3.5	&	1	3	2	4	5	6	&	222	&	2.5	&	2	1	3	4	5	6	\\
132	&	5	&	1	2	5	3	4	6	&	163	&	4.5	&	3	1	2	5	4	6	&	193	&	3.5	&	4	2	3	1	5	6	&	223	&	2.5	&	1	5	3	2	4	6	\\
133	&	5	&	1	2	3	4	5	6	&	164	&	4.5	&	1	2	3	4	5	6	&	194	&	3.5	&	1	2	5	3	4	6	&	224	&	2	&	1	6	3	4	2	5	\\
134	&	5	&	1	3	2	4	5	6	&	165	&	4.5	&	1	2	3	5	6	4	&	195	&	3.5	&	1	5	6	3	4	2	&	225	&	2	&	1	2	3	5	4	6	\\
135	&	5	&	5	1	2	4	3	6	&	166	&	4.5	&	1	2	5	3	4	6	&	196	&	3.5	&	1	2	4	3	6	5	&	226	&	2	&	1	2	5	4	6	3	\\
136	&	5	&	4	1	2	3	5	6	&	167	&	4.5	&	1	2	3	4	5	6	&	197	&	3.5	&	4	2	1	3	5	6	&	227	&	2	&	2	1	3	5	6	4	\\
137	&	5	&	1	2	3	4	5	6	&	168	&	4.5	&	1	2	4	5	3	6	&	198	&	3.5	&	1	3	2	4	5	6	&	228	&	2	&	1	2	3	4	5	6	\\
138	&	5	&	1	2	3	4	6	5	&	169	&	4.5	&	1	2	3	4	5	6	&	199	&	3.5	&	1	3	2	4	5	6	&	229	&	2	&	4	1	5	2	6	3	\\
139	&	5	&	1	2	3	4	6	5	&	170	&	4.5	&	2	3	1	4	5	6	&	200	&	3.5	&	2	1	3	4	5	6	&	230	&	2	&	1	6	5	2	4	3	\\
140	&	5	&	2	1	4	3	5	6	&	171	&	4.5	&	1	2	3	5	4	6	&	201	&	3.5	&	1	2	3	5	4	6	&	231	&	2	&	6	2	5	3	4	1	\\
141	&	5	&	1	3	4	2	5	6	&	172	&	4.5	&	2	1	3	4	5	6	&	202	&	3.5	&	2	1	3	5	6	4	&	232	&	2	&	1	4	6	2	3	5	\\
142	&	5	&	1	2	4	3	6	5	&	173	&	4.5	&	2	1	3	4	5	6	&	203	&	3.5	&	1	2	4	3	5	6	&	233	&	1.5	&	1	2	3	4	5	6	\\
143	&	5	&	5	3	1	2	6	4	&	174	&	4.5	&	1	2	3	4	5	6	&	204	&	3	&	3	1	2	4	5	6	&	234	&	1.5	&	5	6	3	4	2	1	\\
144	&	5	&	1	2	3	4	5	6	&	175	&	4	&	1	2	4	3	5	6	&	205	&	3	&	2	1	3	4	6	5	&	235	&	1.5	&	1	2	3	6	4	5	\\
145	&	5	&	2	5	1	3	6	4	&	176	&	4	&	3	1	2	4	5	6	&	206	&	3	&	1	2	4	3	6	5	&	236	&	1.5	&	1	2	3	4	5	6	\\
146	&	5	&	1	2	3	4	5	6	&	177	&	4	&	1	2	3	6	5	4	&	207	&	3	&	1	6	3	2	5	4	&	237	&	1	&	2	1	3	6	4	5	\\
147	&	5	&	4	5	3	2	6	1	&	178	&	4	&	1	2	3	4	5	6	&	208	&	3	&	1	3	2	5	4	6	&	238	&	1	&	2	3	1	5	6	4	\\
148	&	5	&	2	3	1	4	6	5	&	179	&	4	&	1	2	3	5	4	6	&	209	&	3	&	2	1	6	4	3	5	&	239	&	1	&	1	2	3	4	5	6	\\
149	&	5	&	2	3	4	1	5	6	&	180	&	4	&	1	2	3	4	5	6	&	210	&	3	&	1	2	4	3	5	6	&	240	&	0.5	&	1	4	2	6	3	5	\\
150	&	5	&	1	2	3	4	6	5	&	181	&	4	&	1	2	3	4	6	5	&	211	&	3	&	1	2	3	4	6	5	&	241	&	0.5	&	2	6	1	5	4	3	\\
151	&	5	&	2	1	3	4	5	6	&		&		&							&		&		&							&		&		&							\\
\noalign{\hrule height 1pt}\hline\end{tabular}}
\caption{The data collected in our 2016 field experiment (the table contains data for the remaining $121$ out of $241$ students).}
\label{tab:2016b}
\end{table} 

\end{document}